\documentclass[10pt,twocolumn,letterpaper]{article}

\usepackage[pagenumbers]{iccv} 

\usepackage{amssymb}
\usepackage{mathtools}
\usepackage{amsthm}

\newtheorem{theorem}{Theorem}
\newtheorem{proposition}{Proposition}

\usepackage{diagbox}
\usepackage{multirow}
\usepackage{algorithm}
\usepackage{algorithmic}

\usepackage[misc]{ifsym}

\usepackage{fontawesome5}

\newcommand\blfootnote[1]{%
  \begingroup
  \renewcommand\thefootnote{}\footnote{#1}%
  \addtocounter{footnote}{-1}%
  \endgroup
}

\usepackage{cuted}

\usepackage[absolute]{textpos}

\definecolor{iccvblue}{rgb}{0.21,0.49,0.74}
\usepackage[pagebackref,breaklinks,colorlinks,allcolors=iccvblue]{hyperref}

\def\paperID{13653} 
\def\confName{ICCV}
\def\confYear{2025}

\title{Metric Convolutions: A Unifying Theory to Adaptive Image Convolutions
\vspace{-0.5em}
}

\author{
Thomas Dag\`es\textsuperscript{1,2,3*}
\quad\quad
Michael Lindenbaum\textsuperscript{1} 
\quad\quad
 Alfred M. Bruckstein\textsuperscript{1}\\
{\normalsize
\textsuperscript{1}Technion -- Israel Institute of Technology \quad
\textsuperscript{2}Technical University of Munich \quad
\textsuperscript{3}Munich Center for Machine Learning
}
\vspace{-0.5em}
}

\begin{document}

\setlength{\abovedisplayskip}{3pt}
\setlength{\belowdisplayskip}{3pt}
{
\definecolor{somegray}{gray}{0.5}
\newcommand{\darkgrayed}[1]{\textcolor{somegray}{#1}}
\begin{textblock}{17}(-0.5, 1)  %
\centering
\darkgrayed{To appear in Proceedings of the \emph{IEEE/CVF International Conference on Computer Vision (ICCV)},\\ Honolulu, HI, USA, 2025 \copyright~2025 IEEE.}
\end{textblock}
}

\maketitle

\begin{abstract}
\blfootnote{* Research carried out at the Technion -- Israel Institute of Technology. \\ 
\phantom{azer} \Letter \hspace{0.2em} {\tt thomas.dages@cs.technion.ac.il} \\
\phantom{azer} \faGithubSquare \hspace{0.4em}{\url{https://github.com/Tommoo/MetricConvolutions}}
} 

\vspace{-2em}

Standard convolutions are prevalent in image processing and deep learning, but their fixed kernels limits adaptability. Several deformation strategies of the reference kernel grid have been proposed. Yet, they lack a unified theoretical framework. By returning to a metric perspective for images, now seen as two-dimensional manifolds equipped with notions of local and geodesic distances, either symmetric (Riemannian) or not (Finsler), we provide a unifying principle: the kernel positions are samples of unit balls of implicit metrics. With this new perspective, we also propose \textit{metric convolutions}, a novel approach that samples unit balls from explicit signal-dependent metrics, providing interpretable operators with geometric regularisation. This framework, compatible with gradient-based optimisation, can directly replace existing convolutions applied to either input images or deep features of neural networks. Metric convolutions typically require fewer parameters and provide better generalisation. Our approach shows competitive performance in standard denoising and classification tasks. 
\vspace{-1em}
\end{abstract}

\section{Introduction}

In the realm of computer vision and deep learning, traditional convolutions have established themselves as indispensable image processing tools \cite{forsyth2002computer,goodfellow2016deep}, forming the backbone of various neural network architectures, and in particular of the effective convolutional neural networks (CNNs) \cite{krizhevsky2012alexnet,simonyan2014vgg,szegedy2015going,he2016resnet,liu2022convnet}. 
Traditional convolutions involve applying fixed-size isotropic $k\times k$ filters to images, a strategy known for its weight-sharing property and parameter efficiency. However, their inherent rigidity becomes evident when dealing with deformable objects, complex spatial transformations, or multi-scale phenomena, limiting their adaptability and, consequently, their effectiveness.

To address this limitation, a diverse research community has explored alternative convolutions. We focus on methods modifying $k\times k$ kernel samples, distinct from those increasing kernel width \cite{ding2022scaling}, concatenating \cite{szegedy2015going,chollet2017xception} or linearly combining \cite{yang2019condconv,chen2020dynamic,li2021revisiting,li2022omni} responses from different 
\begin{figure}[t!]
    \centering
    \includegraphics[width=0.92\columnwidth]{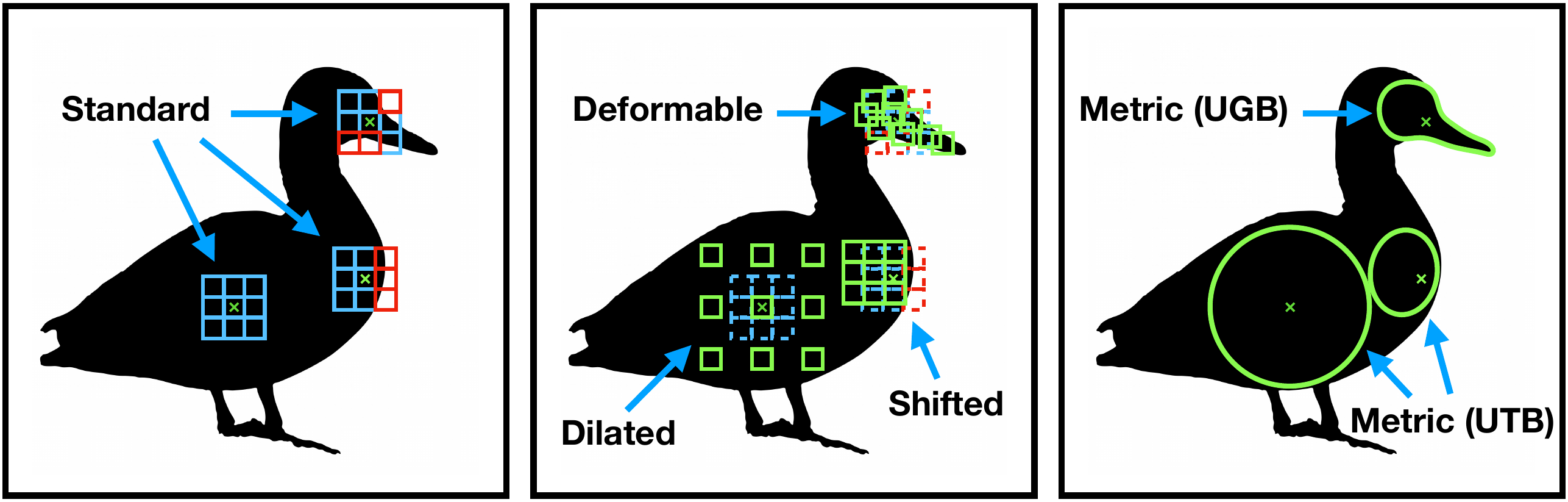}
    \caption{
    Unlike rigid standard convolutions, adaptive kernels avoid sampling undesired locations, e.g.\ background near object boundaries. Our metric convolutions sample unit balls, tangent (UTB) or geodesic (UGB), of explicit metrics,
    providing geometric interpretability and inductive bias, 
    while other methods only rely on metrics implicitly.
    }
    \label{fig: duck motivation adaptive kernels}
    \vspace{-1em}
\end{figure}
kernels. Dilated convolutions \cite{chen2015semantic,yu2015multiscale,chen2017deeplab,yu2017dilated} extend the scale of the convolution grid, improving the effective receptive field \cite{luo2016understanding}, but rely on a predetermined uniform scale, lacking adaptability to data variations. Spatial transformer networks (STN) \cite{jaderberg2015spatial,lin2017inverse} employ data-dependent parametrized transformations, such as affine transforms, for feature sampling. Yet, STNs are constrained to predefined transformation families and apply transformations uniformly across the entire feature map. Active convolutions \cite{jeon2017active} introduce learnable offsets for anisotropic sampling grids but share these offsets across all pixel locations and without data adaptability at inference. Deformable convolutions \cite{dai2017deformable,zhu2019moredeformable} address these issues by learning signal-dependent offsets at each pixel location. In \cite{yu2022entire}, undesirable offset interactions are recognized, leading to entire deformable convolutions which introduce a single input and location-dependent offset for the entire sampling grid. Another approach, combining dilation and deformable convolution \cite{li2020anisotropic}, scales the grid to be deformed non-uniformly and signal-dependently. This involves applying multiple deformable convolutions at fixed scales followed by pooling to select the optimal dilation. Unlike other methods discussed here, this approach requires aggregating across multiple deformed convolutions, offering adaptability only within a fixed set of candidate scales.
While all these methods have demonstrated improved performance in handling deformations and spatial variations, the theoretical framework underpinning these deformations remains elusive, hindering a comprehensive understanding of their capabilities and limitations.

In response, this paper takes a novel perspective by introducing a unifying framework rooted in metric theory. Metric theory treats images as manifolds endowed with a metric, allowing for the computation of distances and neighbourhoods that can deviate from those induced by the typically used Euclidean norm in the image plane. This perspective enables us to reinterpret existing modern convolutions as weighted filtering of samples from the unit ball of a latent metric, effectively linking convolution approaches to implicit metrics. Such an idea echoes yet differs from existing methods working directly on graphs and surfaces \cite{boscaini2016learning}. Metric convolutions extend this concept by explicitly incorporating signal and location-dependent parameterised metrics, offering an interpretable, versatile, and adaptable approach to deformable convolutions (see \cref{fig: duck motivation adaptive kernels}).

The main contributions of this work are threefold:

\begin{itemize}
    \item A unifying metric theory that provides geometric interpretability to both existing convolutions and the CNNs employing them.
    \item Introducing metric convolutions, a novel convolution that can be anisotropic and asymmetric, based on Finsler geometry, that deforms convolution kernels based on explicit, adaptable, and interpretable metrics on the image manifold, promising robustness and versatility.
    \item The explicit interpretable geometric bias of metric convolutions enables
    their direct application to full-resolution images outside of neural networks, rather than solely deep inside a CNN. They are also compatible with neural network architectures. 
\end{itemize}

In the following sections, we provide preliminaries on metric geometry (\cref{sec: preliminaries metric geometry}), present our unifying metric theory to convolutions (\cref{sec: unifying metric theory to convs}), and explore metric convolutions (\cref{sec: metric convs}), their theoretical foundations, practical implementation, and empirical evaluations (\cref{sec: experiments}).

\section{Preliminaries on Metric Geometry}
\label{sec: preliminaries metric geometry}

\subsection{The concept of distance}

In this work, we reinterpret images as parametrised surfaces on the unit plane $\Omega$ and explore them from a metric perspective, as was popular prior to the rise of deep learning \cite{perona1990scale,rudin1992nonlinear,lindenbaum1994gabor,caselles1995geodesic,cohen1997global,sochen1998general,sochen2001diffusions,buades2005review}.
Metric geometry focuses on curved spaces, or manifolds, denoted as $X$, each with well-defined tangent planes $T_xX$ at every point $x\in X$. Manifolds are equipped with metrics, positive functions on the tangent bundle $X\times T_xX\to \mathbb{R}_+$ for local distance calculations (see \cref{fig: Finsler unit tangent balls}). This enables computing curve lengths, geodesic curves, and geodesic distances between points. Different metrics lead to varying geodesic curves and distances. Henceforth, images are two-dimensional manifolds, i.e.\ surfaces, 
such as greyscale intensity height maps. 

\begin{figure}[t]
    \centering
    \includegraphics[width=0.9\columnwidth]{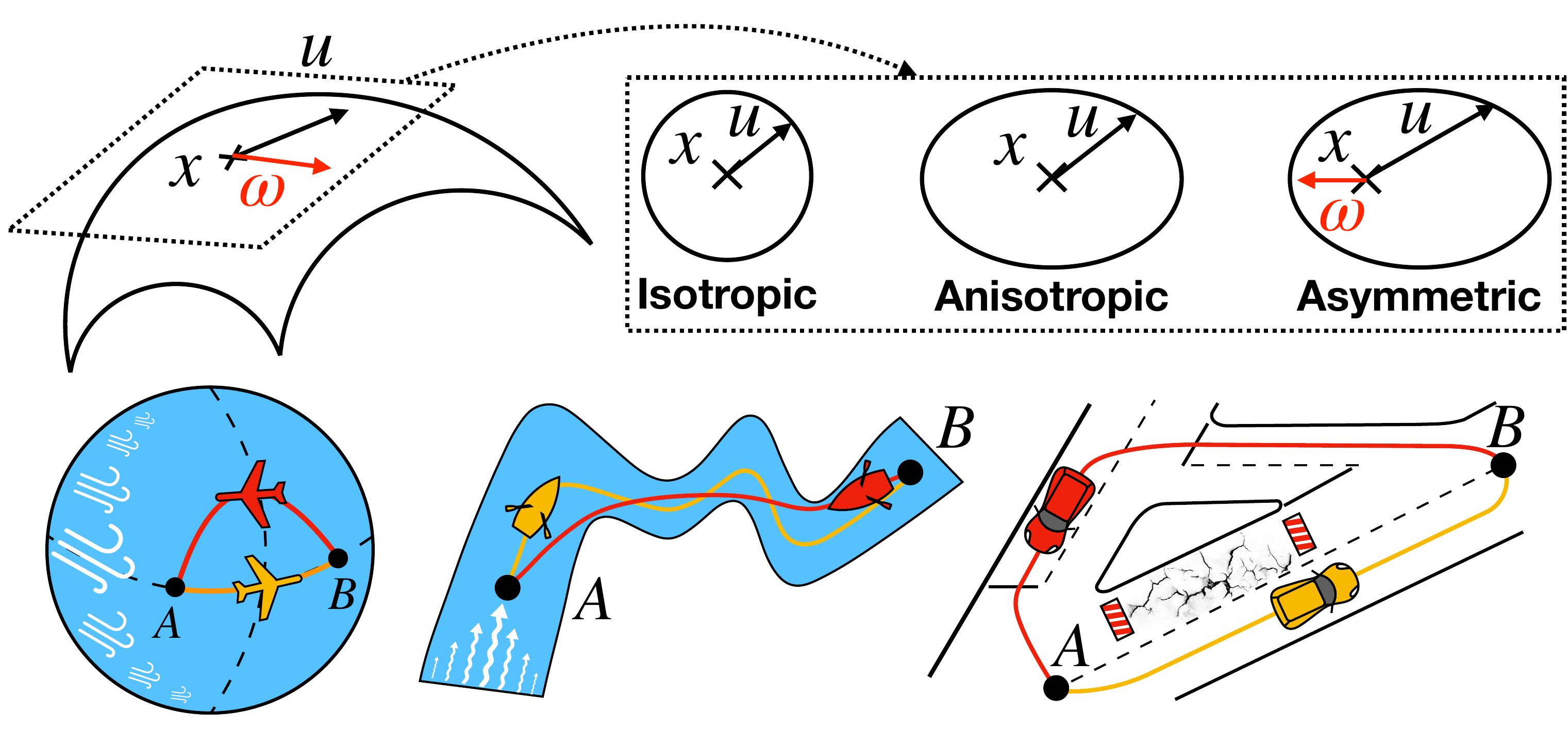}
    \vspace{-0.5em}
    \caption{Metrics provide distance measurements in the tangent space (top), and can be fully described via the shape of their convex unit tangent balls. Riemannian metrics can be isotropic or anisotropic, but they remain symmetric as the unit tangent balls are centred. Finsler metrics generalise them by allowing asymmetry, leading to asymmetric geodesics (bottom -- courtesy of \cite{weber2024finsler,dages2025finsler}).}
    \label{fig: Finsler unit tangent balls}
    \vspace{-1em}
\end{figure}

\paragraph{Riemann.}
The common metric choice is the Riemannian metric $R$, defined by a $2\times 2$ positive definite matrix $M(x)$ at each point $x$, inducing a scalar product on the tangent plane. Formally, the Riemannian metric is $R_x(u) = \sqrt{u^\top M(x) u}$  at point $x$ and tangent vector $u\in T_xX$. 
It 
yields symmetric distances, making traversal direction irrelevant. 
The metric is isotropic if $M$ is everywhere a scaled identity matrix, and it is non-uniform yet isotropic if the scale
differs between points, but it is at times mistakenly called anisotropic \cite{perona1990scale}.

\paragraph{Finsler.} 
While Riemannian metrics create symmetric tangent space neighbourhoods, Finsler metrics provide a generalisation allowing for asymmetric ones. For example, near object boundaries, asymmetric neighbourhoods can prevent neighbours from extending outside the object (see \cref{fig: duck motivation adaptive kernels}). Finsler metrics $F$ define Minkowski norms on tangent spaces, where $F_x(-u)$ and $ F_x(u)$ may differ for $u\in T_xX$ at point $x$. Consequently, local and global distances no longer adhere to symmetry. Formally, a Finsler metric $F$ satisfies $F_x(u)=0$ if and only if $u=0$, it obeys the triangular inequality, and is positive-homogeneous, i.e., $F_x(\lambda u) = \lambda F_x(u)$ for any $\lambda>0$ (see \cref{sec: metric axioms}).

\paragraph{Randers.} 
General Finsler metrics are not parametrisable, leading us to consider a  common \cite{weber2024finsler,dages2025finsler} subset: Randers metrics. They combine a Riemannian metric, parameterised by $M$, with a linear drift component, parameterised by $\omega\in T_xX$. Formally, a Randers metric with parameters $(M,\omega)$ is defined as $F_x(u) = \sqrt{u^\top M(x) u} + \omega(x)^\top u$. For metric positivity, we require $\lVert \omega(x)\rVert_{M^{-1}(x)} < 1$ (details in \cref{prop: randers metric positivity}). Notably, Riemannian metrics result from $\omega \equiv 0$. Other choices for parametric Finsler metrics are possible, like the $(\alpha,\beta)$ ones, inducing Kropina or Matsumoto metrics \cite{javaloyes2011definition}.

\hfill\break
Henceforth, the metric, most generally Finsler, is written $F$.

\subsection{Unit Balls}

In metric geometry, a key concept is the unit ball at point $x$, a collection of objects within unit distance of $x$ according to the metric $F$.
Depending on the context, the unit ball may refer to different objects. 

\paragraph{Unit Tangent Ball.}
Sometimes, unit balls focus on the tangent plane. The unit tangent ball $B_1^t(x)$ at point $x$ is the set $B_1^t(x) = \{u\in T_xX;\; F_x(u) \le 1\}$. For any metric satisfying the triangular inequality, the unit tangent ball (UTB) and its projection onto the image plane are convex sets. In particular, the UTB of a Randers metric is an ellipse, off-centred if $\omega(x)\neq 0$, with equation $u^\top M_x u = (1 - \omega(x)^\top u)^2$. A higher norm of $\omega(x)$ results in less symmetry. In image manifolds, we often project tangent planes onto the image plane and associate the UTB on the tangent space with its projection. If $x$ represents a pixel coordinate, we informally say that $B_1^t(x)$ is the set of pixel coordinates $B_1^t(x) = \{y\in \Omega;\; F_{X(x)}(y-x)\le 1\}$.

\paragraph{Unit Geodesic Ball.}
The unit geodesic ball $B_1^g(x)$ contains all points within a unit geodesic distance from $x$: $B_1^g(x) = \{y\in X;\; \mathrm{dist}_F(x,y) \le 1\}$, where $\mathrm{dist}_F(x,y)$ is the geodesic distance determined by the metric $F$. It is the minimum length of a smooth curve from $x$ to $y$ (see \cref{sec: finsler geodesic distances}). Computing unit geodesic balls (UGB) is challenging as they lack closed-form expressions, unlike UTBs, requiring instead to integrate the metric along geodesics. They can take diverse shapes and may not be convex, especially after projection onto the image plane. As for the UTB, we associate the UGB with its projection on the image plane: $B_1^g(x) = \{y\in \Omega;\; \mathrm{dist}_F(X(x), X(y)) \le 1\}$.

\section{A Unifying Metric Theory to Convolutions}
\label{sec: unifying metric theory to convs}

For simplicity, we assume images are continuous and single-channel, with our theory extendable to multi-channel data. Our domain is $\Omega= [0,1]^2$, and images are two-dimensional signals $f:\Omega\to \mathbb{R}$. 
Discrete images are viewed as samples of the continuous domain, with interpolation (e.g.\ bilinear) for querying non pixel centre position. Note that presenting image convolutions and their anisotropic variants in the continuum is unusual in the neural network community. 

While we focus on images, our theory generalises to higher
manifold or feature dimensions (see \cref{sec: generalising unifying metric theory to arbitrary manifolds}).

\subsection{Pre-existing Convolutions}
\label{sec: pre-existing convs}

\subsubsection{Traditional Fixed Support Convolutions}

The convolution $f*g$ of a signal $f:\Omega\to\mathbb{R}$ with a kernel $g:\Omega\to\mathbb{R}$ is traditionally\footnote{The convolution notion we use is sometimes called correlation.} defined as
\begin{equation}
    \label{eq: conv}
    (f * g)(x) = \int_\Omega f(x + y) g(y) dy.
\end{equation}
Padding allows querying outside $\Omega$. 
Often, the kernel function $g$ is assumed to be localised on a small compact support
$\Delta$,
usually containing $0$.
The convolution then becomes
\begin{equation}
    \label{eq: conv local}
    (f * g) (x) = \int_\Delta f(x+y) g(y) dy.
\end{equation}

Henceforth, we focus on local convolutions. Traditional convolutions assume a fixed kernel support $\Delta$. The weights, i.e., the values of $g$, are either predefined or learnt.
In the discrete world, $\Delta$ is usually localised around and includes the entry $0$. It is universally discretised into a small $k\times k$ odd square grid, often $3\times 3$, a prevalent choice in CNNs. We use the word \textit{support} and $\Delta$-based notations for both continuous and sampled discrete convolution supports.

\subsubsection{From a Fixed to a Changeable Support}

We present three major variations of traditional convolutions based on the shape of the support $\Delta$. While other approaches like \cite{li2020anisotropic} may offer different implementation and optimisation strategies, their kernel support formulation typically fits into one of these categories. They originate from prior work in the discrete world. Starting from a $k\times k$ grid forming a reference support $\Delta^{\textit{ref}}$, each grid cell is shifted to create a new set of $k^2$ points, forming the modified support $\Delta$. While not traditional convolutions, they maintain the weight sharing principle, with kernel weights $g(y)$ independent of the convolution position $x$, even if the positions of the support samples $y$ may vary.

\begin{figure*}[t]
    \centering
    \centerline{\includegraphics[width=0.95\textwidth]{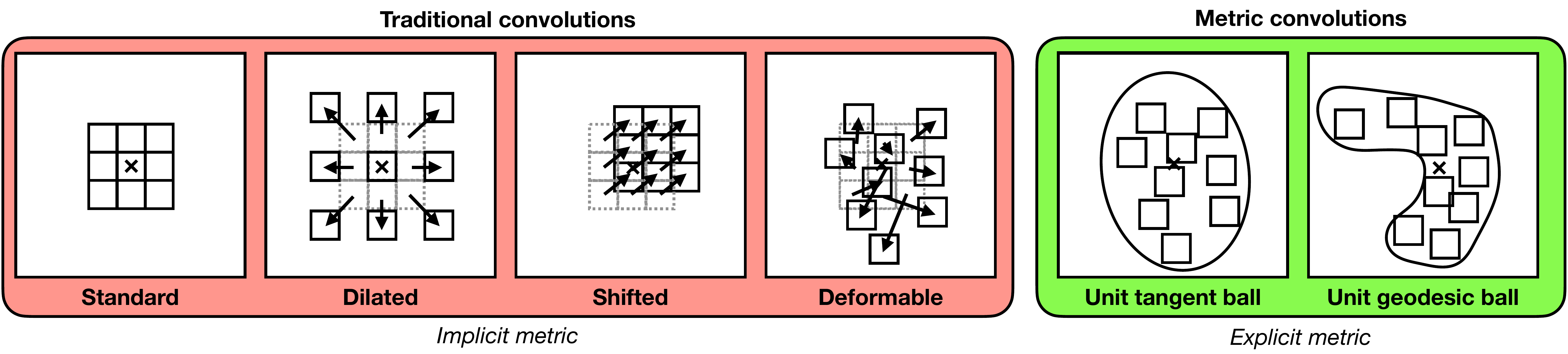}}
   \caption{
   Traditional convolutions, standard or with kernel support deformations, act as weighted averages over unit ball samples of implicit metrics. This theory unifies the various existing convolutions 
   and leads to metric convolutions, a novel convolution paradigm, that samples unit balls of explicit metrics, either learnt or provided by geometric heuristics.
   The unit balls refer to those in the tangent space or to the geodesic ones on the image manifold, both projected onto the image plane.
   }
   \label{fig: overview unit ball conv}
   \vspace{-1em}
\end{figure*}

\paragraph{Dilated Convolutions.} 
Dilated convolutions, introduced by \cite{chen2015semantic,chen2017deeplab} for neural networks, uniformly scale cell positions in the reference support $\Delta^{\textit{ref}}$ by a dilation factor $s$\footnote{The dilation $s$ represents spacing between kernel elements.}, e.g.\ with $k=3$ and $s = 4$, the dilated support includes pixels at indices $\{(i,j);\; i,j \in \{-4, 4, 0\}\}$. Non-standard non-uniform yet isotropic dilated convolution can use different scales $s_x$ per pixel for scale-sensitive filters \cite{li2020anisotropic}.

\paragraph{Shifted Convolutions.}
Also called entire deformable convolutions \cite{yu2022entire}, they shift the entire reference kernel support $\Delta^{\textit{ref}}$ by a single offset vector $\delta_x$ per pixel $x$.

\paragraph{Deformable Convolutions.}
Devised by \cite{dai2017deformable}, these convolutions use different offsets $\delta_x^y$ for each cell in the reference support $\Delta^{\textit{ref}}$, inducing complex deformations.

\hfill\break
Signal-dependent deformations, as in \cite{dai2017deformable}, make deformable convolutions, both entire and regular, and non-standard dilated convolution with adaptive scale \cite{li2020anisotropic} non-linear operations. However, computing deformations by traditional convolutions, as in shifted and deformable convolutions, leaves them shift-equivariant. All convolutions presented so far can be seen as specific deformable convolutions with special offset choices. Beyond support deformation, other modifications to the traditional convolution have emerged, such as the modulation idea.

\paragraph{Modulation.}
Breaking the weight sharing assumption, kernel weights $g(y)$ can be modulated by 
$m_x(y)\in[0,1]$ depending on $x$ \cite{zhu2019moredeformable,yu2022entire,romero2021flexconv}, see \cref{sec: modulation and weight sharing assumption}.

\hfill\break
Generally, non-standard convolution strategies are largely empirical. This diverse field can seem like a multitude of performance-driven tricks lacking a unified theoretical framework. We propose to present and generalise these methods in a systematic way from a metric perspective.

\subsection{Unifying Convolutions Through a Metric Lens}

All previous convolutions can be expressed as follows.

\begin{theorem}
    \label{th: reformulation convs}
    If convolution refers to standard, dilated, shifted, or deformable convolutions, then the convolution of a signal $f$ with a kernel $g$ can be expressed as
    \begin{equation*}
        (f*g)(x) = \int_{\Delta_x} f(x+y)g(y)dm_x(y),
    \end{equation*}
    where $dm_x(y) = m_x(y)dy$ is a distribution with density $m_x(y)\in[0,1]$, and $\Delta_x$ is a local support depending on $x$ (and sometimes also on $f$), and it is given by a transformation of a reference local support $\Delta^{\textit{ref}}$. In the absence of modulation, then $dm_x(y) = dy$.
\end{theorem}

See \cref{sec: proof reformulated convs} for a proof. Convolutions discussed so far perform weighted averaging over local neighbourhoods $\Delta_x$, eventually learnt. Modulation 
breaks uniformity
in $x$, and it can be seen from a discrete perspective as a non-uniform sampling probability distribution of $\Delta_x$.

Viewing images as metric manifolds, local neighbourhoods $\Delta_x$ can be reinterpreted as the points within a small distance of $x$  for some metric $F$, taken to be $1$ without loss of generality by scaling the metric. Thus, we see $\Delta_x$ as a unit ball. This suffices to reveal the existence of implicit metrics harnessed by convolutions, according to the following well-known theorem \cite{javaloyes2011definition}. 

\begin{theorem}
    \label{th: metric defined unit balls}
    A metric is uniquely determined by its unit tangent balls.
\end{theorem}

A proof is provided in \cref{sef: proof metric unique tangent balls} for completeness, and we push to \cref{sec: reconstructing metric from UGB} a discussion on how to reconstruct the metric from unit geodesic balls. We thus understand the multitude of existing anisotropic convolutions as weighted signal averages in unit balls of implicit, possibly non parametric, metrics $F$. In this context, modulation acts as a non-uniform distribution for sampling the unit balls.

In practice, we reinterpret discrete convolution kernels $\Delta_x$ as finite samples of unit balls. With finite samples, metric uniqueness is lost (see \cref{sec: example non unique metric for discrete samples}), yet convolutions described so far (see \cref{fig: overview unit ball conv}) can be approximately explained using Randers metrics. Other interpretations are possible with different metrics, potentially more suitable as Randers UTBs are confined to ellipses. Standard and dilated convolutions imply an underlying scaled isotropic Riemannian metric. The kernel shift in shifted convolutions can be modelled with the drift component $\omega$ of a Randers metric. All three of these can be understood from the perspective of UTBs. Due to constraints on the offset magnitudes, deformable convolution supports usually resemble sampled convex set, which likens to sampling the UTB of a Randers metric. However, UGB sampling is a better interpretation for more involved theoretical deformations resembling the sampling of non-convex sets (see \cref{sec: implicit metrics of existing convolutions}).

\section{Metric Convolutions}
\label{sec: metric convs}

With our new metric perspective, existing convolutions are weighted averages over unit balls of an underlying metric, potentially with non-uniform sampling.
This understanding allows us to change the convolution paradigm by explicitly constructing a metric $F$ and deriving its unit balls for the filtering operation (see \cref{fig: trad vs metric conv pseudo algo,fig: summary adaptive conv algo details}). Unit balls should result from a transformation of a reference set, e.g.\ $\Delta^{\textit{ref}}$, where weights (output values of $g$) are well-defined.
Termed \textit{metric convolutions}, they generalise existing convolutions relying on implicit metrics by using explicit ones, offering interpretability, strong geometric priors as implicit regularisation, and a fixed small number of parameters for metric encoding, regardless of the number of samples.

We provide simple, differentiable metric convolution implementations compatible with gradient-based optimisation. We sample unit balls uniformly for simplicity, but our constructions generalise to modulation with a parameterised density $m_x(y)$, e.g.\ Gaussian \cite{romero2021flexconv}. Our main focus is on UTB-based metric convolutions, with an option for UGB-based ones. 
The general idea for our proposed methods is the following. 
Consider a class of parametric metrics $F^\gamma$, which smoothly depend on location-dependent parameters $\gamma$. For instance, we take for Riemannian metrics $\gamma = M$, and for Randers metrics $\gamma=(M,\omega)$. Given $\gamma$, explicitly compute unit balls to define (sampled) neighbourhoods $\Delta_x$ for calculating weighted averages. We can choose metric parameters $\gamma$ based on geometric or image-related properties and can opt for interpretable filtering methods like fixed uniform kernels $g$. Both kernel weights $g$ and metric parameters $\gamma$ can be learnt through gradient-based optimisation.

Like our theory, metric convolutions are generalisable to higher manifold or feature dimensions (see \cref{sec: generalising metric convolutions to arbitrary manifolds}).

\begin{figure}[ht]
    \centering
        \centerline{\includegraphics[width=\columnwidth]{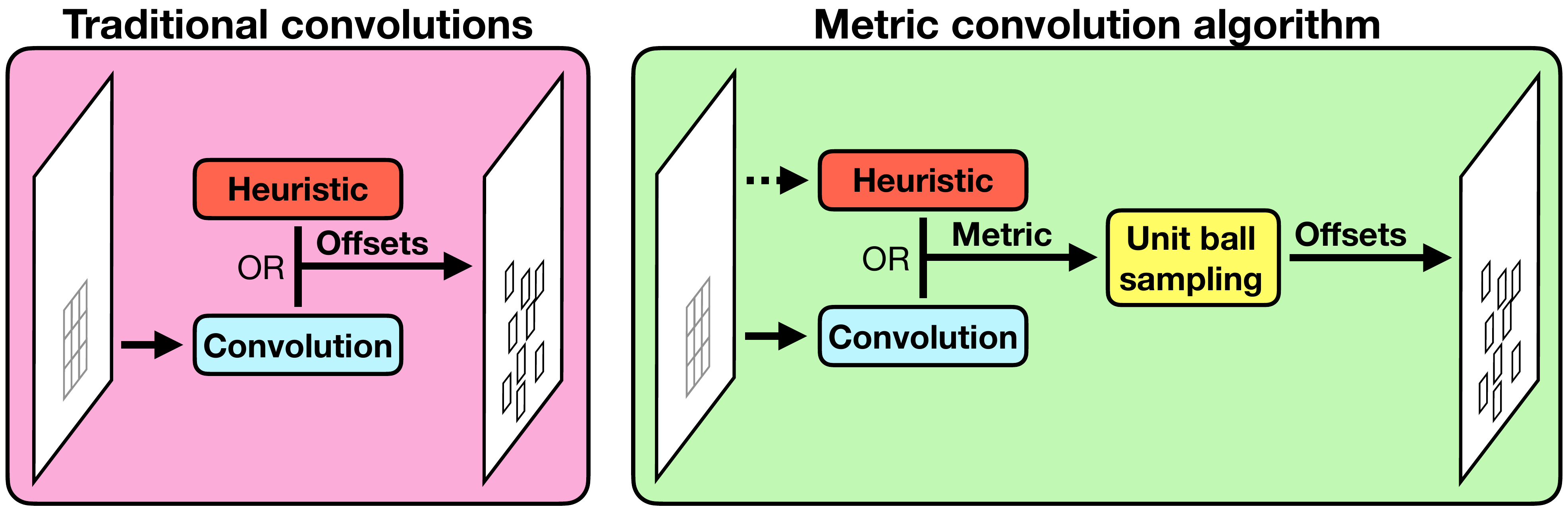}}
        \caption{
        Our metric convolutions compute metrics explicitly for unit ball sampling, while traditional convolutions compute offsets directly. In non-heuristic methods, our intermediate convolution encodes general metrics with only $5$ to $7$ numbers, compared to $2k^2$ for the offsets in a $k\times k$ deformable convolution. Our heuristic methods are compatible with adaptive signal-dependent sampling.
        }
        \label{fig: trad vs metric conv pseudo algo}
\end{figure}

\subsection{Unit Tangent Ball Metric Convolution}

The unit ball $B_1^t(x)$ is obtained by deforming the unit Euclidean disk. Recall that unit tangent balls (UTB) are convex, e.g.\ ellipses for Randers metrics. Thus, we can use angles $\theta$ as a monotonic parameterisation of the unit circle.
Let $y_x(\theta,\gamma)\in B_1^t(x)$ be the point on the unit circle at angle $\theta$ for the Finsler metric of parameters $\gamma$. That is, 
$y_x(\theta,\gamma) = \lVert y_x(\theta,\gamma)\rVert_2 u_\theta$, where $u_\theta = (\cos \theta, \sin\theta)^\top$.
The  $F_x^\gamma$-unit vector, $y_x(\theta,\gamma)$, satisfies $F_x^{\gamma}(y_x(\theta,\gamma)) = 1$. The positive homogeneity of the metric implies that
$\lVert y_x(\theta,\gamma)\rVert_2 = \tfrac{1}{F_x^{\gamma}(u_\theta)}$.
The unit circle is thus given by the points
\begin{equation}
    \label{eq: finsler unit circle}
    y_x(\theta,\gamma) = \frac{1}{F_x^{\gamma}(u_\theta)}u_\theta.    
\end{equation}
With polar coordinates for 
integrals,
convolution becomes
\begin{equation}
    \label{eq: metric utb f*g}
    (f*g)(x) = \int_{s,\theta} f(x + s y_x(\theta,\gamma)) g(s y_x(\theta,\gamma)) ds d\theta.
\end{equation}

For a signal-dependent metric, we can compute $\gamma$ directly from the image, e.g.\ with a standard convolution (as in deformable convolution) we name the \textit{intermediate convolution}. 
Since $M$ is symmetric definite positive, it is parameterisable by $3$ numbers using its Cholesky decomposition $LL^\top = M$ where $L$ is a lower triangular matrix with positive diagonal. However, we also used a more stable spectral implementation using extra numbers when necessary, where two numbers encode the first eigenvector of $M$ and the two others are its eigenvalues, eventually adding one more number encoding the scale of the eigenvalues.
As $\omega$ requires two numbers, our intermediate convolution only needs $5$ to $7$ output channels to compute $\gamma$, depending on
which implementation we choose,
unlike deformable convolution which requires $2k^2$ channels for the offsets of $k\times k$ samples. In \cref{sec: computing metric from 5 numbers,sec: computing metric from 6 7 numbers}, we explain how to recover $\gamma = (M, \omega)$ from these 
numbers while enforcing $\lVert \omega \rVert_{M^{-1}(x)} < 1-\varepsilon_\omega < 1$, where $\varepsilon_\omega \in (0,1]$ is a hyperparameter controlling the maximum tolerated asymmetry.

By analogy with existing convolutions, we can simply discretise the UTB by sampling $k^2$ points using polar coordinates $s\in[0,1]$ and $\theta\in[0,2\pi]$ (see \cref{sec: polar sampling strategies}), yielding support locations $P_x^\gamma(s,\theta) = sy_x(\theta,\gamma)$ for $\Delta_x$. While straightforward, this implementation is compatible with common neural networks and can be differentiated as explained in \cref{sec: differentiating metric utb conv}.

\noindent
Notably, metric convolutions satisfy the following property.

\begin{theorem}
    \label{th: shift equivariant}
    Metric convolutions are shift-equivariant.
\end{theorem}

We provide a detailed proof in \cref{sec: shift-equivariance metric convolutions}. The reason is that we use shift-equivariant offsets for sampling the image, thanks to shift-equivariant heuristics or standard convolutions as intermediate operations to extract metric parameters, and that we satisfy the weight sharing assumption, as the values taken by $g$ are the same for each pixel $x$. Invariance though is achieved by combining metric convolutions with pooling, as in standard convolutions.

\subsection{Unit Geodesic Ball Metric Convolution}

Efficient geodesic metric convolution implementations are challenging due to costly geodesic extraction (see \cref{sec: difficulties fast diff UGB}). We propose a faster approximation inspired by \cite{crane2013geodesics}. It computes geodesic flow fields using normalised gradients of initial Finsler heat flows at pixel $x$ from a simplified local solution with the Finsler-Gauss kernel \cite{ohta2009heat,yang2018geodesic}. Then, a set of candidate sample locations initially near $x$ flows along the geodesics for a fixed duration, providing samples of the unit geodesic ball. See \cref{sec: details implementation metric ugb conv} for full details. Although faster than accurate geodesic extractors \cite{varadhan1967behavior,benmansour2010derivatives,crane2013geodesics,bertrand2023fast} and fully differentiable, it is not fast enough for real scenarii such as neural network modules.

\section{Experiments}
\label{sec: experiments}

\subsection{Denoising Whole Images with a Single Local Convolution}
\label{sec: denoising whole images with a single local convolution}

Here, we focus on image denoising, a fundamental task for convolutions as they filter out noise. Surprisingly, modern non-standard convolutions are untested on it, as they are only used in later stages of neural networks for feature learning. 
We 
add Gaussian noise of standard deviation $\sigma_n$ to greyscale images and we set the following constraints. 

\begin{itemize}
    \item All methods use $k\times k$ kernel samples with same $k$.
    \item All methods apply a single filtering convolution with interpolation allowed for non-pixel centre sampling.
    \item Convolutions are local. Non-local ones are excluded, 
    e.g.\ \cite{tomasi1998bilateral,dabov2007image,buades2011nonlocal}.
\end{itemize}

We explore three filter variations
and learning is done either on a single image or on a dataset.

\paragraph{Geometric heuristic design.}
To geometrically design filters, we recall that to preserve edges, neighbourhoods should not cross them.
Thus,
the eigenvectors of $M$ are taken as the image gradient $\nabla f(x)$ and its orthogonal $\nabla f(x)^\perp$. Anisotropy is induced by different eigenvalues: a smaller one stretches the unit ball further in its eigendirection due to the inverse in \cref{eq: finsler unit circle}. 
For denoising, a natural choice for the Randers drift component is $\omega\equiv0$,  yet we also test with different asymmetry levels by aligning $\omega$ with $\nabla f(x)^\perp$ 
at various scales controlled by a hyperparameter $\varepsilon_\omega$. If $\varepsilon_\omega=1$, then $\omega \equiv 0$.
We visualize some unit balls in \cref{fig: kernels cameraman}. Notice how UTBs remain convex, while UGBs deform more. 
Our method gets better filtering (higher PSNR), compared to other manually designed approaches.
Further results and  implementation details can be found in \cref{sec: implementation heuristic metric convs}.

\begin{figure}[ht]
    \centering
        \centerline{
        \includegraphics[width=\columnwidth]{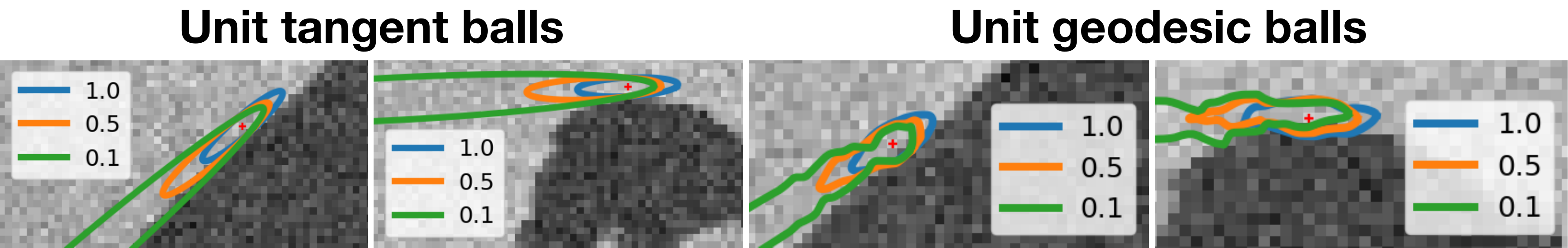}
        }
        \caption{Heuristic unit balls, tangent (left) and geodesic (right). Convolution kernel locations are sampled in them. The different balls correspond to 
        various
        values of $\varepsilon_\omega$ controlling the scale of $\omega$.
        }
        \label{fig: kernels cameraman}
    \vspace{-2em}
\end{figure}

\paragraph{Learning sample locations only.}
As an intermediate experiment, between manual design and learning from a training set, we considered learning the metric parameters of UTBs and the offsets of deformable convolutions when applied to a single image. Our method is systematically on par or outperforms deformable convolution (\cref{tab: mse res single im dataset deform unit tangent no nn small version main manuscript}). 
As our UTB samples are always derived from a metric, yielding a valuable geometric regularisation, our approach remarkably does not overfit, with low and consistent generalisation errors, in contrast to deformable convolution's severe overfitting due to the absence of prior knowledge on the structure of the kernel support.
For more details and more quantitative and visual results see \cref{sec: learning filtering on a single image}.

\begin{table*}[t]
    \Huge
    \centering
        \resizebox{0.8\textwidth}{!}{%
            \begin{sc}
                \begin{tabular}{c ccccc c ccccc c ccccc}
                    \toprule
                     & \multicolumn{5}{c}{\bf{Deformable}} & & \multicolumn{5}{c}{\bf{Unit tangent ball (ours)}  $[\varepsilon_\omega = 0.9]$} & & \multicolumn{5}{c}{\bf{Unit tangent ball (ours)}  $[\varepsilon_\omega = 0.1]$} \\
                      \cmidrule{2-6} \cmidrule{8-12} \cmidrule{14-18} 
                    $k$ & 5 & 11 & 31 & 51 & 121 & & 5 & 11 & 31 & 51 & 121 & & 5 & 11 & 31 & 51 & 121\\
                    \midrule
                     $\textit{MSE}$
                     & \pmb{7.26e-3} & 8.97e-3 & 1.99e-2 & 2.83e-2 & 6.18e-2 & & 8.13e-3 & \underline{7.43e-3} & \underline{7.58e-3} & \underline{8.64e-3} & \underline{8.28e-3} & & \underline{8.10e-3} & \pmb{7.32e-3} & \pmb{7.45e-3} & \pmb{7.52e-3} & \pmb{7.82e-3} \\
                     $\delta_{\textit{MSE}}$
                     & 265 & 74 & 28 & 18 & 6.6 & & 1.1 & 0.9 & 1.1 & 0.8 & 1.2 & & 1.3 & 1.1 & 1.3 & 1.4 & 1.5 \\
                     \bottomrule
                \end{tabular}
            \end{sc}
        }
    \caption{
    Single image denoising results (with noise level $\sigma_n=0.3$). We also provide the normalised generalisation gap $\delta_{\textit{MSE}} = \frac{\textit{MSE}_{\textit{test}} - \textit{MSE}_{\textit{train}}}{\textit{MSE}_{\textit{train}}}$. The parameter $\varepsilon_\omega$ controls the metric asymmetry.
    }
    \label{tab: mse res single im dataset deform unit tangent no nn small version main manuscript}
\end{table*}


\paragraph{Learning filtering on a small dataset.}
In these experiments, we train a single-layer convolution network on standard datasets of $256\times 256$ greyscale images from BSDS300 \cite{martin2001bsds} and PascalVOC Segmentation \cite{everingham2010pascal}. Gaussian noise with standard deviation $\sigma_n\in\{0.1, 0.3,0.5\}$ is added to the images consistently within each experiment. 
We follow the methodology of \cite{dai2017deformable}, making positional parameters, i.e.\ deformation offsets or metric parameters $\gamma$, data-dependent and shift-equivariant by computing them from an intermediate standard $k\times k$ convolution layer (see \cref{fig: trad vs metric conv pseudo algo}), with the same $k$ as in the final convolution with deformed support. 
This intermediate convolution has as many output channels as the required number of parameters per pixel.
Deformable convolution needs $2k^2$ intermediate channels, while our UTB approach uses only $5$ for any $k$ as we use the Cholesky-based implementation. The final convolutions can have fixed uniform kernel weights (FKW) or learnable kernel weights (LKW). We test our method with $\varepsilon_\omega\in\{0.1,0.9\}$, controlling the metric asymmetry.
We provide full training details in \cref{sec: training details denoinsing dataset}.

Results in \cref{tab: mse and delta mse res small dataset deform unit tangent no nn} for learning filtering on datasets show MSE scores and generalisation gaps $\delta_{\textit{MSE}}$.
Our results are on par with deformable convolution, despite our metric UTB convolutions, constrained to ellipses, being theoretically less general. Increasing asymmetry using $\varepsilon_\omega$ improves performance. Surprisingly, fixing or training kernel weights provides similar results for both methods.
Thus, perhaps learning the shape of the kernel is as vital as learning the weights. However, it could also be due to differences in gradient magnitudes for sampling location parameters, as indicated by suggested learning rates (see \cref{sec: lr denoising}).

\begin{table*}[ht]
    \LARGE
    \centering
    \resizebox{0.85\textwidth}{!}{%
        \begin{sc}  
            \begin{tabular}{ccc ccc c ccc c ccc c ccc c ccc c ccc}
                \toprule
                     & & & \multicolumn{7}{c}{\bf{Deformable}} & & \multicolumn{15}{c}{\bf{Unit tangent ball (ours)}}\\
                     & & &\multicolumn{3}{c}{} & & \multicolumn{3}{c}{} & & \multicolumn{7}{c}{$\varepsilon_\omega=0.1$} & & \multicolumn{7}{c}{$\varepsilon_\omega=0.9$} \\
                     \cmidrule{12-18} \cmidrule{20-26}
                     & & &\multicolumn{3}{c}{FKW} & & \multicolumn{3}{c}{LKW} & & \multicolumn{3}{c}{FKW} & & \multicolumn{3}{c}{LKW} & & \multicolumn{3}{c}{FKW} & & \multicolumn{3}{c}{LKW} \\
                      \cmidrule{4-6} \cmidrule{8-10} \cmidrule{12-14} \cmidrule{16-18} \cmidrule{20-22} \cmidrule{24-26} 
                    & & \diagbox[innerleftsep=0em,innerrightsep=0em]{$\sigma_n$}{$k$} & 5 & 11 & 31 & & 5 & 11 & 31 & & 5 & 11 & 31 & & 5 & 11 & 31 & & 5 & 11 & 31 & & 5 & 11 & 31 \\
                    \midrule
                     \multirow{6}{*}{$\textit{MSE}$} & \multirow{3}{*}{BSDS300} & $0.1$  & 1.12e-4 & 1.83e-4 & 1.72e-3 & & 1.72e-4 & 1.70e-4 & 1.29e-3 & & 1.19e-4 & 1.19e-4 & 7.20e-4 & & 1.64e-4 & 1.32e-4 & 7.72e-4 & & 1.20e-4 & 1.24e-4 & 7.24e-4 & & 1.35e-4 & 1.32e-4 & 7.74e-4  \\
                     & & $0.3$  & 2.92e-4 & 2.67e-4 & 1.89e-3 & & 3.47e-4 & 3.33e-4 & 2.10e-3 & & 3.62e-4 & 3.45e-4 & 2.04e-3 & & 3.66e-4 & 3.60e-4 & 2.09e-3 & & 3.64e-4 & 3.40e-4 & 2.05e-3 & & 4.01e-4 & 3.59e-4 & 2.11e-3   \\
                     & & $0.5$  & 4.91e-4 & 4.13e-4 & 3.13e-3 & & 6.44e-4 & 6.01e-4 & 3.67e-3 & & 6.67e-4 & 6.05e-4 & 3.74e-3 & & 7.15e-4 & 6.33e-4 & 3.77e-3 & & 6.66e-4 & 6.08e-4 & 3.76e-3 & & 7.03e-4 & 6.33e-4 & 3.79e-3 \\
                     \cmidrule{2-26}
                     & \multirow{3}{*}{PascalVOC} & $0.1$ & 8.59e-5 & 1.56e-4 & 1.35e-3 & & 1.02e-4 & 9.81e-5 & 7.62e-4 & & 1.01e-4 & 1.00e-4 & 7.95e-3 & & 1.06e-4 & 1.12e-4 & 7.71e-4 & & 1.02e-4 & 1.02e-4 & 7.76e-4 & & 1.17e-4 & 1.12e-4 & 7.77e-4  \\ 
                      & & $0.3$ & 2.40e-4 & 2.37e-4 & 1.84e-3 & & 3.03e-4 & 2.87e-4 & 2.14e-3 & & 3.18e-4 & 2.95e-4 & 2.27e-3 & & 3.33e-4 & 2.96e-4 & 2.05e-3 & & 3.19e-4 & 2.96e-4 & 2.28e-3 & & 3.10e-4 & 2.96e-4 & 2.06e-3 \\ 
                      & & $0.5$ & 4.28e-4 & 3.66e-4 & 3.20e-3 & & 5.77e-4 & 5.47e-4 & 4.01e-3 & & 6.05e-4 & 5.71e-4 & 4.30e-3 & & 6.23e-4 & 5.41e-4 & 3.72e-3 & & 6.37e-4 & 5.59e-4 & 4.33e-3 & & 6.20e-4 & 5.71e-4 & 3.74e-3 \\ 
                    \midrule
                     \multirow{6}{*}{$\delta_{\textit{MSE}}$} & \multirow{3}{*}{BSDS300} & $0.1$  & 0.28 & 0.14 & 0.07 & & 0.24 & 0.20 & 0.05 & & 0.18 & 0.21 & 0.02 & & 0.17 & 0.21 & 0.03 & & 0.17 & 0.21 & 0.02 & & 0.17 & 0.21 & 0.03  \\
                     & & $0.3$  & 0.23 & 0.19 & 0.001 & & 0.21 & 0.20 & 0.05 & & 0.18 & 0.22 & 0.04 & & 0.17 & 0.21 & 0.04 & & 0.18 & 0.24 & 0.04 & & 0.18 & 0.21 & 0.04   \\
                     & & $0.5$  & 0.22 & 0.16 & 0.01 & & 0.20 & 0.20 & 0.05 & & 0.19 & 0.21 & 0.05 & & 0.19 & 0.22 & 0.04 & & 0.18 & 0.21 & 0.04 & & 0.19 & 0.22 & 0.04 \\
                     \cmidrule{2-26}
                     & \multirow{3}{*}{PascalVOC} & $0.1$ & 0.03 & 0.01 & 0.11 & & 0.03 & 0.03 & 0.02 & & 0.03 & 0.03 & 0.02 & & 0.03 & 0.03 & 0.02 & & 0.03 & 0.03 & 0.02 & & 0.03 & 0.03 & 0.02  \\ 
                     & & $0.3$ & 0.03 & 0.02 & 0.01 & & 0.03 & 0.03 & 0.01 & & 0.02 & 0.03 & 0.01 & & 0.02 & 0.03 & 0.02 & & 0.02 & 0.03 & 0.01 & & 0.03 & 0.03 & 0.03 \\ 
                     & & $0.5$ & 0.01 & 0.02 & 0.08 & & 0.02 & 0.02 & 0.001 & & 0.02 & 0.02 & 0.01 & & 0.02 & 0.02 & 0.02 & & 0.01 & 0.02 & 0.01 & & 0.02 & 0.02 & 0.02 \\ 
                    \bottomrule
            \end{tabular}
        \end{sc}
    }
    \caption{
    Denoising results on the greyscale versions of BSDS300 and PascalVOC datasets with noise level $\sigma_n$. Positional parameters of the $k\times k$ convolutions are learnt with a single regular convolution. The weights of the learnt kernel are either fixed (FKW) or learnt as well (LKW). Top are test MSE losses, bottom are the generalisation gaps $\delta_{\textit{MSE}}$.
    Lower is better.
    }
    \label{tab: mse and delta mse res small dataset deform unit tangent no nn}
    \vspace{-0.5em}
\end{table*}

\subsection{From Single to Stacked Convolutions: An Example of CNN Classification}
\label{sec: from single to stacked convolutions: cnn classification}

We show how our metric convolutions can be used with neural networks, specifically CNNs. We work on the MNIST \cite{lecun1998gradient}, Fashion-MNIST \cite{xiao2017fashionmnist}, CIFAR-10 and CIFAR-100 \cite{krizhevsky2009learning} classification benchmarks.
We follow the methodology of \cite{dai2017deformable,zhu2019moredeformable,yu2022entire} to convert standard convolution modules of a CNN to deformed versions. 
We test shifted, deformable, and metric UTB convolutions against a standard convolution baseline.
Specifically, we replace the $3\times 3$ standard convolutions in the later stages, from  layer2 to layer4, of a ResNet18 \cite{he2016resnet} by our metric convolution using the same number of sample locations. 
As is common for this resolution, we set conv1 stride to $1$ and avoid pooling to retain information.
In our experiments, we either fix the kernel weights (FKW) of non-standard convolutions to uniform values or learn simultaneously sample locations and the weights (LKW). The weights, up to module conversion, are either learned from scratch (SC) or transfer learned (TL) from Imagenet \cite{deng2009imagenet} classification with vanilla modules.
Due to resource constraints, we only trained a single run on the MNIST and Fashion-MNIST datasets, but train $8$ runs on the CIFAR-10 and CIFAR-100 datasets. 
Full implementation and training details of our and other methods are provided in \cref{sec: cnn mnist implementation details}.
We also ablate in \cref{sec: Further Ablation Experiments of CNN Classification} on different replacement strategies and using Riemannian instead of Finsler metric convolutions.

\begin{table}[ht]
    \Large
    \centering
    \resizebox{0.9\columnwidth}{!}{%
        \begin{sc}
            \begin{tabular}{l l cc c cc}
                \toprule
                 &  & \multicolumn{2}{c}{FKW} & & \multicolumn{2}{c}{LKW} \\
                 \cmidrule{3-4} \cmidrule{6-7}
                 Dataset & Method & SC & TL & & SC & TL\\
                \midrule
                 \multirow{4}{*}{\bf{MNIST}}  
                 & Standard  & - & - & & 99.61\% & 99.66\% \\
                 \cmidrule{2-7}
                 & Deformable  & 83.79\% & 99.16\% & & 99.51\% & 99.64\% \\
                 & Shifted  & 99.11\% & 99.15\% & & 99.39\% & 99.64\%  \\
                 & Metric UTB (Ours)  & \pmb{99.14\%} & \pmb{99.17\%} & & \pmb{99.64\%} & \pmb{99.68\%} \\
                 \midrule
                 \multirow{4}{*}{\bf{Fashion-MNIST}}  
                 & Standard  & - & - & & 92.37\% & 93.30\% \\
                 \cmidrule{2-7}
                 & Deformable  & 75.50\% & 90.66\% & & 92.54\% & 92.74\% \\
                 & Shifted  & 82.43\% & 89.40\% & & 92.75\% & 93.45\%  \\
                 & Metric UTB (Ours)  & \pmb{89.85\%} & \pmb{90.87\%} & & \pmb{92.76\%} & \pmb{93.49\%} \\
                 \bottomrule
            \end{tabular}
        \end{sc}
    }
    
    \caption{
        Test accuracies of ResNet18 trained using standard or non-standard convolutions.
        Higher is better. The main results are those of the LKW-TL columns.
    }
    \label{tab: classif resnet18 mnist fmnist}
\end{table}

\begin{table*}[ht]
    \Huge
    \centering
        \resizebox{0.93\textwidth}{!}{%
            \begin{sc}
                \begin{tabular}{ll cc c cc c cc c cc}
                    \toprule
                     & & \multicolumn{5}{c}{\bf{CIFAR-10}} & & \multicolumn{5}{c}{\bf{CIFAR-100}}\\
                     \cmidrule{3-7} \cmidrule{9-13} 
                     & & \multicolumn{2}{c}{FKW} & & \multicolumn{2}{c}{LKW} & & \multicolumn{2}{c}{FKW} & & \multicolumn{2}{c}{LKW} \\
                     \cmidrule{3-4} \cmidrule{6-7} \cmidrule{9-10} \cmidrule{12-13}
                     & & SC & TL & & SC & TL & & SC & TL & & SC & TL\\
                     \midrule
                     \multirow{4}{*}{TOP1} 
                     & Standard  & - & - & & 92.26\% ($\pm$ 0.20\%) & 92.64\% ($\pm$ 0.18\%) & & - & - & & 70.48\% ($\pm$ 0.29\%) & 70.52\% ($\pm$ 0.38\%) \\
                     \cmidrule{2-13}
                     & Deformable  & 33.21\% \pmb{($\pm$ 2.05\%)} & 34.10\% ($\pm$6.48\%) & & 72,30\% ($\pm$7.86\%) & \pmb{93,10\%} ($\pm$0.17\%) & & 10.47\% ($\pm$ 0.79\%) & 6.82 ($\pm$1.44\%) & & 54.51\% ($\pm$8.04\%) & 70.05\% ($\pm$0.43\%) \\
                     & Shifted  & 37.89\% ($\pm$2.06\%) & 29.52\% ($\pm$17.19\%) & & 86.63\% ($\pm$4.49\%) & 92.58\% ($\pm$0.28\%) & & 14.47\% ($\pm$2.40\%) & 2.67 ($\pm$0.66\%) & & 47.14\% ($\pm$9.05\%) & 68.58\% ($\pm$0.48\%)  \\
                     & Metric UTB (Ours)  & \pmb{69.83\%} ($\pm$ 2.43\%) & \pmb{72.15\% ($\pm$ 0.98\%)} & & \pmb{91.83\%} \pmb{($\pm$ 0.18\%)} & 93.07\% \pmb{($\pm$ 0.13\%)} & & \pmb{40.63\% ($\pm$ 0.79\%)} & \pmb{38.04\% ($\pm$ 0.77\%)} & & \pmb{70.03\% ($\pm$ 0.45\%)} & \pmb{70.38\%} \pmb{($\pm$ 0.38\%)} \\
                     \midrule
                     \multirow{4}{*}{TOP5} 
                     & Standard  & - & - & & - & - & & - & - & & 90.30\% ($\pm$0.10\%) & 90.03\% ($\pm$0.22\%) \\
                     \cmidrule{2-13}
                     & Deformable  & - & - & & - & - & & 30.75\% ($\pm$1.55\%) & 23.36\% ($\pm$3.37\%) & & 78.85\% ($\pm$5.64\%) & 89.71\% ($\pm$0.32\%) \\
                     & Shifted  & - & - & & - & - & & 37.73\% ($\pm$3.72\%) & 10.67\% ($\pm$2.21\%) & & 74.34\% ($\pm$6.56\%) & 88.77\% ($\pm$0.29\%)  \\
                     & Metric UTB (Ours)  & - & - & & - & - & & \pmb{71.70\% ($\pm$0.78\%)} & \pmb{69.22\% ($\pm$0.63\%)} & & \pmb{90.42\% ($\pm$0.19\%)} & \pmb{90.23\% ($\pm$0.17\%)} \\
                     \bottomrule
                \end{tabular}
            \end{sc}
        }
    \caption{
        Median test accuracies of ResNet18 trained using standard or non-standard convolutions,
        with $8$ independent runs per configuration. Higher is better. In parenthesis is the standard deviation (lower is better). The main results are those of the LKW-TL columns.
    }
    \label{tab: classif resnet18 cifar10 cifar100}
\end{table*}

\begin{figure*}[h!]
    \centering
        \includegraphics[width=0.9\textwidth]{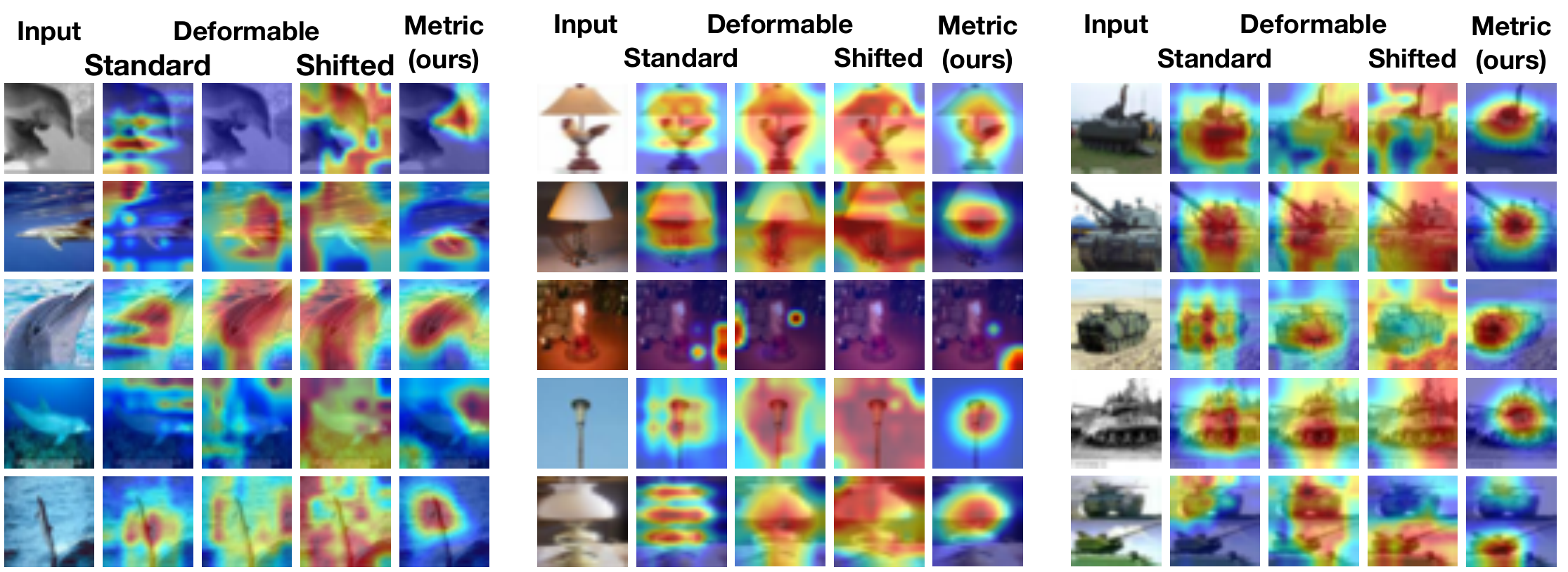}
        \caption{
        GradCAM heatmaps on random samples of the \textit{dolphin}, \textit{lamp}, and \textit{tank} classes of CIFAR-100. Our metric CNN better focuses on relevant objects and meaningful parts and can handle multiple instances unlike the other CNNs. Red means high values, blue means $0$.
        }
        \label{fig: gradcam cifar100}
    \vspace{-1em}
\end{figure*}

We provide test accuracies in \cref{tab: classif resnet18 mnist fmnist,tab: classif resnet18 cifar10 cifar100} and list the following observations.
1) Our metric CNN consistently outperforms deformable and shifted CNNs, with only rare cases where it is only on par with deformable ones.
2) Our geometric construction is a powerful prior leading to a strong implicit bias since 2)a) like standard CNNs, our metric CNN is barely affected by training from scratch (SC), occasionally improving results, and 2)b) our metric CNNs 
remain effective with fixed convolution weights (FKW). In comparison, deformable and shifted CNNs see a significant performance drop when trained from scratch, especially on CIFAR-10 and CIFAR-100, and act close to random predictors when the convolution kernels are fixed. While it could be argued that shifted convolutions has a stronger prior, since its kernel position has only 2 degrees of freedom compared to the 5 to 7 ones for our metric kernel, these results prove that a geometric approach backed by a metric theory has better regularisation than an efficient yet solely empirical approach. 
3) The results of our metric CNNs have a significantly lower standard deviation, even when they only match
deformable convolution. Thus, our CNNs are more stable, hence more reliable if limited to single training runs.
4) Our method is not much more expensive than deformable convolution (see \cref{sec: cnn mnist implementation details,tab: memory time cifar,tab: metric utb cnn vs flexconv}). It even outperforms at lower costs an advanced method with larger kernel sizes and modulation \cite{romero2021flexconv}.

We also plot GradCAM \cite{selvaraju2017grad} visualisations on randomly chosen images and classes from the CIFAR-100 test set (for LKW-TL trained networks) in \cref{fig: gradcam cifar100}. These heatmaps are the ReLU of a weighted sum of the feature representations after the last convolution layer, where the weights are given by the gradient with respect to the groundtruth label. They are then resized and superimposed on the original image.
GradCAMs highlight image regions positively contributing to the prediction of the correct class.
We see that our metric CNN better focuses than CNNs using other types of convolutions on the relevant objects and their meaningful parts, e.g.\ the face or fins of a dolphin, rather than on the background. It is also able to handle multiple instances unlike the other CNNs, rather than focusing on one of them. These visual results point towards the superiority of CNNs with metric convolutions rather than other existing types of convolutions. Not only are metric convolutions more interpretable due to the use of metric theory, but the results seem more interpretable as well.
Our experiments demonstrate how to combine metric convolutions with complex neural architectures for high level tasks and that networks benefit from our metric approach to convolutions.

\section{Conclusion}

By viewing images as metric manifolds, we provided a unified theory explaining the various existing deformed convolution operators as methods sampling unit balls of various implicit metrics. With this new perspective, we designed a novel framework for adaptive image convolutions, called metric convolutions, that sample unit balls of explicit metrics of various complexity to compute kernel locations. This approach allows for adaptive, feature-dependent neighbourhoods that capture anisotropic and asymmetric patterns, providing a geometric inductive bias improving stability and interpretability. 
Our experiments demonstrated competitive results on standard benchmarks while improving interpretability and incorporating strong geometric regularisation, albeit at a higher computational cost from the rigid standard convolution, which is unavoidable for adaptive convolution methods. Metric convolutions offer a valuable tool for tasks where explicit metrics or pixel similarities are needed, such as extracting feature-sensitive edges or tubular structures via geodesics \cite{yang2018geodesic}.
We hope that this work will help popularise the image manifold paradigm and the use of Finsler metrics to the modern computer vision and deep learning community.

\paragraph*{Limitations and future works.}
First, our choice of Randers metrics constrains unit tangent balls to ellipses. Other convex shapes would be of interest, yielding different types of Finsler unit balls. 
Second, our current implementation of unit geodesic ball metric convolutions is too slow for practical use. Improving it
and using them in neural networks is a future research direction. 
Third, our metric convolutions inherit the drawback shared by all non-standard convolutions,
which is the need to store in memory sampling offset locations per pixel. 
This high memory footprint can make learning on high resolution images challenging especially when working on high batch sizes.  
Consequentially, we focused on smaller resolution benchmarks, and leave tests on higher resolution data and more efficient implementation to future works. We also focused on a couple of tasks, yet we are intrigued by how beneficial our metric-based learned representations would be for downstream tasks. Viewing individual images as a metric manifold is largely uncharted in modern machine learning and requires further study.

\section*{Acknowledgments}
This work was supported by the Israel Science Foundation (ISF) (grant No. 3864/21), by ADRI – Advanced Defense Research Institute – Technion and by the Rosenblum Fund.

{
    \small
    \bibliographystyle{ieeenat_fullname}
    \bibliography{main}
}

\clearpage
\setcounter{page}{1}
\maketitlesupplementary

\appendix

\begin{strip}
    \vspace{-2em}
    \centering
        \resizebox{\textwidth}{!}{%
                \begin{tabular}{l l p{30cm}}
                    \toprule
                    Symbol & Type & Description \\
                    \midrule
                     $\alpha$ & $\in \mathbb{R}_+$ & Strictly positive number, either appearing in the dual Finsler metric or as the anisotropy gain factor parameter. \\
                     $B_1^g(x)$ & $\subset{X}$ & Unit geodesic ball (UGB) at point $x$.\\
                     $B_1^t(x)$ & $\subset T_xX$ & Unit tangent ball (UTB) at point $x$.\\
                     $c$ & $\in \mathbb{N}$ & Feature dimensionality. Grayscale images have $c=1$ and are the default in this work. Colour images have $c=3$, and image-like grid data can have arbitrary number of features $c$. \\
                     $c(t)$ & $\in X$ & Curve on the manifold, with $t$ varying from $0$ to $1$. \\
                     $c_{\textit{in}}$ & $\in \mathbb{N}$ & Input number of channels. \\
                     $c_{\textit{out}}$ & $\in \mathbb{N}$ & Output number of channels. \\
                     $d$ & $\in \mathbb{N}$ & Manifold dimensionality. For regular images, the manifold is a surface and $d=2$. \\
                     $\Delta$ & $\subset \Omega$ & Support of the convolution kernel. \\
                     $\Delta^{\textit{def}}_x$ & $\subset \Omega$ & Deformed kernel support when considering pixel location $x$. \\
                     $\Delta^{\textit{dil}}$ & $\subset \Omega$ & Uniformly dilated kernel support. \\
                     $\Delta^{\textit{dil}}_x$ & $\subset \Omega$ & Dilated kernel support when considering pixel location $x$. \\
                     $\Delta^{\textit{ent}}_x$ & $\subset \Omega$ & Shifted kernel support when considering pixel location $x$. \\
                     $\Delta^{\textit{ref}}$ & $\subset \Omega$ & Reference support. In the discrete world, it is typically the usual kernel $k\times k$ grid. \\
                     $\delta_{\textit{MSE}}$ & $\in \mathbb{R}_+$ & Normalised generalisation gap of the $\textit{MSE}$. \\
                     $\delta_x$ & $\in \mathbb{R}^d$ & Offset sampling vector at point $x$ for all pixels in the reference support of the convolution kernel. \\
                     $\delta_x^y$ & $\in \mathbb{R}^d$ & Offset sampling vector at point $x$ for pixel $y$ in the reference support of the convolution kernel. \\
                     $\delta_{x,t}$ & $: \Omega \to \mathbb{R}^c$ & Diffused Dirac image at time step $t$. \\
                     $\mathrm{dist}_F(x,y)$ & $\in \mathbb{R}_+$ & Geodesic distance according to the metric $F$ from points $x$ to $y$. \\
                     $dt$ & $\in \mathbb{R}_+$ & Geodesic diffusion time step. \\
                     $\varepsilon$ & $\in\mathbb{R}_+$ & Small positive number used for numerical stability.\\
                     $\varepsilon_L$ & $\in \mathbb{R}_+$ & Small scalar controlling the maximum scale of the metric. \\
                     $\varepsilon_\omega$ & $\in (0,1]$ & Hyperparameter controlling the maximum tolerated asymmetry, with $1$ being no asymmetry. \\
                     $\eta$ & $\in\mathbb{R}_+$ & Learning rate. \\
                     $f$ & $:\Omega \to \mathbb{R}^c$ & Image with $c$ colour channels. By default in this paper, it has $c=1$ channel. \\
                     $F$ & $: X \times T_xX\to \mathbb{R}_+$ & Finsler metric on the tangent bundle $X \times T_xX$ of manifold $X$. \\
                     $F^*$ & $: X \times T_xX\to \mathbb{R}_+$ & Dual Finsler metric. \\
                     $F_x$ & $:T_xX \to \mathbb{R}_+$ & Finsler metric on the tangent plane $T_xX$ at point $x$ on the manifold $X$. \\
                     $F^\gamma$ & $: X \times T_xX\to \mathbb{R}_+$ & Finsler metric parametrised by $\gamma$. \\
                     $F_x^\gamma$ & $:T_xX \to \mathbb{R}_+$ & Finsler metric, parametrised by $\gamma$, at point $x$. \\
                     $g$ & $:\Omega \to \mathbb{R}^c$ & Convolution kernel. It can be viewed as an image, just like $f$, or we can focus only on its support. Its values $g(y)$ are the weights of the convolution kernel. \\
                     $\gamma$ & $\in\mathbb{R}^p$ & Metric parameters, such as $M$ for Riemannian metrics and $(M,\omega)$ for Randers metrics. \\
                     $h_x$ & $: \Omega \to \mathbb{R}_+$ & Finsler-Gauss kernel. \\
                     $\iota$ & $\in\mathbb{R}_+$ & Average metric scale parameter. \\
                     $k$ & $\in \mathbb{N}$ & Edge size of discrete convolution filters, containing $k\times k$ values for two-dimensional image convolution. \\
                     $L$ & $\in \mathbb{R}^{d\times d}$ & Lower triangular matrix with positive diagonal defining the Cholesky decomposition of $M = LL^\top$. For images, it is a $2\times 2$ matrix with only $3$ non-zero entries. \\
                     $\tilde{L}$ & $\in \mathbb{R}^{d\times d}$ & Slightly modified version of the lower triangular matrix $L$ by $\tilde{L} = L + \varepsilon_L I$ to control the maximum scale of the metric. \\
                     $\Lambda$ & $\in\mathbb{R}^{d\times d}$ & Diagonal matrix. \\
                     $\lambda$ & $\in \mathbb{R}_+$ & Strictly positive scalar. \\
                     $\lambda_i$ & $\in\mathbb{R}$ & Scalar, $i$-th eigenvalue. \\
                     $\tilde{\lambda}_i$ & $\in\mathbb{R}_+$ & Positive and stable modification of $\lambda_i$. \\
                     $\lambda_i'$ & $\in\mathbb{R}_+$ & Unscaled modification of $\lambda_i$. \\
                     $M$ & $:X \to S_d^{++}$ & Riemannian tensor. It fully encodes the Riemannian metric. At point $x$, $M(x)$ is a symmetric positive definite matrix. For standard image manifolds, $M(x)$ is $2\times 2$. For conciseness, $M$ also refers to $M(x)$ without ambiguity. \\
                     $\tilde{M}$ & $\in\mathbb{R}^{d\times d}$ & Perturbed version of the estimated metric $M$ for stability. \\
                     $M^*$ & $:X \to S_d^{++}$ & Riemannian tensor of the dual Randers metric of the Randers metric with parameters $(M,\omega)$. \\
                     $m_x(y)$ & $\in [0,1]$ & Modulation value for the convolution weight $g(y)$ when considering the pixel location $x$. \\
                     $\nabla f(x)$ & $\in \mathbb{R}^d$ & Gradient of $f$ at pixel $x$. \\
                     $\nabla f(x)^\perp$ & $\in \mathbb{R}^d$ & Orthogonal of the gradient $\nabla f(x)$. \\
                     $\lVert \cdot\rVert_A$ & $\in \mathbb{R}_+$ & $L_2$ norm of a vector with symmetric positive definite matrix $A$, meaning $\lVert u\rVert_A = \sqrt{u\top A u}$.\\
                     $\Omega$ & $\subset \mathbb{R}^d$ & Parametrisation domain. For images, it can be seen as the unit square of $\mathbb{R}^2$. \\
                     $\omega$ & $: X\to T_xX$ & $\omega(x)$ is the tangent vector parametrising the linear drift component of a Randers metric at point $x$. \\
                     $\omega^*$ &  $: X\to T_xX$ & Linear drift vector field of the dual Randers metric of the Randers metric with parameters $(M,\omega)$. \\
                     $\tilde{\omega}$ & $: X\to T_xX$ & Modified scaled version of $\omega(x)$. \\
                     $P_x^\gamma$ & $: \mathbb{R}^{d+1} \to \Omega$ & Stencil of points to be geodesically flowed at unit speed. \\
                     $\mathbb{R}_+$ & $\subset \mathbb{R}$ & Set of positive numbers (includes $0$). \\
                     $R$ & $\in \mathbb{R}^{d\times d}$ & Rotation matrix (or a triangular matrix in the QR decomposition). \\
                     $R$ & $: X \times T_xX\to \mathbb{R}_+$ & Riemannian metric on the tangent bundle $X \times T_xX$ of manifold $X$. \\
                     $R_x$ & $:T_xX \to \mathbb{R}_+$ & Riemannian metric on the tangent plane $T_xX$ at point $x$ on the manifold $X$. \\
                     $r$ & $\in\mathbb{R}^d$ & Vector whose normalisation is the first column of a rotation matrix. \\
                     $\tilde{r}$ & $\in\mathbb{R}^d$ & Perturbed version of $r$ for stability. \\
                     $\tilde{r}_\perp$ & $\in\mathbb{R}^d$ & Orthogonal vector to $\tilde{r}$. \\
                     $s$ & $\in \mathbb{R}_+$ & Dilation factor. It is global as it is the same at all pixel locations $x$. It is also used as the radius integration variable when using polar coordinates. \\
                     $\tilde{s}$ & $\in\mathbb{R}_+$ & Estimated eigenvalue scale. \\
                     $s_0$ & $\in \mathbb{R}_+$ & Initial scaling factor for geodesic flow of a stencil of points. \\
                     $s_{\mathrm{max}}$ & $\in\mathbb{R}_+$ & Maximum tolerated eigenvalue scale. \\
                     $s_{\mathrm{min}}$ & $\in\mathbb{R}_+$ & Minimum tolerated eigenvalue scale. \\
                     $s_x$ & $\in \mathbb{R}_+$ & Dilation factor at point $x$. It is local as it can differ between pixel locations $x$. \\
                     $\sigma$ & $: \mathbb{R}\to \mathbb{R}_+$ & Sigmoid function. \\
                     $\tilde{\sigma}$ & $: \mathbb{R}\to \mathbb{R}_+$ & Modified sigmoid function. \\
                     $\sigma_n$ & $\in \mathbb{R}_+$ & Standard deviation of white Gaussian noise. It encodes the noise level. \\
                     $T_{\mathrm{max}}$ & $\in\mathbb{R}_+$ & Temperature hyperparameter of the Adam optimiser. \\ 
                     $T_xX$ & $\equiv \mathbb{R}^d$ & Tangent space at point $x$ of the manifold $X$. For standard image manifolds, it is a two-dimensional plane and can be associated as $\mathbb{R}^2$ for a fixed coordinate system. \\
                     $\theta$ & $\in [0,2\pi]$ & Angle. \\
                     $\theta_x$ & $\in [0,2\pi]$ & Angle of the image gradient. \\
                     $u$ & $\in T_xX$ & Tangent vector. \\
                     $u_\theta$ & $\in \mathbb{R}^2$ & Euclidean unit vector with angle $\theta$: $u_\theta = (\cos \theta, \sin\theta)^\top$. \\
                     $X$ & $\subset \mathbb{R}^d$ & Manifold. For images, it is the same as the parametrisation space $\Omega$, and thus can be viewed as the unit square of $\mathbb{R}^2$. This perspective is different from a common convention where images are viewed as embedded manifolds, for instance colour images would be curved two-dimensional surfaces embedded in $\mathbb{R}^3$. \\
                     $x$ & $\in X$ & Point of the manifold. For image manifolds, it is by extension the pixel coordinate position in $\Omega$. \\
                     $y$ & $\in \Omega$ & Pixel location of the convolution kernel. \\
                     $\Delta_x$ & $\subset \Omega$ & Support of the convolution kernel when considering the pixel location $x$. \\
                     $y_x(\theta,\gamma)$ & $\in B_1^t(x)$ & Point on the unit tangent circle of $x$ with angle $\theta$ for the metric $F^\gamma$. \\
                     $\mathcal{Z}(x)$ & $\in\mathbb{R}_+$ & Normalisation factor in the Finsler-Gauss kernel. \\
                     $z_{k, c_{\textit{in}}}$ & $\in\mathbb{R}_+$ & Uniform initialisation value of kernel weights. \\
                     \midrule
                     CNN & Acronym & Convolutional neural networks. \\
                     FKW & Acronym & Fixed kernel weights. It is an experimental setting where convolution kernels have fixed uniform weights. \\
                     LKW & Acronym & Learnable kernel weights. It is an experimental setting where convolution weights are learned. \\
                     MSE & Acronym & Mean squared error. \\
                     SC & Acronym & Scratch, for training convolution kernels from scratch. \\
                     TL & Acronym & Transfer learning, for training convolution kernels using transfer learning. \\
                     UGB & Acronym & Unit geodesic ball: set of points on the manifold within unit geodesic length according to the metric. \\
                     UTB & Acronym & Unit tangent ball: convex set of tangent vectors of unit length according to the metric. \\
                     \bottomrule
                \end{tabular}
        }
    \captionof{table}{
        Table of notations. Most notations are used only in the appendix.
    }%
    \label{tab: notations}
\end{strip}

\clearpage

\section{Finsler and Randers Metrics: Further Details}

We refer the interested reader for more information on Finsler and Randers metrics to the specialised Finsler literature, such as \cite{ohta2009heat,mirebeau2014efficient,bonnans2022linear}. The details mentioned in this section are well-known in the community, but we put them here so that our paper is self-contained.

\subsection{{Finsler Metric Axioms}}
\label{sec: metric axioms}

We provide in \cref{fig: metric axioms} a simple visualisation of the Finsler metric axioms. The triangular inequality is equivalent to the convexity of the unit tangent balls. Positive homogeneity implies that the metric scales with the same ratio as tangent vectors when considering only positive scaling (no flips). This differs from regular homogeneity, as in Riemannian metrics, where the metric scales with the absolute value of the ratio of tangent vectors. In a homogeneous metric, if $\lambda>0$ and $u\in T_xX$, then $\lambda u$ and $-\lambda u$ both have their metric scaled by the same number $|\lambda| = \lambda$, making the metric symmetric. In contrast, in a positively homogenous metric, the metric at vectors oppositely scaled $\lambda u$ and $-\lambda u$ is different, as $F_x(\lambda u) = \lambda F_x(u)$ and $F_x(-\lambda u) = \lambda F_x(-u)$, where $F_x(u)\neq F_x(-u)$ in general.
Note that a homogeneous metric is also positively homogeneous. For this reason, Riemannian metrics, which are root-quadratic homogeneous metrics, are special cases of Finsler metrics (see \cref{fig: venn diagram finsler metrics}). Randers metrics are special cases of Finsler metrics containing Riemannian metrics, but  when $\omega\neq 0$ they only satisfy the positive homogeneity.

\subsection{Randers Metric Positivity}

The positivity of the Randers metric $F$ is ensured by $\lVert \omega\rVert_{M^{-1}} < 1$. In fact, we can generalise the statement in the following proposition that links the Randers metric with the regular $L_2$ norm.

\begin{proposition}
    \label{prop: randers metric positivity}
    Let $0<\varepsilon<1$. If for any point $x$ on the manifold we have $\lVert \omega(x)\rVert_{M^{-1}(x)} \le 1-\varepsilon$, then the metric satisfies $F_x(u) \ge \varepsilon \lVert u\rVert_2$ for any $u\in T_xX$. In particular, if $\lVert \omega\rVert_{M^{-1}(x)} < 1$, then $F_x(u)>0$ for any $u\neq 0$.
\end{proposition}

\begin{proof}
    All tangent vectors of $T_xX$ can be rewritten as $M(x)^{-1}u$. We then have
\begin{align}
    &F_x(M(x)^{-1}u) = \sqrt{u^\top M(x)^{-1} u} + \omega(x)^\top  M(x)^{-1}u\\
    &\quad= \lVert M(x)^{-\tfrac{1}{2}}u\rVert_2 + \langle M(x)^{-\tfrac{1}{2}}\omega(x),  M(x)^{-\tfrac{1}{2}}u\rangle.
\end{align}

The Cauchy-Schwartz inequality provides that
\begin{equation}
    |\omega(x)^\top  M(x)^{-1}u | \le \lVert M(x)^{-\tfrac{1}{2}}\omega(x) \rVert_2 \lVert M(x)^{-\tfrac{1}{2}} u\rVert_2.
\end{equation}
The assumption on $\omega$ can be rewritten as 
\begin{equation}
    \lVert M(x)^{-\tfrac{1}{2}}\omega(x)\rVert \le 1-\varepsilon.    
\end{equation}
Thus, for any tangent vector $u$, we get $F_x(M(x)^{-1}u) \ge \varepsilon \lVert M(x)^{-1}u\rVert_2$, and so for any such $u$, we obtain the desired result $F_x(u) \ge \varepsilon\lVert u\rVert_2$.
\end{proof}

\begin{figure}[t]
    \centering
    \includegraphics[width=\columnwidth]{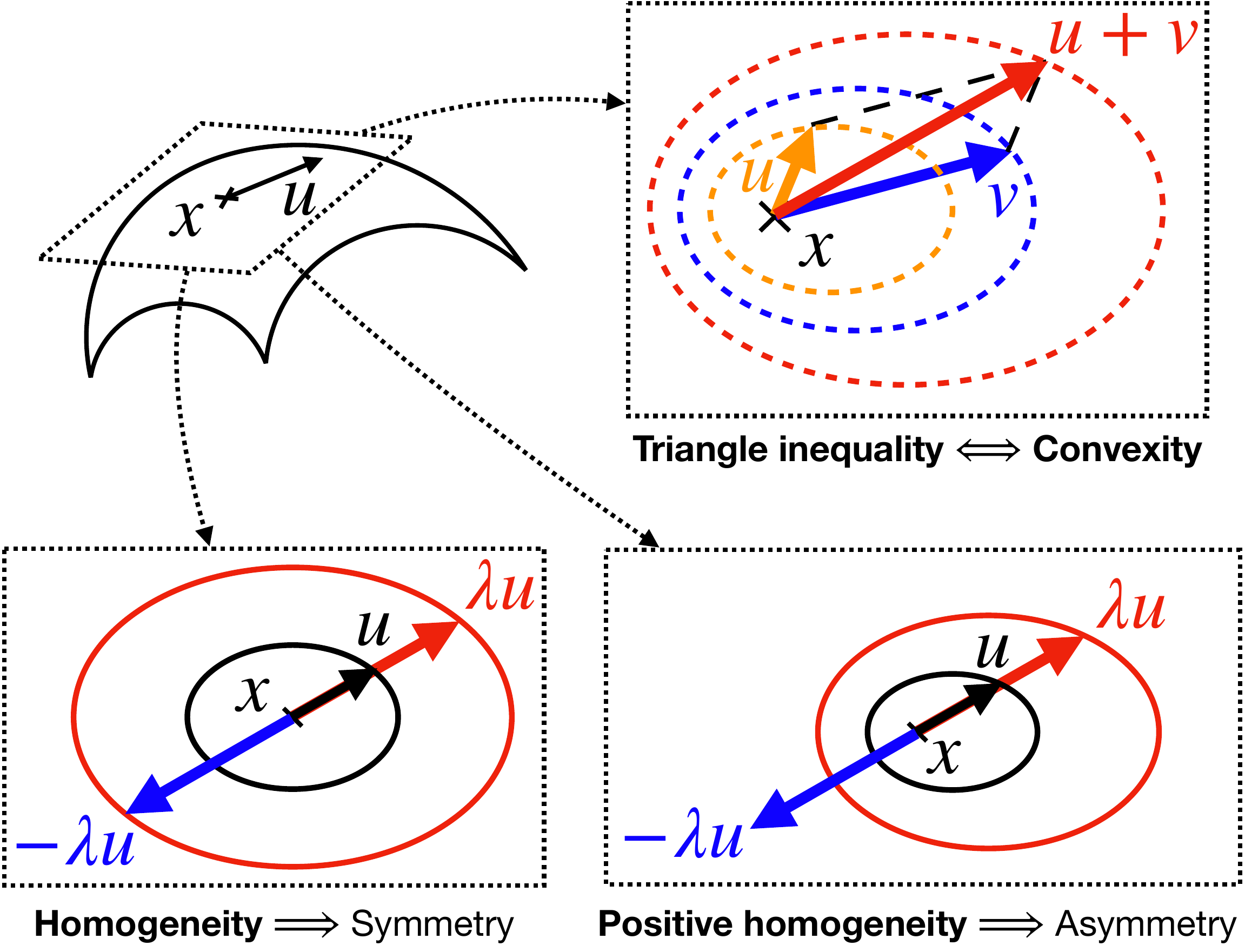}
    \caption{Illustration of metric axioms. The triangular inequality is equivalent to the convexity of unit tangent balls. Finsler metrics are positively homogeneous, which is less restrictive than regular homogeneity and allows asymmetric distances. In contrast, Riemannian metrics are homogeneous and thus always symmetric.}
    \label{fig: metric axioms}
\end{figure}

\begin{figure}[t]
    \centering
    \includegraphics[width=0.6\columnwidth]{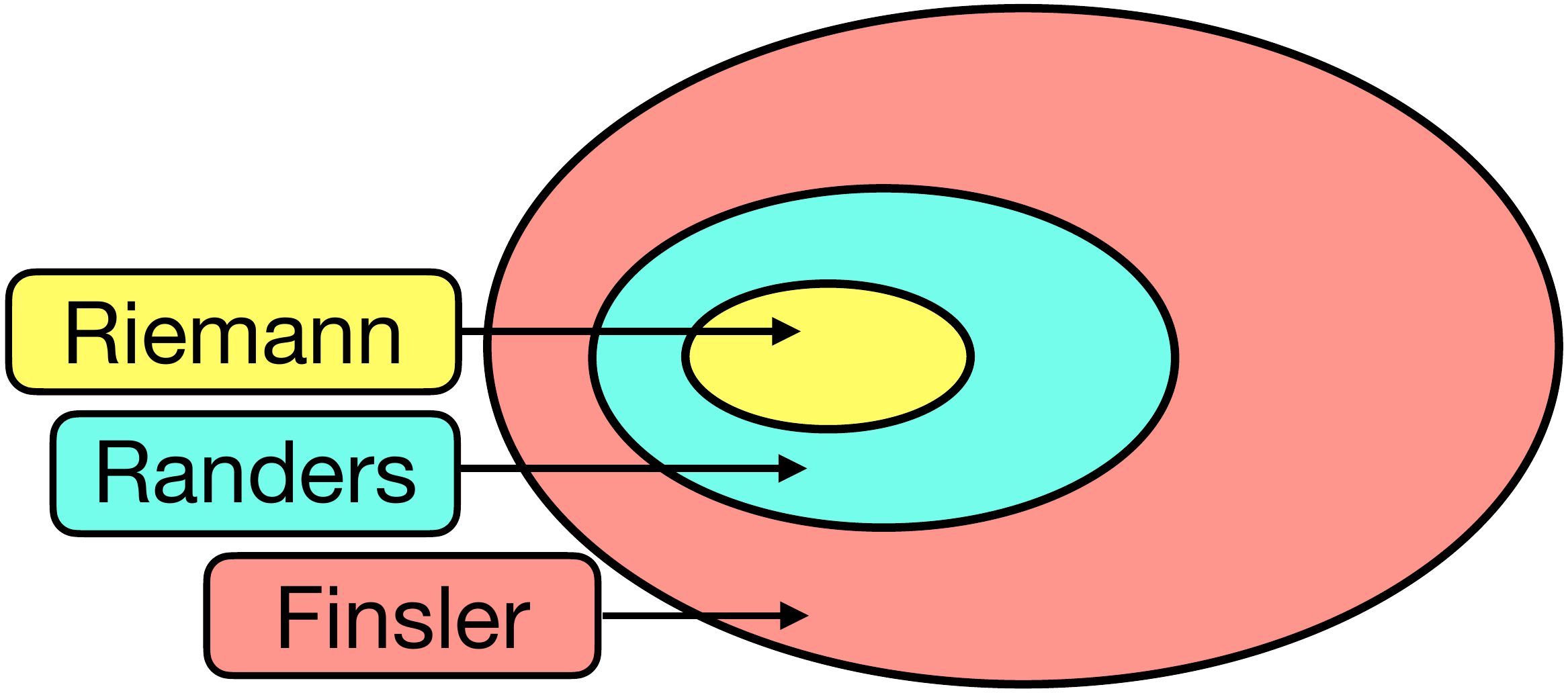}
    \caption{Venn diagram of Finsler metrics. Riemannian metrics are a special case of symmetric Finsler metrics. Randers metrics contain Riemannian ones, but are in general asymmetric.}
    \label{fig: venn diagram finsler metrics}
\end{figure}

\subsection{Finsler Geodesic Distances}
\label{sec: finsler geodesic distances}

Given a Finsler metric $F$, the geodesic distance $\mathrm{dist}_F(x,y)$ between points $x$ and $y$ on the manifold is given by the minimum length of a smooth curve $c(t)$ from $x$ to $y$
\begin{equation}
    \mathrm{dist}_F(x,y) = \min\limits_{\substack{c(t)\\c(0) = x \text{ and } c(1) = y}} \int_{[0,1]}F_{c (t)}\big(\tfrac{\partial c}{\partial t}(t)\big)dt.
\end{equation}
In particular, the orientation $x$ to $y$ is important for non Riemannian asymmetric metrics, as then $F_{c (t)}\big(\tfrac{\partial c}{\partial t}(t)\big)$ may differ from $F_{c (t)}\big(-\tfrac{\partial c}{\partial t}(t)\big)$.

\subsection{Dual Randers Metric}
\label{sec: dual randers metric}

The dual metric of a Finsler metric plays a key role in differential geometry on manifolds as it systematically appears in major differential equations, such as the Eikonal equation or the heat diffusion equation. Formally the dual metric of $F$ is the metric $F^*$ such that
\begin{equation}
    \label{eq: dual Finsler metric}
    F_x^{*}(u) = \max\{u^\top v ;\; v \in T_xX, F_x(v) \leq 1\}.    
\end{equation} 
One can easily verify that it satisfies the Finsler metric axioms. If $F$ is a Randers metric parameterised by $(M,\omega)$, then the dual metric $F^*$ is also a Randers metric. It is parameterised by $(M^*,\omega^*)$, which are explicitly given by $(M,\omega)$.

\begin{proposition}
    \label{prop: dual randers formula}
    The dual of a Randers metric $F$ parameterised by $(M,\omega)$ is a Randers metric $F^*$. If we denote $\alpha = 1 - \lVert \omega\rVert_{M^{-1}}^2 >0$, then the parameters $(M^*,\omega^*)$ of $F^*$ are given by
    \begin{equation*}
        \begin{cases}
            M^{*} = \frac{1}{\alpha^2}\Big(\alpha M^{-1} + (M^{-1}\omega)(M^{-1}\omega)^{\top}\Big), \\
            \omega^{*} = -\frac{1}{\alpha}M^{-1}\omega.            
        \end{cases}
    \end{equation*}
\end{proposition}

\begin{proof}
    To ease notations, we drop the explicit dependence on $x$.
    Although the definition of the dual Randers metric is given in  \cref{eq: dual Finsler metric}, since $F$ is positive homogeneous it is also given by
    \begin{equation}
        F^*(u) = \max \left\{\frac{u^\top v}{F(v)} ;\; v\neq 0 \right\}.
    \end{equation}
    Therefore, the inverse solves the minimisation problem $\tfrac{1}{F^*(u)} = \min \left\{\tfrac{F(v)}{u^\top v} ;\; v\neq 0 \right\}$. Likewise, by positive homogeneity, we have that the inverse of the dual metric satisfies the constrained minimisation problem
    \begin{equation}
        \label{eq: 1 / F* = min F}
        \frac{1}{F^*(u)} = \min \{F(v) ;\; u^\top v = 1 \}.
    \end{equation}
    We can also solve this constrained optimisation problem with Lagrangian optimisation. The Lagrangian is given by $L(v, \lambda) =  F(v) + \lambda u^\top v$. To satisfy the KKT conditions, we differentiate $L$ with respect to $v$ and set the gradient to $0$. The optimal $v^*$ thus satisfies
    \begin{equation}
        \label{eq: randers dL/db = 0}
        M\frac{v^*}{\lVert v^* \rVert_M} + \omega + \lambda u = 0.
    \end{equation}
    By computing the scalar product of this equation with $v^*$, and recalling that the constraint guarantees $u^\top v^* = 1$, we get, since $F(v^*) = \sqrt{{v^*}^\top M v^*} + {v^*}^\top \omega$, that $\lambda = -F(v^*)$. Recall that $v^*$ solves the minimisation problem \cref{eq: 1 / F* = min F}. Therefore,

    \begin{equation}
        \label{eq: randers lambda = - 1/F*}
        \lambda = -\frac{1}{F^*(v)}.
    \end{equation}
    Returning to \cref{eq: randers dL/db = 0}, we can compute $\lVert \omega + \lambda u \rVert_{M^{-1}}$ as
    \begin{equation}
        \lVert \omega + \lambda u \rVert_{M^{-1}} = \frac{\lVert Mv^* \rVert_{M^{-1}}}{\lVert v^* \rVert_M} = 1.
    \end{equation}
    By squaring this equation, we obtain a polynomial of degree two for which $\lambda$ is a root
    \begin{equation}
        \lambda^2 \lVert u \rVert_{M^{-1}}^2  + 2\lambda \langle \omega, u\rangle_{M^{-1}} + \lVert \omega\rVert_{M^{-1}}^2 - 1 = 0.
    \end{equation}
    Let $\alpha = 1 - \lVert \omega \rVert_{M^{-1}}^2 >0$. The roots are then given by
    \begin{equation}
        \lambda_\pm = \frac{-\langle \omega, u\rangle_{M^{-1}} \pm \sqrt{
    \langle \omega, u\rangle_{M^{-1}}^2 + \lVert u\rVert_{M^{-1}}^2 \alpha}}{\lVert u \rVert_{M^{-1}}^2}.
    \end{equation}
    Clearly, we have $\lambda_- < 0 < \lambda_+$ for $u\neq 0$. However, $\lambda < 0$ as $\lambda = -F(v^*)$ and the metric is always positive. As such, $\lambda = \lambda_-$. Inverting \cref{eq: randers lambda = - 1/F*}, we get
    \begin{align}
        F^*(u) &= \frac{\lVert u\rVert_{M^{-1}}^2}{\langle \omega, u\rangle_{M^{-1}} + \sqrt{\langle \omega, u\rangle_{M^{-1}}^2 + \lVert u\rVert_{M^{-1}}^2 \alpha}} \\
        &= \sqrt{u^\top \frac{1}{\alpha^2}\left( \alpha M^{-1} + M^{-1}\omega \omega^\top M^{-1}\right) u} \nonumber\\
        &\quad\quad-\frac{1}{\alpha} \langle \omega, u\rangle_{M^{-1}},
    \end{align}
    where the classical trick $\tfrac{1}{x+\sqrt{y}} = \tfrac{x-\sqrt{y}}{x^2-y}$ was used to remove the square root from the denominator. We now identify the dual metric $F^*$ as a Randers metric associated to $(M^*,\omega^*)$ as initially claimed 
    \begin{equation}
        \begin{cases}
            M^* = \frac{1}{\alpha^2}(\alpha M^{-1} + (M^{-1}\omega) (M^{-1}\omega)^\top ), \\
            \omega^* = -\frac{1}{\alpha}M^{-1}\omega.
        \end{cases}
    \end{equation}

\end{proof}

\paragraph{Going further.}
When studying the propagation of a wave front \cite{mirebeau2014efficient}, the dual metric naturally appears yielding the Finsler Eikonal equation
\begin{equation}
    F_x^*(-\nabla f) = 1.
\end{equation}
Likewise, by observing that the heat equation in the Riemannian case is given by the gradient flow of the Dirichlet energy, we can descend on the dual energy $\tfrac{1}{2}F_x^*(u)^2$ to define the Finsler heat equation \cite{ohta2009heat,bonnans2022linear}
\begin{equation}
    \label{eq: finsler heat equation}
    \frac{\partial f}{\partial t} = \mathrm{div}(F_x^*(\nabla f)\nabla F_x^*(\nabla f)).
\end{equation}
From the heat equation we can then compute various fundamental operators, mainly the Laplace-Finsler operator, a generalisation of the Laplace-Beltrami operator, that describes the shape \cite{bonnans2022linear}. These interesting constructions are beyond the scope of this paper. For the interested reader, we point out the nice presentation and exploration of Finsler and Randers metrics and heat equation, leading to the Finsler-based Laplace-Beltrami operator \cite{weber2024finsler}, which was used instead of the traditional Laplacian to solve the shape matching problem \cite{bracha2023partial,bracha2024unsupervised,bracha2024wormhole}.
.
Additionally, we refer to \cite{dages2025finsler} for a recent example of a beautiful application of elementary Finsler geometry for manifold learning and dimensionality reduction of asymmetric data.

\subsection{Finsler and Hyperbolic geometry}

Finsler geometry is a generalisation of Riemannian geometry where distances can depend on both position and direction, leading to more flexible and varied shapes. Within Finsler geometry, some spaces show properties similar to hyperbolic geometry \cite{beardon2012geometry,reynolds1993hyperbolic,weber2024flattening,lin2021hyperbolic,lin2023hyperbolic,lin2025treewasserstein}, a special type of Riemannian geometry having constant negative curvature. Extending further, Finsler-Lorentz geometry \cite{caponio2011interplay,alvarez2015introduction,pendas2024application} allows the metric to have a Lorentzian signature, meaning it can describe not just spatial distances, but also time and causality as in relativity, all while keeping the direction-dependent qualities of Finsler spaces. This helps model more general and realistic behaviours in physics and geometry.

\section{From Single-Channel Images to Manfifolds with Any Intrinsic or Feature Dimensions}

\subsection{Generalising our Unifying Metric Theory}
\label{sec: generalising unifying metric theory to arbitrary manifolds}

\paragraph{Preliminaries.}
An image is a special type of two-dimensional manifold $X$, where for each point $x\in X=\Omega=[0,1]^2$ on the manifold we associate a feature vector $f(x)\in \mathbb{R}^c$. An other equivalent popular perspective for image manifolds is to see them as surfaces $X\subset \Omega\times \mathbb{R}^c \subset \mathbb{R}^{d + c}$ of the form $(x, f(x))$, but we do not adopt this perspective in this paper. The image feature dimensionality $c$ is also called the number of channels of the image.
Image manifolds can be generalised beyond two-dimensions. A $d$-dimensional hyperimage is a manifold $X = \Omega = [0,1]^d$ with feature function $f:X\to \mathbb{R}^c$. The manifold dimensionality of a hyperimage is $d$, regardless of the number of channels of the hyperimage: when embedding the hyperimage in $\mathbb{R}^{d+c}$, the manifold $(x, f(x))$ for $x\in \Omega$ is locally $d$-dimensional, meaning its tangent space $T_xX$ can be associated with $\mathbb{R}^d$. Points $x\in \Omega$ parametrising the hyperimage are called hypervoxels, generalising the concept of pixels when $d=2$ and voxels when $d=3$.

The particularity of image and hyperimage manifolds is that they possess a global $uv$ parametrisation via $x\in\Omega = [0,1]^d$. For general manifolds, finding such a parametrisation is only possible locally\footnote{For instance there is no global smooth $uv$ parametrisation of the Klein bottle, even though it is a smooth two-dimensional manifold.}. Depending on the topology of the manifold, it may be possible to find such a global $uv$ parametrisation, as for instance in genus $0$ manifolds (topological planes or spheres). For general manifolds $X$ satisfying this global $uv$ parametrisation, e.g.\ a closed surface in $\mathbb{R}^3$, a manifold hyperimage can be viewed as a manifold enriched with a texture feature function $f:X\to\mathbb{R}^c$. Thanks to the global parametrisation, this formulation is mathematically equivalent to regular hyperimages where $X = \Omega = [0,1]^d$, e.g.\ $[0,1]^2$ for a closed surface in $\mathbb{R}^3$.

We are now ready to generalise our unifying theory on image convolutions to hyperimage convolutions.

\paragraph{Theoretical generalisation.}
Given a hyperimage of dimensionality $d$ having $c_{\mathit{in}}$ channels associated to the function $f:\Omega=[0,1]^d\to \mathbb{R}^{c_{\mathit{in}}}$, and a kernel hyperimage of dimensionality $d$ with $c_{\mathit{out}}\times c_{\mathit{in}}$ channels associated to the function $g:\Omega\to \mathbb{R}^{c_{\mathit{out}}\times c_{\mathit{in}}}$. A hyperimage convolution produces a $d$-dimensional hyperimage with $c_{\mathit{out}}$ channels, and is defined per output channel as
\begin{equation}
    \label{eq: convolution high dims per channel}
    (f*g)_j (x) = \sum\limits_{i=1}^{c_{\mathit{in}}}\int_\Omega g_{j,i}(y) f_i(x+y)dy,
\end{equation}
for $j\in\{1,\cdots, c_{\mathit{out}}\}$. This definition can be summarised in matrix-vector form as
\begin{equation}
    (f*g) (x) = \int_\Omega g(y) f(x+y)dy,
\end{equation}
which uses the same formalism as \cref{eq: conv}.

Our entire theoretical discussion in \cref{sec: unifying metric theory to convs} then immediately generalises to hyperimages of any dimensionality and feature channels. In particular, \cref{th: reformulation convs} would still apply. For instance, after discretising $\Omega=[0,1]^d$ into hypervoxels, the reference kernel support $\Delta$ can be given by a $k\times \cdots \times k$ grid with $k^d$ hypervoxels, e.g.\ with hypervoxel indices $\{(i_1, \cdots, i_d)\in \{-1,1,0\}^d\}$ for $k=3$. Dilation would scale this reference support by a factor $s$ in all $d$ dimensions. Shifted convolutions would shift the support by $\delta_x\in \mathbb{R}^d$, and deformable convolution would associate a shift vector $\delta_x^y\in\mathbb{R}^d$ for each of the $k^d$ cells.

Note that the hypervoxel grid discretisation is a standard procedure for hyperimages as they share a global $uv$ parametrisation via $x\in\Omega=[0,1]^d$. This procedure though is less frequent for manifolds equivalent to hyperimages, such as genus $0$ manifolds like closed two-dimensional surfaces. Nevertheless, our formalism is not reliant on a specific type of sampling, it only requires the ability to approximately query feature values at non sampled locations of the manifold. This is usually possible for most parametrisations via some form of interpolation of feature values.

Finally, consider general manifolds that are not equivalent to hyperimages, where a global $uv$ parametrisation does not exist. Our formulation still generalises to them as long as the local kernel support $\Delta_x$ is small enough within the local neighbourhood where we can use a local $uv$ parametrisation to approximate the manifold around point $x$.

\subsection{Generalising Metric Convolutions}
\label{sec: generalising metric convolutions to arbitrary manifolds}

Metric convolutions easily generalise beyond single-channel image convolutions. 

\paragraph{Feature dimensionality.}
As defined in \cref{eq: convolution high dims per channel}, when images are not single-channel, the $j$-th output feature of the convolution is given by the aggregation of $c_{\mathit{in}}$ convolutions between the single-channel images $f_i$ and $g_{j,i}$. As such, to generalise metric convolutions to multi-channel data, it suffices to define $c_{\mathit{in}}$ single-channel metric convolutions for each output dimension. This procedure would require $c_{\mathit{out}}\times c_{\mathit{in}}$ single-channel metric convolutions to compute the convolutions from $c_{\mathit{in}}$ to $c_{\mathit{out}}$ channel images. This generalisation procedure is the same as for generalising standard convolutions. As the required number of kernel samples, $k\times k \times c_{\mathit{in}} \times c_{\mathit{out}}$ for $k\times k$ convolutions, linearly increases with the number of input channels this can make high channel convolution operation particularly costly. To overcome the issue, the kernel $g$ can be made sparse using separability, such as $g_{j,i} = 0$ for $j\neq i$ in depthwise convolutions \cite{sifre2014rigid,sifre2013rotation,chollet2017xception,howard2017mobilenets} or $g_{j,i} = 0$ for $j\notin C_i$ in group convolutions \cite{krizhevsky2012alexnet}, where $C_i$ is a group of input feature dimensions\footnote{Standard convolution has a single group, i.e.\ $C_i = \{1,\cdots, c_{\mathit{in}}\}$, whereas depthwise convolutions has as many groups as input (and output) features, i.e.\ $C_i = \{i\}$.} for dimension $j$. The same can be done in metric convolutions. By using (block) separable filters, the size of the convolutions can drastically decrease, as depthwise convolutions use $k\times k \times c_\mathit{in}$ kernel locations and group convolutions use $\tfrac{k\times k \times c_\mathit{in} \times c_\mathit{out}}{G}$ for $G$ groups.

\paragraph{Manifold dimensionality.}
For a $d$-dimensional manifold $X$, its tangent spaces $T_xX$ are $d$-dimensional spaces and can be associated with $\mathbb{R}^d$. Given a fixed Cartesian parametrisation of $\mathbb{R}^d$, the Riemannian metric of such a manifold is defined as in two-dimensions: $R_x(u) = \sqrt{u^\top M(x) u}$. However, now $u\in T_xX=\mathbb{R}^d$ is a $d$-dimensional vector and the metric tensor $M(x)\in\mathbb{R}^{d\times d}$ is a $d\times d $ symmetric positive definite matrix. Likewise, Finsler metrics generalise in a similar fashion. In particular, the Randers metric of a $d$-dimensional manifold is then given by $F_x(u) = \sqrt{u^\top M(x) u} + \omega(x)^\top u$, where now $\omega(x)\in T_xX=\mathbb{R}^d$ is a $d$-dimensional vector. With these generalised definitions, we can then perform minor modifications to generalise our metric convolutions to $d$-dimensional manifolds. Note that we focus on hyperimage manifolds, which provide a shared universal parametrisation of the manifold and of the tangent spaces, given by the canonical Cartesian coordinates of $\mathbb{R}^d$. This avoids the need to compute changes of local coordinate systems, simplifies derivations, and allows simple interpolation of features beyond sampled values.

For illustration, we here present a generalisation of UTB metric convolution. Given the metric parameters $M(x)\in\mathbb{R}^d$ and $\omega(x)\in\mathbb{R}^d$, we need to compute explicitly the unit tangent ball, which is now a convex hypersurface of dimension $d-1$ in the tangent plane rather than a curve in the two-dimensional tangent plane. As it is convex and containing $0$, it can be parametrised via $d-1$ angular spherical coordinates $\theta\in\mathbb{R}^{d-1}$. For Randers mertrics, the UTB is a hyperellipsoid given by the equation $u^\top M(x)  u= (1 - \omega(x)^\top u)^2$ in $u$. Using this spherical reparametrisation, \cref{eq: finsler unit circle} still holds for points on the unit ball. Thus, the metric UTB convolution formula \cref{eq: metric utb f*g} still holds, with now $\theta$ being $d-1$-dimensional angular spherical coordinates and $s\in\mathbb{R}$ is the scalar radius coordinate.

We can generalise our discretisations of metric convolutions. Focusing once again for simplicity on UTB metric convolutions, the Cholesky decomposition approach would require encoding the symmetric positive definite matrix $M$ as $M = LL^\top$ where $L$ is a lower triangular matrix, thus encoded by $\tfrac{d(d-1)}{2}$ values. For the spectral approaches, we would need to compute $d$ values for the eigenvalues of $M$. To compute the orthogonal eigenbasis $U$, like in the $2$-dimensional case, several approaches are possible. We could use the exponential map $U = e^S$ for some skew-symmetric matrix $S = -S^\top$, parametrised by $\tfrac{d(d-1)}{2}$ unconstrained entries. Another option is to encode $U$ via Givens rotations $U = R_1\cdots R_{{d(d-1)}/{2}}$ where each Givens rotation is a planar rotation in the plane of two of the coordinate axes, requiring more parameters for encoding as we need the angles (or their cosine and sine) and the chosen axes. Another option is to take the $QR$ decomposition, since any matrix $A$ can be decomposed as $A = QR$ where $Q$ is orthogonal and $R$ is upper triangular. We could thus encode $U = AR^{-1}$ requiring $d^2 + \tfrac{d(d-1)}{2} = \tfrac{d(3d-1)}{2}$ parameters, and making sure that the diagonal of $R$ is non-zero\footnote{For instance by feeding the diagonal to any positive non-linearity followed by a positive perturbation of $\varepsilon$.}. Regarding $\omega$, we can encode it in the same way as in the two-dimensional case, requiring now $d$ numbers instead of $2$. These decompositions can be tweaked manually according to heuristics, or can once again be learned. Note that the numbers required to compute $M$ can either be learned directly or can be learned as output of an intermediate standard $d$-dimensional convolutions, which would guarantee signal-adaptive shift-equivariant metric computation and thus shift-equivariant metric convolution, as in the two-dimensional case.

Note that an important aspect of our convolutions is knowing how to query features at non-sampled locations. When using hyperimages, the canonical axes of the domain $\Omega = \mathbb{R}^d$ are universal and shared by all hyperimages and are the same in each tangent plane of each pixel position. In particular, non sampled locations are guaranteed to fall within a hypercube of uniformly sampled hypervoxels grid points. This implies that feature interpolation requires only a small number of support samples with analytically known locations and analytically derivable interpolation formula. This makes the interpolation cheap and fast to implement, and would be more challenging if we were working on non-hyperimage data with irregular samples locations that differ between data instances (e.g.\ pointclouds).


\section{Proofs of Our Unifying Metric Theory}

\subsection{Proof of \cref{th: reformulation convs}}
\label{sec: proof reformulated convs}

Unlike most of the community, we are rephrasing the pre-existing convolutions in the continuum. For now, we put aside modulation, which breaks the weight sharing assumption of convolution. Dilated convolutions scale by a factor $s$ the reference support $\Delta^{\textit{ref}}$, usually uniformly in the image, to provide a dilated support $\Delta^{\textit{dil}}$:  $\Delta^{\textit{dil}} = s\Delta^{\textit{ref}}$. Dilated convolution is thus given by 
\begin{equation}
    (f*g)(x) = \int_{\Delta^{\textit{ref}}} f(x+s y) g(s y) dy.
\end{equation}

In the non-standard case of using a different scale $s_x$ for each pixel $x$, the support changes per pixel $\Delta_x^{\textit{dil}} = s_x\Delta^{\textit{ref}}$, and then dilated convolution is defined as
\begin{equation}
    (f*g)(x) = \int_{\Delta^{\textit{ref}}} f(x+s_x y) g(s_x y) dy.
\end{equation}

In all cases, dilated convolutions can be rewritten as
\begin{equation}
    (f*g)(x) = \int_{\Delta_x^{\textit{dil}}} f(x+y) g(y) dy.
\end{equation}

Shifted convolutions, also called entire deformable convolutions, simply shift the reference support $\Delta^{\textit{ref}}$ by an offset $\delta_x$ shared by all cells
\begin{equation}
    (f*g)(x) = \int_{\Delta^{\textit{ref}}} f(x + y + \delta_x) g(y + \delta_x).
\end{equation}

As such, if we denote $\Delta_x^{\textit{ent}} = \delta_x + \Delta^{\textit{ref}}$, we have for entire deformable convolutions
\begin{equation}
    (f*g)(x) = \int_{\Delta_x^{\textit{ent}}} f(x+y) g(y) dy.
\end{equation}

On the other hand, deformable convolution introduces different offset vectors $\delta_x^y$ for each entry in the reference kernel support $\Delta^{\textit{ref}}$. Therefore, it is given by
\begin{equation}
    (f*g)(x) = \int_{\Delta^{\textit{ref}}} f(x + y + \delta_x^y) g(y + \delta_x^y) dy.
\end{equation}

If we now write the deformed support $\Delta_x^{\textit{def}} = \{y+\delta_x^y; \forall y\in\Delta^{\textit{ref}}\}$, then deformed convolution can be rewritten as
\begin{equation}
    (f*g)(x) = \int_{\Delta_x^{\textit{def}}} f(x+y) g(y) dy.
\end{equation}

We have thus managed to express these various convolutions with the same formulation
\begin{equation}
    (f*g)(x) = \int_{\Delta_x} f(x+y) g(y) dy,
\end{equation}
where $\Delta_x$ is either $\Delta^{\textit{ref}}$ in the standard case, $\Delta_x^{\textit{dil}}$ in the dilated case, $\Delta_x^{\textit{ent}}$ in the entire deformable case, and $\Delta_x^{\textit{def}}$ in the deformable case.  If we now break the weight sharing assumption of any of these convolutions by introducing modulation, we have mask numbers $m_x(y)\in[0,1]$ that multiply the kernel values $g(y)$. As such, the convolutions become
\begin{equation}
    (f*g)(x) = \int_{\Delta_x} f(x+y) g(y) m_x(y) dy.
\end{equation}
We can then define the distribution $dm_x(y) = m_x(y) dy$ to prove the theorem and show that all these convolutions perform weighted filtering on some neighbourhood $\Delta_x$ sampled with distribution $dm_x$.
\qed

\subsection{On Modulation and Weight Sharing}
\label{sec: modulation and weight sharing assumption}

We here further present and discuss modulation and how it breaks the weight sharing assumption.

\paragraph{Weight Sharing.}
In order to provide deterministic sampling locations, the modulation $m_x(y)$ is not usually understood in the literature as a probabilistic density from which to sample kernel locations. Instead, it is rather viewed as modulation density, or mask, depending on the location $x$, that is applied to the $k\times k$ convolution weights at position $x$. 
This means that instead of computing at location $x$ the convolution between the two functions $f$ and $g$ with varying sampling distribution $dm_x(y)$, the convolution uses a uniform sampling distribution $dy$ but convolves between functions $f$ and each $g_x: y\mapsto g(y)m_x(y)$ at point $x$, explaining why the weights are not shared in modulation.

Additionally, implementation strategies of modulation can further deviate from the weight sharing assumption, removing ever more the essence of what makes a convolution.
There are essentially two main ways modulation is implemented in existing works. 

\paragraph{Modulating deformations.} 
The first approach, as in \cite{zhu2019moredeformable,yu2022entire}, is to encode the modulation as an array with as many entries as kernel samples, i.e.\ $k\times k$ values, that is to be pointwise multiplied with the kernel weights. This strategy does not aim to construct the modulation at all possible continuous locations in the continuous kernel, but only focuses on the sampled locations, and by doing so it is not constrained to any geometric prior. This implementation is particularly useful for sparse sampling locations, when $k$ is small and the reference grid $\Delta^{\textit{ref}}$ is heavily deformed. The mask values are then generally constrained to $[0,1]$ and optimised in the training process. While this approach slightly boosts performance, it is theoretically unnecessary, as zeroing out sample locations is less optimal that simply moving the filtered-out samples to better locations. In practice, the resulting modulation tends to focus on sample locations close to the original pixel, and filters out distant ones.

\paragraph{Modulating without deformation.}
Given the prior that mask weights should decay with distance, the second way to encode modulation is to define a parametric mask function satisfying the prior. This function defines the mask weights at all possible continuous locations of the continuous non-deformed kernel, and then needs to be queried at the sampled locations of the discrete kernel. A famous example is the FlexConv method \cite{romero2021flexconv} and its Gaussian modulation. Unlike the previous implementation, this approach is tailored for large kernel sizes $k\times k$, e.g.\ $k=25$ instead of $k=3$, making it expensive, as it relies on a continuous perspective of kernels\footnote{In fact, in FlexConv the kernel is encoded continuously with an implicit neural representation via a Multi-Layer Perceptron.}.  Importantly, in this strategy, the kernels are not deformed in any way $\Delta = \Delta^{\textit{ref}}$. By adopting large kernel size and not deforming the kernel, the discrete kernel densely approximates the continuous kernel.
Instead of explicit kernel deformation, Gaussian modulation, with learnable mean and covariance, is then pointwise applied to the kernel at each location $x$. This allows the operation to focus more on a part of the kernel map at each pixel location. 

From our universal metric perspective, convolution operations using modulation without deformations as in FlexConv can be understood from two perspectives. The first is that unit balls are always the same and large, equal to the reference kernel $\Delta^{\textit{ref}}$, but the sampling density is not uniform. In particular, the implicit metric is the uniformly scaled standard Riemannian metric. A more interesting perspective is to understand the unit balls of the implicit metrics to fit the elliptical effective support of the Gaussian kernel, with scale given by the covariance and asymmetrically offset by the mean. In this second perspective, we must relax the unit ball sampling assumption with a sampling that can extend beyond unit balls but with a smooth rather than binary decay. This new perspective might mislead us into thinking that such constructions are equivalent to our metric convolutions, both offsetting and deforming a circle into an ellipse, with FlexConv requiring to encode a larger global weight map to be cropped. However, this is incorrect and a major difference exists between both approaches.

We illustrate this discussion in \cref{fig: modulations}.
For explanation clarity, binarise the Gaussian mask: only an ellipse of weights are then considered. In metric convolutions or other approaches deforming a reference kernel, these methods deform a reference kernel shape, e.g.\ a disk, and its texture to provide a new shape, with identically deformed texture, eventually modulated. On the other hand, in adaptive methods relying on modulation like FlexConv, the transformation is fundamentally different. In the second understanding of this modulation, the reference kernel shape and its texture are not deformed identically. Instead, a crop of a larger texture map is performed. This implies that there is no consistency in the convolution weights at different pixel locations, even in the simplified case where we binarise the Gaussian weight mask. FlexConv-like modulation thus fundamentally breaks the weight sharing assumption even more so than other modulation techniques, putting into question the ``convolution'' nomenclature of such methods.

\begin{figure}[t]
    \centering
    \includegraphics[width=\columnwidth]{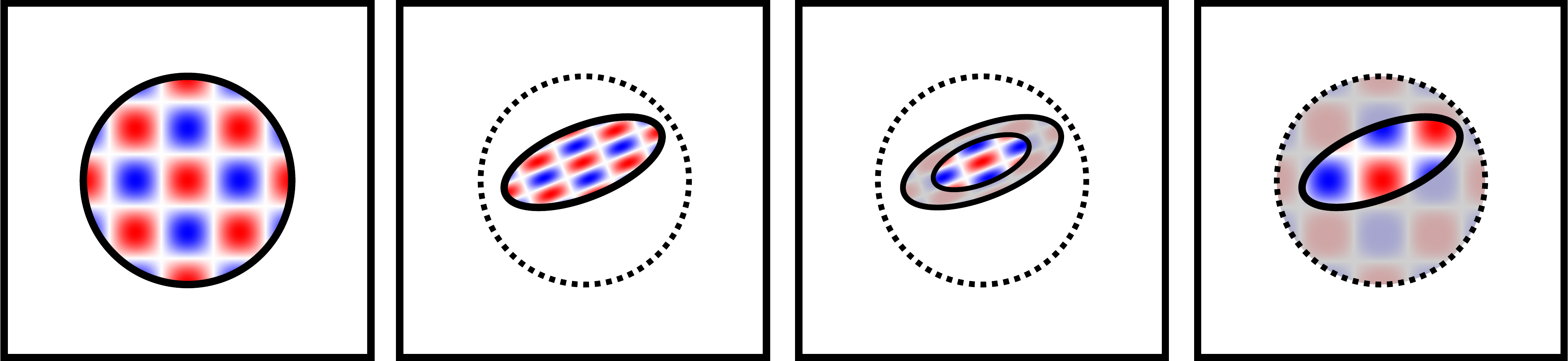}
    \caption{Illustration of different modulation strategies. Each image represents a kernel with weights given by the plotted texture. Left: reference kernel shape for standard convolution $\Delta^{\textit{ref}}$. Its texture $g$ is the reference weight map shared by all pixels. Centre left: the reference kernel shape and its associated texture, i.e.\ the weights, are scaled, deformed, and shifted, resulting in a different convolution support $\Delta\neq\Delta^{\textit{ref}}$ but same convolution weights. This case covers the approaches mentioned in the main manuscript, namely scaled, shifted, deformed, and our metric convolutions. Centre right: Modulation is added to the deformed convolution kernels, following the implementation strategy of \cite{zhu2019moredeformable,yu2022entire}. The weights of the kernel are preserved but some are attenuated by some factor. Right: Modulation is added by cropping the non-deformed reference kernel, e.g.\ via Gaussian masking as in FlexConv \cite{romero2021flexconv}. This requires an original kernel that is large enough, unlike the other methods, which is expensive. Additionally, weights of the kernel at the effectively sampled regions are no longer shared between points, even with binary modulation, unlike in the other methods. This second implementation implies that the statistics of the effective kernel weights are  fundamentally different from those using the first strategy.
    }
    \label{fig: modulations}
\end{figure}

\subsection{Proof of \cref{th: metric defined unit balls}}
\label{sef: proof metric unique tangent balls}

This result is well-known. It is a direct consequence of the positive homogeneity of the metric (see \cref{fig: finsler metric reconstruction from utb}).  Assume that the UTB $B_1^t(x)$ is given at any point $x$. Let $u\in T_xX$ be a non zero tangent vector. Then there exists a unique $v\in B_1^t(x)$ that is positively aligned with $u$, i.e.\ there exists $\lambda>0$ for which $v = \lambda u$, that has unit metric $F_x(v) = 1$. Note that $v$ is the intersection of the ray with direction $u$ (with origin $x$) with the boundary of $B_1^t(x)$. We can then define $F_x(u) = \lambda$. We also define $F_x(0) = 0$. We can then easily check that $F_x$ satisfies the positive homogeneity. Given any $\lambda'>0$, we have $\lambda' u = \lambda'\lambda v$. By construction of $F_x$, we thus have $F_x(\lambda' u) = \lambda'\lambda = \lambda' F_x(u)$. As easily, we can check that $F_x$ satisfies the triangular inequality and separability. As such, the provided UTBs implicitly defines a metric $F_x$.
\qed

\begin{figure}[t]
    \centering
    \includegraphics[width=\columnwidth]{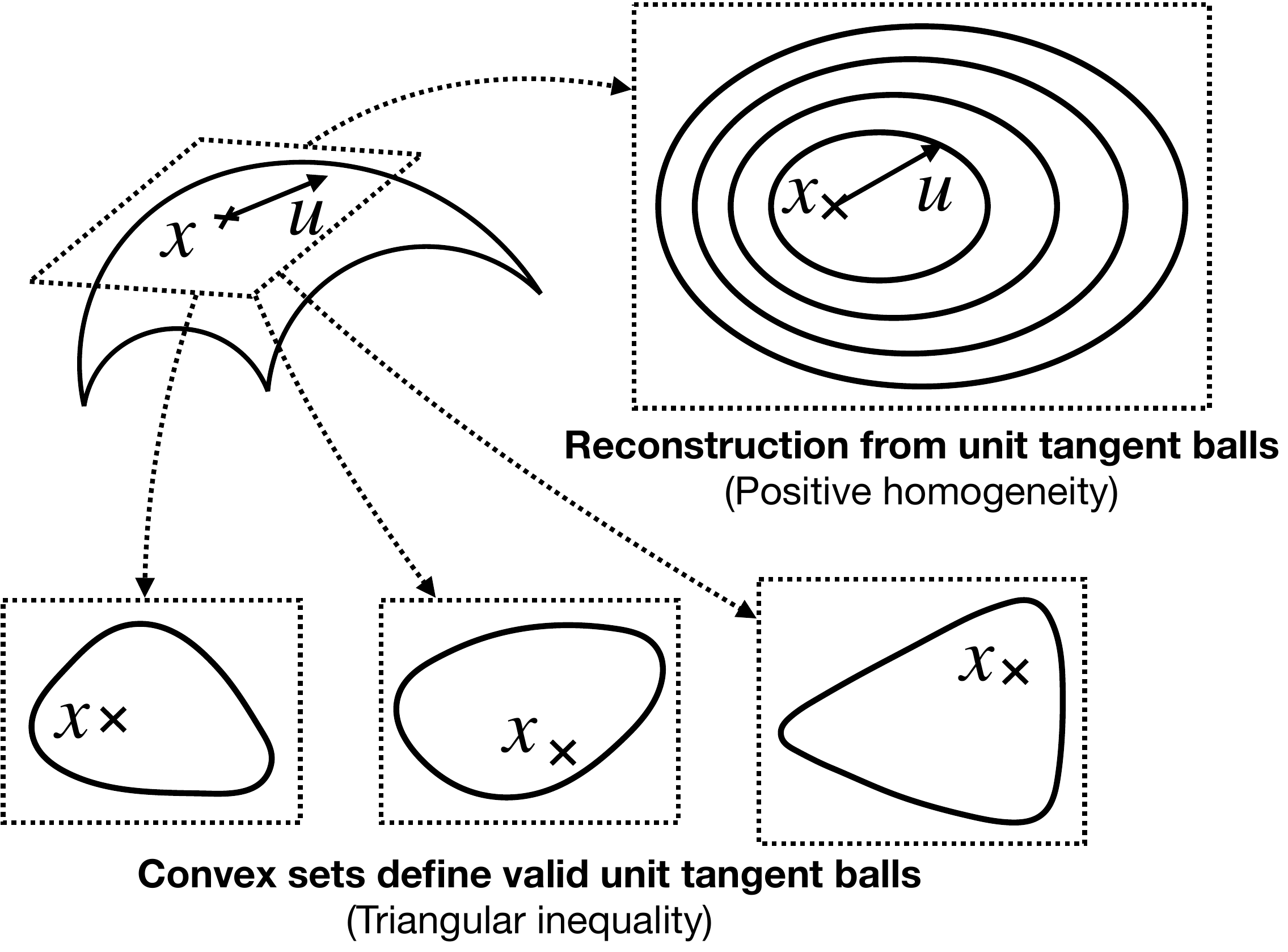}
    \caption{Thanks to their positive homogeneity, Finsler metrics are fully described by their unit tangent balls, which can be any convex set due to the triangular inequality.}
    \label{fig: finsler metric reconstruction from utb}
\end{figure}

\subsection{Reconstructing a Metric from Unit Geodesic Balls}
\label{sec: reconstructing metric from UGB}

Reconstructing the metric, or an approximation, is possible if further assumptions are introduced. For instance, if we are provided with the knowledge of distances within the unit ball, i.e.\ level sets within $B_1^g(x)$, or if the unit ball is sufficiently localised, i.e.\ $B_1^g(x)$ is sufficiently small (in Euclidean distance) to be approximated by its projection onto the tangent plane $T_xX$, then we can (approximately) reconstruct the unit tangent ball at point $x$ and from there use \cref{th: metric defined unit balls} to recover the entire metric.

The issue for reconstructing metrics from UGBs is that geodesic distances consist in an integration of the metric along the tangents of the geodesics. By performing this summation, we can lose local information on the original metric. As an extreme counter-example, consider a small sphere, with radius smaller than $\tfrac{1}{2\pi r}$, then the unit ball at any point for the isotropic Riemannian metrics $M\equiv sI$ with $s\le 1$ will cover the entire sphere, and likewise, other more complex metrics will provide the same unit ball. In this simplistic example, recovering the underlying metric is impossible without other prior knowledge.

\subsection{Example of non-unique Metrics Explaining Discretised Sample Locations}
\label{sec: example non unique metric for discrete samples}

Given a discretised sampling of a unit ball, the underlying continuous unit ball is ambiguous. As such, several metrics may provide continuous unit balls for which the given samples provide a good covering of its unit ball. We here provide two examples.

First, assume that we have finite samples and that an oracle provides us the information that these samples all lie on the unit tangent circle of some metric, i.e.\ $F_x(u) = 1$ for each of these samples $u$. Then any convex closed simple curve interpolating the provided samples yields a plausible Finsler metric $F_x$.

In general however, we are not aware of the distance of the sampled points within the unit ball. Sampled points do not necessarily lie at the same distance from $x$ and no oracle provides their distance. Consider the following example. Assume we are provided with the reference template $\Delta^{\textit{ref}}$ with $3\times 3$ samples. We want to interpret it as a natural sampling of the unit tangent ball of some metric. A first natural possibility is to invoke the isotropic Riemannian Euclidean distance, for which the unit ball is the round disk\footnote{It would be scaled so that the radius of the unit circle is in $[\sqrt{2},2)$ in Euclidean measurements}. Another possibility is to consider the traditional non-Riemannian $L^\infty$ metric yielding unit balls in the shape of squares with straight edges parallel to those of the domain axes. Unlike the isotropic Euclidean suggestion, in the $L^\infty$ one the samples on the border of the convex hull of the samples all lie on the unit circle of the metric.

\subsection{Implicit Metrics of Existing Convolutions}
\label{sec: implicit metrics of existing convolutions}

We provide in \cref{fig: adaptive neighbourhoods sampled} an intuitive visualisation of the metrics used for existing convolution methods. All of them can be well-approximated by adopting a tangent perspective to unit ball sampling of implicit Riemannian (for standard and dilated convolutions) or Finsler (for shifted and deformable convolutions) metrics. Our metric convolutions explicitly compute metrics and their unit balls, which then provide analytical sampling locations.

\begin{figure*}
    \centering
    \includegraphics[width=0.8\textwidth]{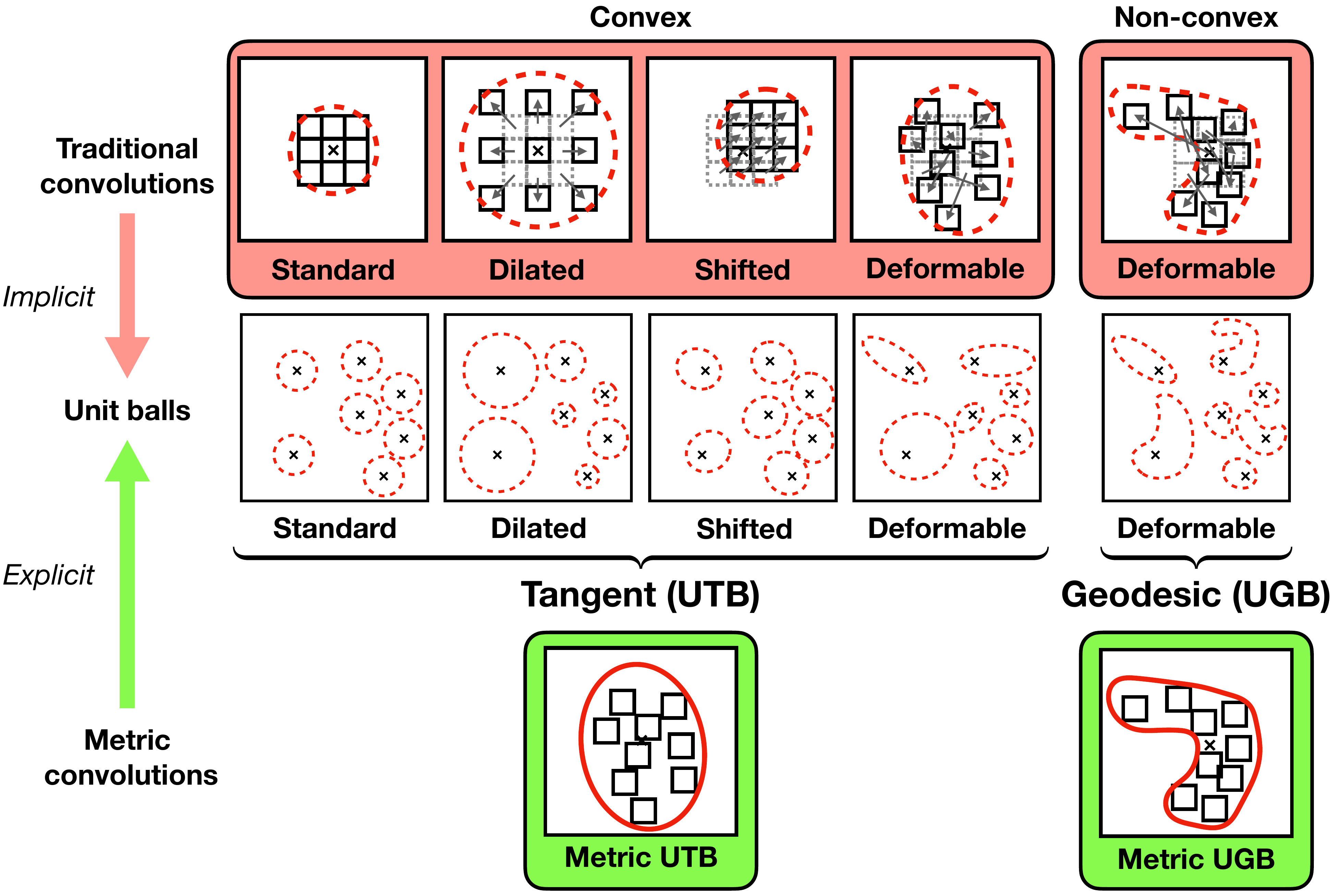}
    \caption{Traditional convolutions sample pixel neighbourhoods, which can be understood as implicitly sampling unit balls of implicit metrics. Metric convolutions explicitly compute the metric and its unit balls, which are then analytically sampled. By using Randers metrics, unit tangent ball (UTB) metric convolutions provide sampling neighbourhoods similar to those of traditional methods. Extreme non-convex sampling neighbourhoods can be provided by unit geodesic ball (UGB) metric convolutions.
    }
    \label{fig: adaptive neighbourhoods sampled}
\end{figure*}


\section{Further Discussions on the Theory of Metric Convolutions}

Metric convolutions are versatile in their implementation, via heuristic or learnable metric designs. When implementing metric convolutions with learnable adaptive metric (see \cref{fig: summary adaptive conv algo details}), we apply an intermediate standard convolution to learn a fixed number ($5$ to $7$)  of metric (hyper)parameters $\gamma$, from which we then extract the Randers metric parameters $(M,\omega)$ at each pixel using simple $2\times 2$ matrix operations. From these parameters, we can then analytically compute kernel sampling locations of the unit tangent balls in a simple closed form that will be used for the final averaging of the metric convolution. The choice of standard convolution as intermediate hyperparameter extractor echoes the design of deformable \cite{dai2017deformable,zhu2019moredeformable} and shifted \cite{yu2022entire} convolutions, providing a shift-equivariant metric convolution. A major difference though is that in these competing methods, the intermediate convolution computes directly the offset sampling locations, without any additional metric prior. Additionally, as the offsets are independent for each kernel cell in deformable convolution, the intermediate standard convolution requires $2k^2$ output channels, making this operation particularly costly as the kernel size grows. This contrasts with our metric convolution, which requires a fixed number of output channels for the intermediate convolution regardless of the number of kernel samples.

\begin{figure*}
    \centering
    \includegraphics[width=\textwidth]{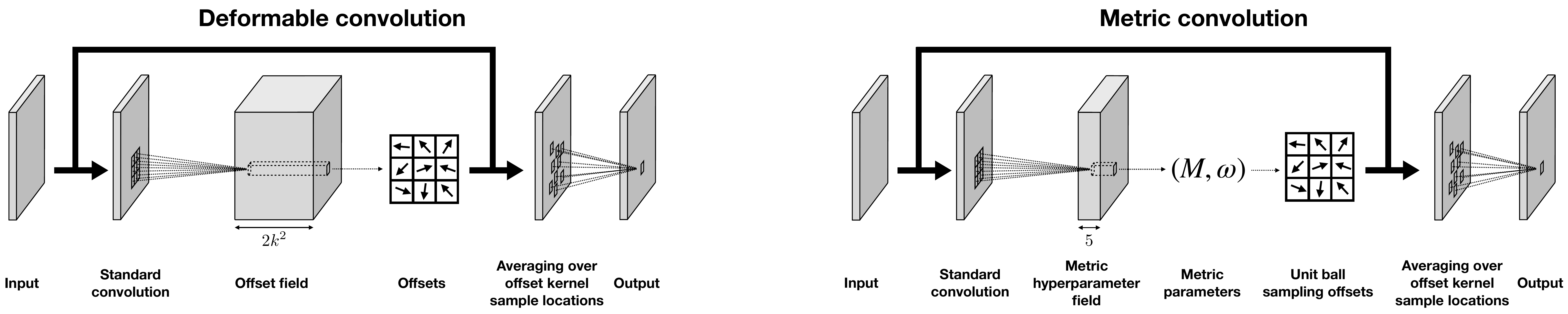}
    \caption{Metric convolutions can be implemented for learnable adaptive metrics with a similar design to deformable convolutions relying on an intermediate standard convolution. Instead of directly computing deformation offsets of the deformable convolution, which is costly as it requires $2k^2$ output channels for the intermediate convolution for a $k\times k$ convolution, we compute metric hyperparameters, which are a fixed number independent of the number of samples. From those, we can compute the metric and analytically sample unit tangent balls with simple and cheap operations involving solely $2\times 2$ matrix multiplications. The use of standard convolution as metric parameter extractors, along with polar sampling schemes, makes metric convolutions a shift-equivariant operation by design.}
    \label{fig: summary adaptive conv algo details}
\end{figure*}

\subsection{Computing the Metric from $5$ Learnt Numbers: Cholesky Approach}
\label{sec: computing metric from 5 numbers}

The most general metrics we consider in this work are Randers metrics $F$, which are parameterised by $(M,\omega)$, where $M(x)\in \mathbb{R}^{2\times2}$ is a symmetric definite positive matrix and $\omega(x)\in\mathbb{R}^{2}$ must satisfy $\lVert \omega(x) \rVert_{M^{-1}(x)}<1$. Here, we will consider a unique location $x$, thus for conciseness, we drop the explicit dependence on it.

This discussion explains how to compute the metric parameters $\gamma$ from $5$ numbers using a Choleksy-based approach. Recall the fundamental linear algebra result that symmetric definite positive matrices possess a Cholesky decomposition and vice versa.

\begin{proposition}[Cholesky decomposition]
    If $M\in\mathbb{R}^{d\times d}$ is a symmetric definite positive matrix, then there exists $L\in\mathbb{R}^{d\times d}$ that is lower triangular and with only positive diagonal entries such that $M = LL^\top$.
\end{proposition}

In our case, we can reparameterise the Riemannian metric matrix $M$ with its Cholesky decomposition matrix $L$ requiring only three numbers instead of $4$. On the other hand, $\omega$ needs two. As such, our metric convolutions require an intermediate operation to compute $5$ numbers per pixel location to fully describe the metric. By analogy with deformable convolution, we chose to use a standard convolution with only $5$ output channels. Nevertheless, several issues remain. First, we need to make sure that $M$ does not become singular. This can happen through uncontrolled optimisation if for instance the lower right entry $L_{2,2}$ becomes close to $0$. Secondly, we need to enforce the norm constraint on $\omega$ for the metric to remain positive. To overcome these challenges, we used the following strategy.

To avoid the non singularity of $M$, we construct $\tilde{L} = L + \varepsilon_L I$ , where $I$ is the identity matrix and $\varepsilon_L>0$ is a small number controlling the maximum scale of the metric\footnote{The smaller $\varepsilon_L$ the larger the maximum scale, quadratically.}. The Riemannian component is then given by $M = \tilde{L}\tilde{L}^\top$. In our experiments, we chose $\varepsilon_L = 0.01$. 

To enforce the positivity of the metric, we introduce another hyperparameter $\varepsilon_\omega \in (0,1]$ and devise a strategy to enforce $\lVert \omega\rVert_{M^{-1}} \le 1-\varepsilon_\omega$. This choice would guarantee that the metric does not accumulate around $0$ as then $F_x(u)\ge \varepsilon_\omega \lVert u \rVert_2$ for any tangent vector $u$ following \cref{prop: randers metric positivity}. Taking $\varepsilon_\omega\to1^{-}$ forces metric symmetry, whereas $\varepsilon_\omega\to0^{+}$ allows the strongest asymmetry. Note that $\varepsilon_\omega = 1$ is equivalent to taking $\omega \equiv 0$. The strategy is the following. Compute $\lVert \omega \rVert_{M^{-1}}\in[0,\infty)$, and feed it to a modified sigmoid function to get a new number in $[0,1-\varepsilon_L)$. Recall that the sigmoid is defined as $\sigma(x) = \tfrac{1}{1+e^{-x}}$. Our modified sigmoid is thus $\tilde{\sigma}(x) = 2(1-\varepsilon_\omega)(\sigma(x) - \tfrac{1}{2})$. The computed number $\tilde{\sigma}(\lVert \omega\rVert_{M^{-1}})\in [0,1-\varepsilon_\omega)$ represents the desired $M^{-1}$-norm of the Randers drift component. A such, we could take $\tilde{\omega} = \tfrac{\tilde{\sigma}(\lVert \omega\rVert_{M^{-1}})}{\lVert \omega\rVert_{M^{-1}}}\omega$. An issue arises though if we are learning the metric from data and initialise with $\omega \equiv 0$, as the square root has divergent gradient at the origin. To avoid this issue, we replace the computation of $\lVert \omega\rVert_{M^{-1}}$ by the more stable $\sqrt{\omega^\top M^{-1} \omega + \varepsilon}$, where $\varepsilon>0$ is a small number, typically $\varepsilon = 10^{-6}$, that we systematically use to avoid instabilities, e.g.\ divisions by 0. As such, the modified drift component forming the metric is $\tilde{\omega} = \frac{2(1-\varepsilon_\omega)(\sigma({\sqrt{\omega^\top M^{-1}\omega + \varepsilon}})-\frac{1}{2})}{\sqrt{\omega^\top M^{-1}\omega + \varepsilon}}\omega$.

We summarise the approach in \cref{alg: computing the metric from 5 numbers}. Thus, when learning the metric, we use the metric defined by $(\tilde{M}, \tilde{\omega})$, and performed gradient descent on the $5$ parameters of $L$ and $\omega$ (per pixel), which is possible since $\tilde{M}$ and $\tilde{\omega}$ are given by differentiable operations from $L$ and $\omega$.

\begin{algorithm}[tb]
   \caption{Metric computation from $5$ numbers}
   \label{alg: computing the metric from 5 numbers}
    \begin{algorithmic}
       \STATE {\bfseries Input:} Five numbers $L_{1,1}$, $L_{1,2}$, $L_{2,2}$, $\omega_1$, $\omega_2$
       \STATE {\bfseries Hyperparameters:} $\varepsilon_L>0$, $\varepsilon_\omega\in(0,1]$, $\varepsilon>0$
       \STATE Construct $L = \begin{pmatrix} L_{1,1} & 0 \\ L_{2,1} & L_{2,2} \end{pmatrix}$ and $\omega = (\omega_1,\omega_2)^\top$
       \STATE Compute $\tilde{L} = L + \varepsilon_L I$\\[0.5em]
       \STATE Compute $\tilde{M} = \tilde{L}\tilde{L}^\top$
       \STATE Compute $\tilde{\omega} = \frac{2(1-\varepsilon_\omega)\left(\sigma({\sqrt{\omega^\top M^{-1}\omega + \varepsilon}})-\frac{1}{2}\right)}{\sqrt{\omega^\top M^{-1}\omega + \varepsilon}}\omega$
       \STATE {\bfseries Return} $(\tilde{M},\tilde{\omega})$
    \end{algorithmic}
\end{algorithm}

\subsection{Computing the Metric from $6$ or $7$ Learnt Numbers: Spectral Approach}
\label{sec: computing metric from 6 7 numbers}

The previous Cholesky-based implementation (\cref{sec: computing metric from 5 numbers}) sometimes suffers from instabilities during training when combined with neural networks. This issue persisted when using an LDLT approach, where $LDL^\top = M$ and $L$ is lower triangular with unit diagonal entries and $D$ is positive diagonal matrix. We here present a more stable implementation to compute $\gamma$ at the cost of one or two extra numbers to encode the metric.

As in \cref{sec: computing metric from 5 numbers}, we work with Randers metrics parametrised by $(M,\omega)$ and focus on a unique location $x$, allowing us to drop its explicit dependence on it for conciseness. Recall the fundamental linear algebra result that symmetric matrices can be diagonalised in an orthogonal basis.

\begin{proposition}[Spectral theorem]
    If $M\in\mathbb{R}^{d\times d}$ is a symmetric matrix, then $M$ can be diagonalised in an orthogonal basis. This means that there exists an orthogonal matrix $R\in\mathbb{R}^{d\times d}$ and a diagonal matrix $\Lambda$ such that $M = R\Lambda R^\top$.
\end{proposition}

As the dimensionality of the image surface manifold is $d=2$, we could encode the rotation matrix $R$ by an angle $\theta$ as $R = \begin{psmallmatrix}
    \cos\theta & -\sin\theta\\
    \sin\theta & \cos\theta
\end{psmallmatrix}$. Thus, only $3$ numbers $\theta$, $\lambda_1$, and $\lambda_2$ could suffice to encode $M$, as in the Cholesky approach. However, regressing raw angle values is well-known to be significantly harder than estimating their cosine and sine values. This is due to periodic nature of angles: raw angle values $\varepsilon$ and $2\pi-\varepsilon$ for small $\varepsilon>0$ have a large difference but they correspond to almost identical angles. Instead, given two unconstrained numbers $r\in\mathbb{R}^2$, we construct the rotation matrix $R$ using $R = \tfrac{1}{\lVert \tilde{r}\rVert_2 + \varepsilon}(\tilde{r} \;|\; \tilde{r}_\perp)$, where $\tilde{r} = r + \varepsilon$ to avoid singular $R$ on rare instances\footnote{We add $\varepsilon$ here as if rigorously $r=0$, which happened in practice on clean noiseless data like MNIST, then the vector $\tfrac{r}{\lVert \tilde r\rVert_2 + \varepsilon}$ would still be $0$ leading to a singular matrix $R$.} and $\tilde{r}_\perp = (-\tilde{b},\tilde{a})^\top$ if $\tilde{r} = (\tilde{a},\tilde{b})^\top$.

Since $M$ is also positive definite, then its eigenvalues $\lambda_1,\hdots,\lambda_d$ forming the diagonal of $\Lambda$, are strictly positive. Given two unconstrained numbers $\lambda_1$ and $\lambda_2$, we could construct the eigenvalue matrix 
$\Lambda = \begin{psmallmatrix}
    \tilde{\lambda}_1 & 0 \\
    0 & \tilde{\lambda}_2
\end{psmallmatrix}$, 
where $\tilde{\lambda_i} = |\lambda_i| + \varepsilon_L$ for $i\in\{1,2\}$. The Riemannian component would then be given by $\tilde{M} = R\Lambda R^\top$. This strategy requires $6$ numbers to compute the metric parameters $\gamma$.

However, we obtained marginally better results when separating the scale of the eigenvalues, as it introduces regularisation on them. Let $s$ be an additional unconstrained number used to compute the scale of the eigenvalues. Denoting $\lambda_i' = 1 + 2(\sigma(\lambda_i)-\tfrac{1}{2}) = 2\sigma(\lambda_i)\in[0,2]$ the ``unscaled'' eigenvalues centred around 1, we compute their scale as $\tilde{s} = \tfrac{s_\text{min} + s_\text{max}}{2} + 2\left(\sigma(s) - \tfrac{1}{2}\right)(s_\text{max} - s_\text{min})\in[s_\text{min},s_\text{max}]$, where $s_\text{min}$ and $s_\text{max}$ are user-defined minimum and maximum eigenvalue scales\footnote{Recall that due to the inverse in \cref{eq: finsler unit circle}, a smaller eigenvalue scale of $M$ leads to longer unit balls along that direction. For instance, a scale of 0.1 corresponds to stretching the ball to 10 pixels.}. In our experiments, we took $s_\text{min} = 0.1$ and $s_\text{max} = 1.5$. Finally we use the eigenvalues
$\tilde{\lambda}_i = \lambda_i'\tilde{s}_i$ to build the matrix $\Lambda = \begin{psmallmatrix}
    \tilde{\lambda}_1 & 0 \\
    0 & \tilde{\lambda}_2
\end{psmallmatrix}$. The Riemannian component is then provided by $\tilde{M} = R \Lambda R^\top$. This strategy, requiring $7$ numbers to compute the metric, is our preferred strategy and we only report results for this implementation when referring to the spectral implementation. 

For both strategies, we compute the linear drift component from two unconstrained numbers $\omega$ as in \cref{sec: computing metric from 5 numbers}. This means that we use $\tilde{\omega} = \tfrac{\tilde{\sigma}(\lVert \omega\rVert_{M^{-1}})}{\lVert \omega\rVert_{M^{-1}}}\omega$ satisfying the norm constraints for the positivity of the metric. As we use the spectral approach for training stability, we found that we can also improve stability and marginally results by avoiding propagating the gradients through the invert operation $M^{-1}$. To this end, we detach the gradient of the factor $\tfrac{\tilde{\sigma}(\lVert \omega\rVert_{M^{-1}})}{\lVert \omega\rVert_{M^{-1}}}$ in the calculation of $\tilde{\omega}$. We used this strategy for all results reported in this work referring to the spectral implementation.

We summarise the spectral approaches using either $6$ or $7$ numbers in \cref{alg: computing the metric from 6 numbers,alg: computing the metric from 7 numbers}. Our preferred version uses $7$ numbers as it strongly encourages stability and leads to comparable or marginally better results. All results provided in this work using the spectral approach use $7$ numbers. Thus, when learning the metric, we use the metric defined by $(\tilde{M}, \tilde{\omega})$, and performed gradient descent on the $7$ parameters encoding the metric parameters $\gamma$ (per pixel), which is possible since $\tilde{M}$ and $\tilde{\omega}$ are given by differentiable operations from these $7$ numbers.

In our experiments using a simplistic architecture -- a single convolution layer for denoising (\cref{sec: denoising whole images with a single local convolution}), we did not encounter stability issues and provide results using the Cholesky implementation for computing metric parameters $\gamma$ from $5$ numbers. In our experiments using complex architectures -- well-established CNNs for classification (\cref{sec: from single to stacked convolutions: cnn classification}), we strongly benefited from improved stability and provide results using only the spectral implementation for computing the metric parameters $\gamma$ from $7$ numbers.

\begin{algorithm}[htb]
   \caption{Metric computation from $6$ numbers}
   \label{alg: computing the metric from 6 numbers}
    \begin{algorithmic}
       \STATE {\bfseries Input:} Six numbers $r_1$, $r_2$, $\lambda_1$, $\lambda_2$, $\omega_1$, $\omega_2$
       \STATE {\bfseries Hyperparameters:} $\varepsilon_L>0$, $\varepsilon_\omega\in(0,1]$, $\varepsilon>0$
       \STATE Let $r = (r_1,r_2)^\top$
       \STATE Compute $\tilde{r} = r+\varepsilon$ and $\tilde{r}_\perp = (-\tilde{r}_2, \tilde{r}_1)^\top$
       \STATE Construct $R = \tfrac{1}{\lVert \tilde{r}\rVert_2 + \varepsilon}
       \begin{pmatrix}
           \tilde{r}_1 & \tilde{r}_{\perp,1}\\
           \tilde{r}_2 & \tilde{r}_{\perp,2}\\
       \end{pmatrix}$
       \STATE Construct $\Lambda = \begin{pmatrix}
            |\lambda_1| & 0\\
            0 & |\lambda_2|
       \end{pmatrix} + \varepsilon_L I$
       \STATE Compute $\tilde{M} = R \Lambda R^\top$
       \STATE Compute $\tilde{\omega} = \frac{2(1-\varepsilon_\omega)\left(\sigma({\sqrt{\omega^\top M^{-1}\omega + \varepsilon}})-\frac{1}{2}\right)}{\sqrt{\omega^\top M^{-1}\omega + \varepsilon}}\omega$
       \STATE {\bfseries Return} $(\tilde{M},\tilde{\omega})$
    \end{algorithmic}
\end{algorithm}

\begin{algorithm}[htb]
   \caption{Metric computation from $7$ numbers}
   \label{alg: computing the metric from 7 numbers}
    \begin{algorithmic}
       \STATE {\bfseries Input:} Seven numbers $r_1$, $r_2$, $\lambda_1$, $\lambda_2$, $s$, $\omega_1$, $\omega_2$
       \STATE {\bfseries Hyperparameters:} $0<s_\text{min}\le s_\text{max}$, $\varepsilon_\omega\in(0,1]$, $\varepsilon>0$
       \STATE Let $r = (r_1,r_2)^\top$
       \STATE Compute $\tilde{r} = r+\varepsilon$ and $\tilde{r}_\perp = (-\tilde{r}_2, \tilde{r}_1)^\top$
       \STATE Construct $R = \tfrac{1}{\lVert \tilde{r}\rVert_2 + \varepsilon}
       \begin{pmatrix}
           \tilde{r}_1 & \tilde{r}_{\perp,1}\\
           \tilde{r}_2 & \tilde{r}_{\perp,2}\\
       \end{pmatrix}$
       \STATE Compute $\lambda_1' = 2\sigma(\lambda_1)$ and $\lambda_2' = 2\sigma(\lambda_2)$\\[0.5em]
       \STATE Compute  $\tilde{s} = \tfrac{s_\text{min} + s_\text{max}}{2} + 2\left(\sigma(s) - \tfrac{1}{2}\right)(s_\text{max} - s_\text{min})$\\[0.5em]
       \STATE Compute $\tilde{\lambda}_1 = \lambda_1' \tilde{s}$ and $\tilde{\lambda}_2 = \lambda_2' \tilde{s}$
       \STATE Construct $\Lambda = \Lambda = \begin{pmatrix}
            \tilde{\lambda}_1 & 0\\
            0 & \tilde{\lambda}_2
       \end{pmatrix}$
       \STATE Compute $\tilde{M} = R \Lambda R^\top$
       \STATE Compute $\tilde{\omega} = \frac{2(1-\varepsilon_\omega)\left(\sigma({\sqrt{\omega^\top M^{-1}\omega + \varepsilon}})-\frac{1}{2}\right)}{\sqrt{\omega^\top M^{-1}\omega + \varepsilon}}\omega$
       \STATE {\bfseries Return} $(\tilde{M},\tilde{\omega})$
    \end{algorithmic}
\end{algorithm}

\subsection{Polar Kernel Sampling strategies}
\label{sec: polar sampling strategies}

In the continuum, the support $\Delta_x$ is given by the locations $sy_x(\theta,\gamma)$ where $s\in [0,1]$ and $\theta\in [0,2\pi]$. We design two complementary approaches to sample $k\times k$ grid points, as illustrated in \cref{fig: finsler utb analytical offsets}.

\paragraph{Grid sampling.} To provide a $k\times k$ sampled kernel, a natural approach is to uniformly grid sample $s$ and $\theta$ with a $k\times k$ grid. 

\paragraph{Onion Peeling.} For very small $k$ however, e.g.\ $k=3$, polar grid sampling undersamples angles, unlike standard image convolutions using a regular $k\times k$ grid. For instance, $8$ angles are considered for $k=3$ in a regular grid. To better compare with standard convolutions for small $k$, we propose an onion peeling sampling strategy. A standard convolution sampling $k\times k$ grid can be understood as a succession of onion layers of pixels at $L^\infty$ distance $k'\in\{0,\hdots,\lfloor \tfrac{k-1}{2}\rfloor\}$ from the central pixel. In the $k'$-th layer, there are either $1$ pixel if $k'=0$ or $8k'$ pixels for $k'\ge 1$. We use this idea to design our onion peeling metric sampling: we sample $\lfloor \tfrac{k-1}{2}\rfloor+1$ values $s\in[0,1]$ uniformly, and for each layer index $k'\in\{0,\hdots,\lfloor \tfrac{k-1}{2}\rfloor\}$, we sample either the original point $x$, which is given by any $\theta$, if $k'=0$ or $8k'$ angles $\theta\in[0,2\pi)$ uniformly for $k'\ge 1$.

\begin{figure}[t]
    \centering
    \includegraphics[width=\columnwidth]{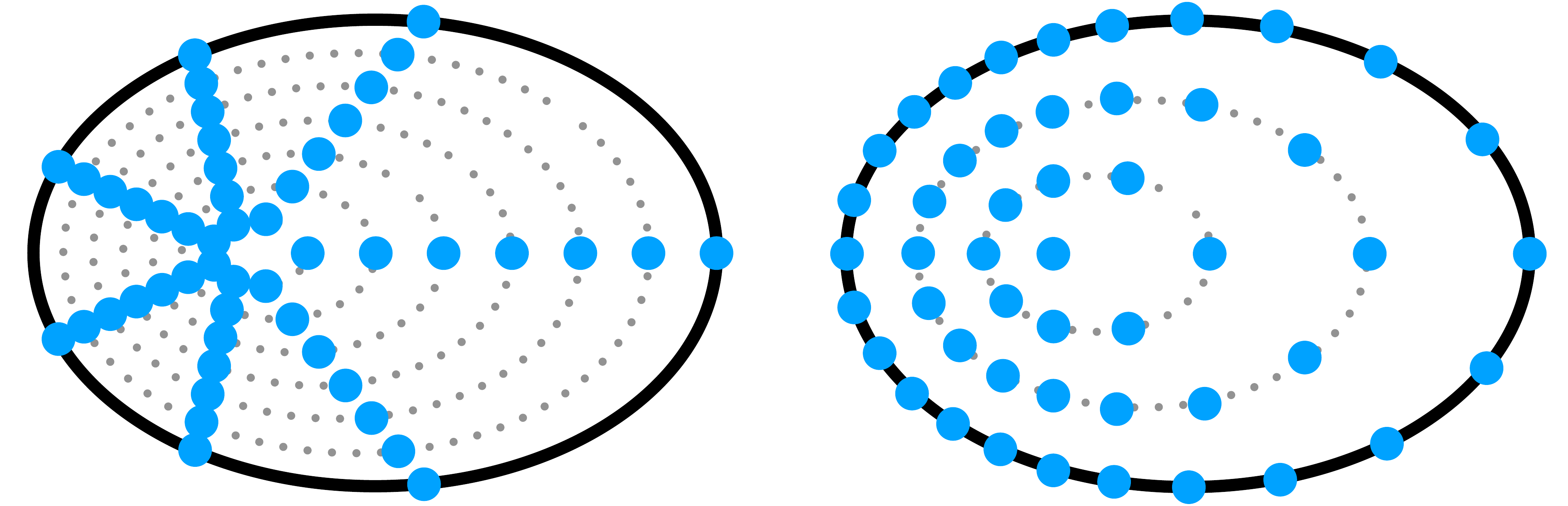}
    \caption{Polar kernel sampling strategies for $k\times k$ samples when $k=7$. Left: grid sampling polar coordinates. Right: onion peeling strategy, analogous to what is is done in the standard square $k\times k$ grid. In both cases, the position of each sample is analytically given by a simple closed-form formula depending solely on the Randers metric parameters $\gamma = (M, \omega)$.}
    \label{fig: finsler utb analytical offsets}
\end{figure}

In our denoising experiments (\cref{sec: denoising whole images with a single local convolution}), we work with larger $k$, thus we used a grid sampling strategy for them. In our classification experiments using well-established CNNs (\cref{sec: from single to stacked convolutions: cnn classification}), these neural architectures systematically rely on $3\times 3$ convolutions, thus we utilized an onion peeling strategy for them.

\subsection{Differentiating our Metric Convolution}
\label{sec: differentiating metric utb conv}

We focus on the continuous case from \cref{eq: metric utb f*g}. Differentiable changes of metric induce the variation of the unit ball in
\begin{equation}
    \frac{\partial y_x(\theta,\gamma)}{\delta \gamma} = -\frac{1}{(F_x^\gamma)^2(u_\theta)} \frac{\partial F_x^\gamma}{\partial\gamma}(u_\theta) u_\theta.
\end{equation}

Assuming that the convolution weights are fixed, i.e.\ $g(sy_x(\theta,\gamma)) = g_{\theta,s}$, differentiating the convolution with respect to the metric parameters gives
\begin{equation}
    \frac{\partial (f*g)}{\partial \gamma}(x) = \int_{s,\theta} \Big(s \frac{\partial y_x}{\partial \gamma} \Big)^\top\nabla f(x + s y_x(\theta,\gamma)) g_{\theta,s} ds d\theta.
\end{equation}
Likewise, if the metric is fixed, dependence on the weights is given by
\begin{equation}
    \frac{\partial (f*g)}{\partial g_{\theta,s}}(x) = f(x+sy_x(\theta,\gamma))
\end{equation}

If needed, we can learn the parameters of the metric or the values of the kernel by gradient descent on some loss function $L$ according to the dynamics
\begin{equation}
    \begin{cases}
        \frac{\partial g_{\theta,s}}{\partial t} = - \frac{\partial L}{\partial g_{\theta, s}} \\
        \frac{\partial \gamma}{\partial t} = - \frac{\partial L}{\partial \gamma}.
    \end{cases}
\end{equation}
Naturally, we can generalise this to any descent-based optimisation algorithm and discrete optimisation steps.

We can easily extend this discussion to our discretised version of \cref{eq: metric utb f*g}.

\subsection{Shift-Equivariance of Metric Convolutions}
\label{sec: shift-equivariance metric convolutions}

Denote $\mathcal{T}_t$ the shift-operator of vector $t$, defined as $\mathcal{T}_t(f):x\mapsto f(x-t)$. We say that the convolution of $f$ with kernel $g$ is shift-equivariant if for any vector $t$, we have\footnote{Note that shift-equivariance should not be confused with shift-invariance, which is defined as $\mathcal{T}_t(f*g) = (f*g)$. Old terminology, prior to the rise of deep learning, used to call shift-invariance what is now properly called shift-equivariance.} $\mathcal{T}_t(f*g) = (\mathcal{T}_t(f) * g)$. This means that shifting the input image and applying the convolution operations yields the same output image, albeit shifted by the same amount.
Note that when mentioning shift-equivariance in image convolution, only periodic padding guarantees perfect shift-equivariance, otherwise the operation is approximately shift-equivariant with small differences arising at the borders of the image. Generalising the argument for standard convolutions, we here implicitly assume periodic padding, and the same issue will arise for other types of common padding, such as constant $0$ padding.
Let us formally prove \cref{th: shift equivariant}. 

\begin{proof}
    First, the value of $g$ at the cell with polar coordinates $(s_x,\theta)$, where $s_x = sy_x(\theta,\gamma)$, is independent of $x$. This implies that the convolution kernels $g$ in metric convolutions have shared weights. Schematically, the discrete $k\times k$ convolution kernel $g$ is an array of $k\times k$ values independent of $x$, that will be pointwise multiplied at each location $x$ with the $k\times k$ array of sampled values $f(s_x, \theta)$ for the $k\times k$ sampled locations $(s_x,\theta)$, and then summed together to yield the output convolution at $x$.
    
    Second, our metric parameters $\gamma$ are computed based solely on shift-equivariant objects. For heuristic designs, these objects can be image gradient or its norm, which are naturally shift-equivariant. For the advanced learnable design, we use an intermediate standard convolution to extract $\gamma$ directly. Since standard convolutions are shift-equivariant, the computed array of metric parameters $\gamma$ at all locations $x$ will be shifted by the same amount as the input image. In turn, this implies that the unit tangent balls of our metric convolutions are shift-equivariant.
    
    Third, thanks to the universal Cartesian coordinate system of images, our sampling strategies are based on angles $\theta$ with the horizontal axis of the image domain $\Omega=[0,1]^2$, which is invariant to image shifts. This implies that our sampling scheme will always sample the same location of a given unit tangent ball, regardless of its position in the image.
    
    Combining these three properties, the output of the metric convolution of a shifted image $f$ with a kernel $g$ will be the same output as the original convolution with the same kernel $g$, albeit shifted by the same shift amount. This means that metric convolutions are shift-equivariant.
\end{proof}

\subsection{On Difficulties for Fast Differentiable Unit Geodesic Ball Computations}
\label{sec: difficulties fast diff UGB}

In contrast to the UTB case, computing the unit geodesic balls (UGB) is expensive as it is not given by a simple closed-form expression. It requires instead finding geodesic curves and then integrating along them. Many approaches exist for geodesic distance computations, and most of them require solving differential equations. They are usually either the Eikonal equation, which describes the propagation of a wave front on the manifold, or the heat equation, as initially heat diffuses along geodesics.

In the traditional Riemannian case, a wide variety of solvers exist, even differentiable ones, such as the recent differentiable fast marching algorithm \cite{benmansour2010derivatives,bertrand2023fast}, Varadhan's formula \cite{varadhan1967behavior} or the idea of \cite{crane2013geodesics} to flow heat initially in one small time step, normalise the obtained gradient field and interpret the normalised field as the tangential components of the geodesics. Unfortunately, existing solvers are too expensive to be used in reasonable applications, such as a neural network module, and we would need to apply these solvers at least as many times as there are pixels in the image since distances are computed from the single starting points $x$.

An extra layer of complexity arises when using the less common Finsler metrics, for which even discretisation of differentiable operators becomes tricky. We here discuss some of these difficulties when using the idea from \cite{crane2013geodesics}. In Randers geometry, deriving linear solvers to length-related differentiable equations becomes highly non trivial. In the Riemannian case, this is not an issue. For instance, the Riemannian heat equation $\frac{\partial f}{\partial t} = \mathrm{div}(\nabla f) = -\Delta_M f$ is governed by the linear Laplace-Beltrami operator $\Delta_M$. Many works handle its possible discretisation strategies, such as the popular cotangent weight scheme \cite{desbrun1999implicit,meyer2003discrete}. The linearity allows \cite{crane2013geodesics} to diffuse heat from a Dirac image $\delta_x$, that are one at pixel $x$ and zero everywhere else, by solving a set of linear equations to perform a single time-backwards iteration $(I - t\Delta_M) \delta_{x,t} = \delta_x$. The time-backwards operation essentially imagines what the heat should look like after a small time step should it have originated from the Dirac image. A forward time difference scheme would struggle to do so as the gradient of the image is zero almost everywhere, and in turn the Laplace-Beltrami operator, and so almost no heat would be flown that way in a single step. This elegant solution becomes highly non trivial in the Randers case. This is why in our metric UGB convolution, we modify the approach from \cite{crane2013geodesics} and revert to many smaller time forward iterations to flow heat. Also, as it is unclear how to discretise operators in Randers geometry, we use a local solution rather than the global one, which provides an approximating solution satisfying some properties of the differential equation.

\subsection{Details on our Naive Implementation of our Metric Unit Geodesic Balls Convolutions}
\label{sec: details implementation metric ugb conv}

Finding global solutions to the Finsler heat equation (\cref{eq: finsler heat equation}) is difficult. However, we can easily provide local solutions \cite{ohta2009heat}. Local solutions merely satisfy some local properties of the differential equation. Our local solution, the Finsler-Gauss kernel $h_x$ \cite{ohta2009heat,yang2018geodesic}, is explicit and is given by
\begin{equation}
    h_x(y) = \frac{1}{\mathcal{Z}(x)}\frac{1}{t}e^{-\frac{(F_x^*(y))^2}{4t}}
\end{equation}
where $F_x^*$ is the dual Finsler metric and is equal to the invert metric in the Riemannian case, and $\mathcal{Z}(x) = \int\tfrac{1}{t} e^{-\tfrac{(F_x^*(y))^2}{4t}}dy$ is a normalisation factor. The dual metric is beyond the scope of this paper so we refer to the \cref{sec: dual randers metric} for more details on it. We will just point out that the dual of a Randers metric is also a Randers metric with explicit parameters $(M^*,\omega^*)$. We then perform a standard convolution using the Finsler-Gauss kernel to diffuse the heat from a Dirac image $\delta_x$, that are one at pixel $x$ and zero everywhere else, to produce the diffused Dirac image $\delta_{x,t}$ according to the update rule 
\begin{equation}
    \delta_{x,t+dt}(x') = \int \delta_{x,t}(x'+y) h_{x'}(y)dy.
\end{equation}
We can then compute the normalised gradient field $-\tfrac{\delta_{x,t}}{\lVert \delta_{x,t}\rVert_2}$ to get the unit speed geodesic flow field, from which we need to compute a unit ball. To get a differentiable sampling $\Delta_x$, we can flow a stencil of points $\Delta^{\textit{ref}}$ close to $x$ given by $P_x^\gamma(s,\theta,0) = s_0 s y_x(\theta,\gamma)$ where $s_0$ is an optional scaling factor and then flow for a fixed amount of time according to $\tfrac{\partial P}{\partial t} = -\tfrac{\delta_{x,t}}{\lVert \delta_{x,t}\rVert_2}$. This deforms the stencil and the obtained unit ball is no longer necessarily convex. If the initial stencil is sampled using $k\times k$ uniform polar grid points by analogy with the UTB case, the obtained stencil can be interpreted as covering the non convex unit geodesic ball (or its approximation).

This algorithm is significantly cheaper than traditional more accurate geodesic solvers, it is fully differentiable as in particular the unit ball is not obtained via a thresholding operation. However, it is still too costly to be used in real scenarii such as neural network modules. For instance, if the image has $256\times 256$ pixels, we need to diffuse $65536$ Dirac images of the same resolution and then flow $65536$ sets of stencil. This either quickly occupies all available memory in RAM for single modest commercial GPUs or implies a slow sequential bottleneck for simply computing a single convolution operation.

\section{Experiments}

\subsection{Implementation Considerations of Heuristic Geometric Designs of Metric Convolutions and Other Methods and Results}
\label{sec: implementation heuristic metric convs}

We here show how to design sample locations from geometry. We take uniform kernel weights $\tfrac{1}{k^2}$ and no learning is involved. We denoise the $256\times 256$ greyscale cameraman image using standard, dilated, and our metric UTB and UGB convolutions, along with deformable convolution using random offsets due to their lack of interpretability.

As mentioned in the main paper, a natural desire for the shape of unit balls when considering denoising is to be wide along the orthogonal gradient direction $\nabla f(x)^\perp$ and thin along the image gradient $\nabla f(x)$. This shape avoids blurring out edges. 

\paragraph{Unit Tangent Ball.} In the UTB case, unit balls are easily given in closed form. We can thus sample them directly without having to pass through the dual metric. Our anisotropic Riemannian metric of parameter $M$ favours the direction $\nabla f(x)^\perp$ by taking it as an eigenvector with smaller eigenvalue
\begin{equation}
    \label{eq: M unit tangent ball convolution experiment}
    M(x) \!=\! 
    R_{\theta_x} \!\!
    \begin{pmatrix}
        \iota\!\left(1+\alpha \tfrac{\lVert\nabla f(x)\rVert_2}{\max \lVert \nabla f \rVert_2}\right) & 0 \\
        0 & \frac{\iota}{1+\alpha \tfrac{\lVert\nabla f(x)\rVert_2}{\max \lVert \nabla f \rVert_2}}
    \end{pmatrix}\!\!
    R_{\theta_x}^\top,
\end{equation}
where $\iota = 0.1$ controls the average metric scale\footnote{If $\nabla f(x) = 0$ then $\iota = 0.1$ creates a ball of $10$ pixel edge radius.}, $\alpha = 100$ is an anisotropy gain factor, and $R_{\theta_x} $ is the rotation matrix with angle $\theta_x$ such that $(\cos\theta_x,\sin\theta_x)^\top = \tfrac{\nabla f(x)}{\lVert \nabla f(x)\rVert_2 + \varepsilon}$. The small $\varepsilon = 10^{-6}$ is added here for stability and to avoid dividing by $0$ in uniform areas of the image. The image gradient is computed using a Sobel filter of size $3\times 3$.

It is legitimate to consider symmetric neighbourhoods for denoising, i.e.\ $\omega\equiv 0$. Nevertheless, we also tried asymmetric metrics by controlling the scale of $\omega$. We first compute
\begin{equation}
    \tilde\omega(x) = \frac{\nabla f (x)^\top}{\lVert \nabla f(x)\rVert_2 + \varepsilon},
\end{equation}
and then for various scales $(1-\varepsilon_\omega)\in\{0,0.5,0.9\}$, we choose
\begin{equation}
    \label{eq: omega unit tangent ball convolution experiment}
    \omega(x) = (1-\varepsilon_\omega) \frac{\tilde \omega(x)}{\lVert \tilde\omega\rVert_{M^{-1}(x)} + \varepsilon}.
\end{equation}

\paragraph{Unit Geodesic Ball.} In the UGB case, our proof of concept algorithm requires the use of the dual metric to guide the heat flow of the Dirac images. After normalising this initial flow, we reflow a stencil of points to obtain the sample locations. We use the same metric $(M,\omega)$ as in the UTB case, except that now $\alpha=10$ and $\iota=1$. From this metric, we can explicitly compute $(M^*,\omega^*)$ given \cref{prop: dual randers formula}. For each pixel $x$, we diffuse the Dirac image $\delta_x$, equal to $1$ at pixel $x$ and $0$ everywhere, until $t=0.1$ with time steps $dt=0.01$. This means that starting from $\delta_x$ we iteratively convolved with the Finsler-Gauss kernel $10$ times.

We then define a stencil of $k\times k$ points $P(s,\theta,0) = s_0s y_x(\theta,(M^*,\omega^*)))$ with a uniform uniform grid of radial values $(s,\theta)\in[0,1]\times [0,2\pi)$ sampled $k\times k$ times. The stencil is to flow along the diffused Dirac image $\delta_{x,t}$. We diffuse for the same amount of time as the heat diffusion with the same time steps. Tuning $s_0$ is of interest but was not searched, we simply took $s_0=2$. For small $s_0$, all the points in the stencil will lie in the pixel $x$, and when using bilinear interpolation they can decentre early from $x$ before this drift is magnified. This behaviour is compatible with what is observed in deformable convolution. Too large values of $s_0$ will place stencil points in areas with unreliable normalised flow as they are at most barely reached by the heat diffusion. Flowing is done using a simple time forward scheme. 

Our suggested implementation is fully differentiable but prohibitively expensive from a space and time perspective for more common application such as neural network compatibility. It merely provides a proof of concept of UGB metric convolutions.

\paragraph{Unit Ball Plots.} To improve visualisation in \cref{fig: kernels cameraman}, we slightly modified the hyperparameters in the UTB and UGB case. In the UTB plots, we take $\alpha=10$. In the UGB plots, diffusion of the Dirac image is done with time steps $dt=0.1$ until time $t=0.5$, whereas the stencil is flown with time steps $dt=0.1$ until time $1$.

\paragraph{Details of Other Methods.} All convolutions used in \cref{fig:cameraman blurs no learn} use $k\times k$ samples with $k=11$. Dilated convolution has a dilation factor of $3$. 
Deformable convolution uses a dilation factor of $1$ and each offset of each kernel cell for each pixel location is randomly, independently, and uniformly chosen in $[-\tfrac{k}{2},\tfrac{k}{2}]^2$. 
Standard convolution covers a $k\times k$ pixel area and lacks interpolation, which further filters out noise, yet our method employs it. For fairness, we also test a non-standard interpolated convolution with $k\times k$ uniform samples covering a fixed smaller area. In practice, the interpolated standard convolution uses $k\times k$ uniform samples in a $5\times 5$ pixel area. We chose this area as it is a common size of non interpolated standard convolutions. As such, the interpolated standard convolution can be seen as a standard $5\times 5$ convolution equipped with the extra filtering from interpolation. Note that interpolated standard convolutions can also be understood as non-standard dilations with dilation factor less than $1$.

\paragraph{Results.} Results in \cref{fig:cameraman blurs no learn} show our metric convolutions outperforming standard, dilated, and deformable convolution.  Metric convolutions offer interpretability to flexible convolutions and strong geometric adaptable priors beneficial for basic tasks like noise filtering. The asymmetric drift component $\omega$ degrades performance for Gaussian noise filtering but could be useful in intrinsically asymmetric tasks like motion deblurring. 
Providing interpretable components to neural networks is fundamental as most models fundamentally lack interpretability. This is even the case for the simplest CNNs trained on the most simple tasks \cite{dages2023compass}.

\begin{figure*}[t]
    \centering
        \centerline{\includegraphics[width=0.7\textwidth]{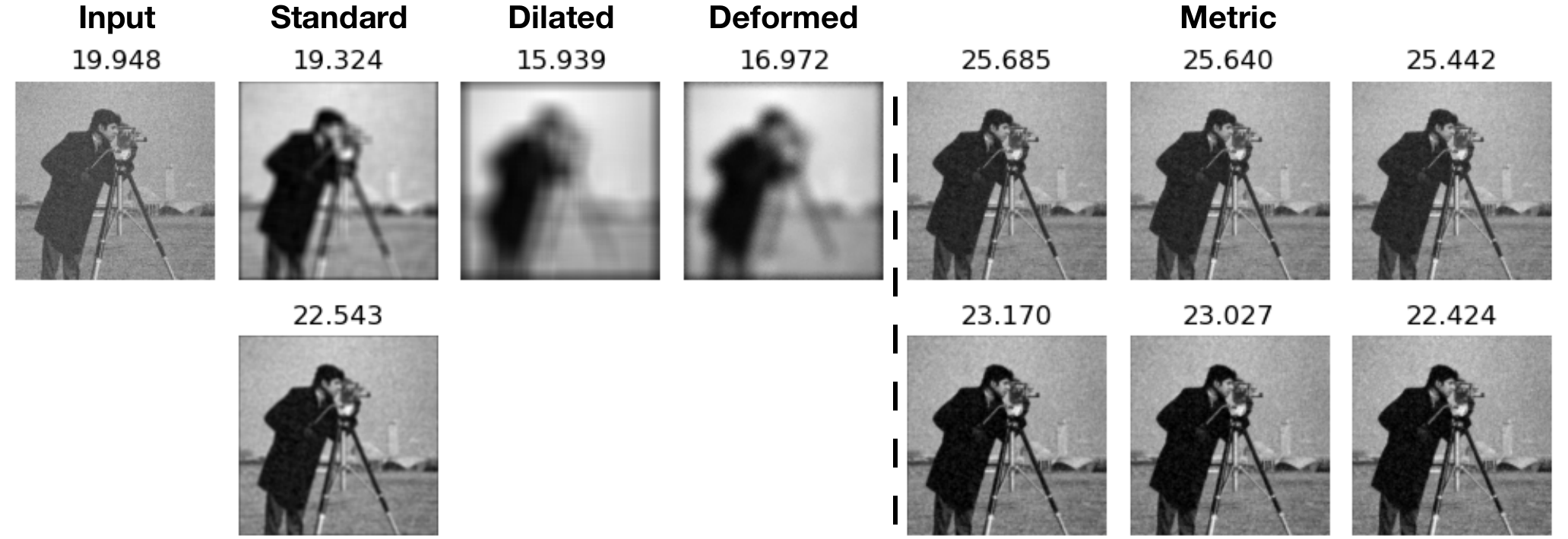}
        }
        \caption{
        Denoising results using $11\times 11$ samples from left to right: input, standard, dilated, randomly deformed, and metric convolutions with $\varepsilon_\omega = 1, 0.5, 0.1$. The bottom standard convolution uses interpolation on a smaller area. For our metric convolutions, the top row uses unit tangent balls, and the bottom row uses geodesic balls. Displayed values are the PSNR (higher is better).
        }
        \label{fig:cameraman blurs no learn}
\end{figure*}

\subsection{Learning Filtering on a Single Image}
\label{sec: learning filtering on a single image}

\begin{table*}[t]
    \centering
        \resizebox{0.9\textwidth}{!}{%
        \begin{small}
            \begin{sc}
                \begin{tabular}{cc ccccc c ccccc c ccccc}
                    \toprule
                     & & \multicolumn{5}{c}{\bf{Deformable}} & & \multicolumn{11}{c}{\bf{Unit tangent ball (ours)}}\\
                     & &\multicolumn{5}{c}{} & & \multicolumn{5}{c}{$\varepsilon_\omega = 0.9$} & & \multicolumn{5}{c}{$\varepsilon_\omega = 0.1$} \\
                      \cmidrule{3-7} \cmidrule{9-13} \cmidrule{15-19} 
                    & \diagbox[innerleftsep=0em,innerrightsep=0em]{$\sigma_n$}{$k$} & 5 & 11 & 31 & 51 & 121 & & 5 & 11 & 31 & 51 & 121 & & 5 & 11 & 31 & 51 & 121\\
                    \midrule
                     \multirow{3}{*}{$\textit{MSE}$}
                     & $0.1$  & 2.00e-3 & 4.49e-3 & 1.46e-2 & 2.24e-2 & 5.49e-2 & & \underline{1.68e-3} & \underline{1.57e-3} & \underline{1.54e-3} & \underline{1.54e-3} & \underline{1.54e-3} & & \pmb{1.60e-3} & \pmb{1.46e-3} & \pmb{1.44e-3} & \pmb{1.42e-3} & \pmb{1.41e-3}  \\
                     & $0.3$  & \pmb{7.26e-3} & 8.97e-3 & 1.99e-2 & 2.83e-2 & 6.18e-2 & & 8.13e-3 & \underline{7.43e-3} & \underline{7.58e-3} & \underline{8.64e-3} & \underline{8.28e-3} & & \underline{8.10e-3} & \pmb{7.32e-3} & \pmb{7.45e-3} & \pmb{7.52e-3} & \pmb{7.82e-3} \\
                     & $0.5$  & \pmb{7.86e-3} & \pmb{1.20e-2} & 2.23e-2 & 3.07e-2 & 6.50e-2 & & 1.85e-2 & \underline{1.68e-2} & \pmb{1.70e-2} & \pmb{1.70e-2} & \underline{1.72e-2} & & \underline{1.84e-2} & 1.70e-2 & \underline{1.73e-2} & \underline{1.72e-2} & \pmb{1.71e-2}  \\
                     \midrule
                     \multirow{3}{*}{$\delta_{\textit{MSE}}$} & $0.1$  & 5.3 & 3.0 & 1.6 & 1.3 & 0.9 & & 0.6 & 0.5 & 0.7 & 0.7 & 0.7 & & 0.7 & 0.6 & 0.7 & 0.9 & 1.1  \\
                     & $0.3$  & 265 & 74 & 28 & 18 & 6.6 & & 1.1 & 0.9 & 1.1 & 0.8 & 1.2 & & 1.3 & 1.1 & 1.3 & 1.4 & 1.5 \\
                     & $0.5$  & 4554 & 971 & 210 & 113 & 35 & & 0.8 & 1.3 & 1.8 & 2.0 & 2.2 & & 1.3 & 1.3 & 1.6 & 2.1 & 2.9  \\
                     \bottomrule
                \end{tabular}
            \end{sc}
        \end{small}
        }
    \caption{
    Denoising test MSE results (top) on a single noisy greyscale cameraman image when training on a different single noisy version, with noise level $\sigma_n$. Filters use $k\times k$ samples at each pixel location. Positional parameters of the convolutions are learnt, but weights are fixed. We also give the normalised generalisation gap $\delta_{\textit{MSE}}$. The parameter $\varepsilon_\omega\in(0,1]$ controls the tolerated amount of metric asymmetry, with $\varepsilon_\omega = 1$ being symmetric. For all numbers, lower is better. Best test MSE are in bold, and second best are underlined.
    }
    \label{tab: mse res single im dataset deform unit tangent no nn}
\end{table*}

In this experiment, we learn convolution kernel shapes for deformable and UTB convolutions with fixed uniform kernel weights using the Cholesky-based implementation. Learning is performed on a single noisy cameraman image with varying noise levels $\sigma_n\in\{0.1, 0.3, 0.5\}$ and tested on another noisy version of the same image with the same noise level. 
Gradient descent on the Mean Squared Error (MSE) loss is employed to optimise raw offsets and metric parameters, rather than those provided by a standard convolution as we only operate on cameraman images. This experiment evaluates if convolutions learn the image surface structure or just overfit to random noise. Convolutions use $k\times k$ samples with $k\in\{5,11,31,51,121\}$, and training runs for 100 iterations. Surprisingly, both methods require high learning rates (of magnitude $10^4$ to $10^8$), differing from typical rates or those in the original deformable convolution works \cite{dai2017deformable,zhu2019moredeformable}. A fixed learning rate did not yield meaningful results in all setups, necessitating a search for a good learning rate each time (see \cref{sec: lr denoising} for details).

Full quantitative performance is given in \cref{tab: mse res single im dataset deform unit tangent no nn}, where we give the MSE on the train and test image, along with the normalised generalisation gap $\delta_{\textit{MSE}} = \frac{\textit{MSE}_{\textit{test}} - \textit{MSE}_{\textit{train}}}{\textit{MSE}_{\textit{train}}}$. We also provide full qualitative results in \cref{sec: vis res learning single im denoising}.

\subsection{Training Details for Learning Filtering on a Dataset}
\label{sec: training details denoinsing dataset}

We trained all models using stochastic gradient descent for 100 epochs on the MSE loss with a learning rate chosen through logarithmic grid search (see \cref{sec: lr denoising}) that is the same for the sample locations and the kernel weights following the default methodology of \cite{dai2017deformable}. Training involved various kernel sizes $k\in\{5,11,31\}$ on a single small commercial GPU. For $k=31$, the batch size was reduced to $4$ to fit GPU memory, and was $32$ otherwise.

\subsection{Learning Rate Search for Denoising Convolution Filters}
\label{sec: lr denoising}

The learning rate for denoising convolution filters, whether they be of deformable convolution or our metric UTB convolution, or having fixed or learnable kernel weights, were found with a learning rate finder using a logarithmic grid search on the train data for a single epoch \cite{smith2017cyclical}.
We report in \cref{tab: lr single im dataset deform unit tangent no nn,tab: lr small dataset deform unit tangent no nn} the chosen learning rates.

\begin{table*}[ht]
    \centering
        \resizebox{0.8\textwidth}{!}{%
        \begin{small}
            \begin{sc}
                \begin{tabular}{c ccccc c ccccc c ccccc}
                    \toprule
                     & \multicolumn{5}{c}{\bf{Deformable}} & & \multicolumn{11}{c}{\bf{Unit tangent ball (ours)}}\\
                     &\multicolumn{5}{c}{} & & \multicolumn{5}{c}{$\varepsilon_\omega = 0.9$} & & \multicolumn{5}{c}{$\varepsilon_\omega = 0.1$} \\
                      \cmidrule{2-6} \cmidrule{8-12} \cmidrule{14-18} 
                    \diagbox[innerleftsep=0em,innerrightsep=0em]{$\sigma_n$}{$k$} & 5 & 11 & 31 & 51 & 121 & & 5 & 11 & 31 & 51 & 121 & & 5 & 11 & 31 & 51 & 121\\
                    \midrule
                     $0.1$  & 4.6e6 & 2.4e7 & 1.9e8 & 6.1e8 & 1.1e9 & & 3.7e5 & 4.0e5 & 4.5e5 & 5.7e5 & 5.7e5 & & 2.5e5 & 3.0e5 & 4.0e5 & 4.0e5 & 6.0e5  \\
                     $0.3$  & 1.5e6 & 7.4e6 & 6.0e7 & 1.9e8 & 9.8e8 & & 1.0e4 & 1.0e4 & 1.0e4 & 1.2e5 & 1.4e5 & & 1.0e4 & 1.0e4 & 1.0e4 & 1.0e4 & 1.1e5 \\
                     $0.5$  & 5.7e5 & 2.9e6 & 2.4e7 & 7.6e7 & 4.9e8 & & 1.0e3 & 5.0e3 & 1.0e4 & 1.0e4 & 1.0e4 & & 2.0e3 & 5.0e3 & 1.0e4 & 1.0e4 & 1.0e4 \\
                     \bottomrule
                \end{tabular}
            \end{sc}
        \end{small}
        }
    \caption{Chosen learning rates for training the positional parameters on a single noisy image in the experiments of \cref{tab: mse res single im dataset deform unit tangent no nn}.
    }
    \label{tab: lr single im dataset deform unit tangent no nn}
\end{table*}

\begin{table*}[ht]
    \centering
    \resizebox{0.95\textwidth}{!}{%
    \begin{small}
        \begin{sc}  
            \begin{tabular}{cc ccc c ccc c ccc c ccc c ccc c ccc}
                \toprule
                     & & \multicolumn{7}{c}{\bf{Deformable}} & & \multicolumn{15}{c}{\bf{Unit tangent ball (ours)}}\\
                     & &\multicolumn{3}{c}{} & & \multicolumn{3}{c}{} & & \multicolumn{7}{c}{$\varepsilon_\omega=0.1$} & & \multicolumn{7}{c}{$\varepsilon_\omega=0.9$} \\
                     \cmidrule{11-17} \cmidrule{19-25}
                     & &\multicolumn{3}{c}{FKW} & & \multicolumn{3}{c}{LKW} & & \multicolumn{3}{c}{FKW} & & \multicolumn{3}{c}{LKW} & & \multicolumn{3}{c}{FKW} & & \multicolumn{3}{c}{LKW} \\
                      \cmidrule{3-5} \cmidrule{7-9} \cmidrule{11-13} \cmidrule{15-17} \cmidrule{19-21} \cmidrule{23-25} 
                     & \diagbox[innerleftsep=0em,innerrightsep=0em]{$\sigma_n$}{$k$} & 5 & 11 & 31 & & 5 & 11 & 31 & & 5 & 11 & 31 & & 5 & 11 & 31 & & 5 & 11 & 31 & & 5 & 11 & 31 \\
                     \midrule
                     \multirow{3}{*}{BSDS300} & $0.1$  & 5.7e2 & 3.4e3 & 1.0e4 & & 1.0e-2 & 1.0e-2 & 6.6e-4 & & 6.0e0 & 3.1e0 & 7.0e-1 & & 1.0e-2 & 1.0e-2 & 1.0e-3 & & 6.1e0 & 1.1e0 & 1.0e0 & & 5.0e-2 & 1.0e-2 & 1.0e-3  \\
                     & $0.3$  & 3.2e2 & 1.5e3 & 1.0e4 & & 1.0e-2 & 1.0e-2 & 6.6e-4 & & 1.0e0 & 1.0e-1 & 1.0e-1 & & 1.0e-1 & 1.0e-2 & 1.0e-3 & & 1.0e0 & 1.1e0 & 1.0e-1 & & 1.0e-2 & 1.0e-2 & 1.0e-3   \\
                     & $0.5$  & 2.5e2 & 1.5e3 & 1.0e4 & & 1.0e-2 & 1.0e-2 & 6.6e-4 & & 6.1e-1 & 3.2e-1 & 4.5e-2 & & 1.0e-2 & 1.0e-2 & 1.0e-3 & & 1.0e0 & 3.8e-1 & 3.1e-2 & & 1.0e-2 & 1.0e-2 & 1.0e-3 \\
                     \midrule
                     \multirow{3}{*}{PascalVOC} & $0.1$ & 4.1e2 & 4.0e3 & 4.0e3 & & 1.0e-2 & 1.0e-2 & 1.0e-3 & & 1.2e1 & 4.3e0 & 5.0e-2 & & 1.0e-1 & 1.0e-2 & 1.0e-3 & & 1.4e1 & 1.5e0 & 1.0e0 & & 1.0e-2 & 1.0e-2 & 1.0e-3  \\
                     & $0.3$  & 3.0e2 & 4.0e3 & 4.0e3 & & 1.0e-2 & 1.0e-2 & 1.0e-3 & & 1.5e0 & 1.0e-1 & 5.0e-2 & & 1.0e-2 & 1.0e-2 & 1.0e-3 & & 4.0e0 & 5.0e-1 & 1.0e-1 & & 1.0e-1 & 1.0e-2 & 1.0e-3   \\
                     & $0.5$  & 4.1e2 & 1.1e3 & 4.3e3 & & 1.0e-2 & 1.0e-2 & 1.0e-3 & & 4.3e-1 & 1.0e-2 & 1.0e-2 & & 1.0e-2 & 1.0e-2 & 1.0e-3 & & 8.1e0 & 3.8e-1 & 1.0e-2 & & 1.0e-2 & 5.0e-3 & 1.0e-3 \\ 
                    \bottomrule
            \end{tabular}
        \end{sc}
    \end{small}
    }
    \caption{
    Chosen learning rates for training the positional parameters on noisy image datasets in the experiments of \cref{tab: mse and delta mse res small dataset deform unit tangent no nn}.
    }
    \label{tab: lr small dataset deform unit tangent no nn}
    \vspace{-1em}
\end{table*}

\subsection{Further Visual Results on Learning Convolutions on a Single Image}
\label{sec: vis res learning single im denoising}
We provide in \cref{fig: cameraman deform vs tangent sigma 01 k 5 11 31,fig: cameraman deform vs tangent sigma 01 k 51 121,fig: cameraman deform vs tangent sigma 03 k 5 11 31,fig: cameraman deform vs tangent sigma 03 k 51 121,fig: cameraman deform vs tangent sigma 05 k 5 11 31,fig: cameraman deform vs tangent sigma 05 k 51 121} visual comparisons of learnt deformable and metric UTB convolutions, when learning only the shape of the convolution and keeping the filtering weights fixed. These results correspond to the quantitative ones in \cref{tab: mse res single im dataset deform unit tangent no nn}.

In each of these figures, two consecutive rows of plots correspond to results with a fixed noise standard deviation $\sigma_n$ and number of samples $k\times k$ for the convolution. Each of these sets of two rows of plots are organised as follows. Top left is the groundtruth image and bottom left is the train and test MSE during training. Starting from the second column, the top row corresponds to train whereas the bottom one refers to test. Starting from the second column, from left to right: input noisy image, deformable convolution result, our metric UTB convolution results with different $\varepsilon_\omega\in\{0.9,0.1\}$ controlling the scale of $\omega$, with $\omega \equiv 0$ for $\varepsilon_\omega = 1$. Numbers provided correspond to the PSNR with respect to the groundtruth image, with higher scores being better.

Note that using an extremely high number of samples, e.g.\ $121 \times 121$, does not increase the size of the sampling domain for our metric UTB convolution as the unit ball does not depend on the sample size. Larger kernels imply more samples in the same unit ball. On the other hand, deformable convolution suffers from high number of samples as it relies on the reference template of $121\times 121$ pixels, which in our experiments is half the image size in width. As such, in many pixel locations, the reference support overlaps with the outside of the image, where it is padded to $0$, which makes it impossible for gradient-descent based strategies to learn meaningful offsets in such cases.

\subsection{Implementation Considerations of CNN Classification}
\label{sec: cnn mnist implementation details}

We here use the common ResNet terminology. All traditional ResNet architectures are a succession of \textit{layers}. The initial layer, sometimes called \textit{conv1}, has a single convolution module, along with other operations. Following the initial layer, comes a succession of four layers, named \textit{layer1}, \textit{layer2}, \textit{layer3}, \textit{layer4}. Each of these layers consist in a sequence of convolution \textit{blocks}. These blocks can be \textit{basic} for smaller networks like ResNet18, or \textit{bottleneck} ones for larger versions like ResNet50 and ResNet152. Each block of a network has the same structure, up to a final pooling. Basic blocks have two $3\times 3$ convolution modules, whereas bottleneck blocks have only one, when no downsampling is involved. None of these convolutions have an additional bias term. The first $3\times 3$ convolution of the first block of every layer has a stride of $2$, whereas all the other have a stride of $1$. All convolutions use a dilation of $1$ (no dilation). 

In the experiments of \cref{tab: classif resnet18 mnist fmnist,tab: classif resnet18 cifar10 cifar100}, we directly replace only the $3\times 3$ convolution modules with their $3\times 3$ adaptive counter-parts, i.e.\ deformable, shifted, and our UTB convolution. We use the same number of input and output channels and no bias. Only the convolutions in layer2, layer3, and layer4 are changed. Those in layer1 or conv1 are unchanged and remain standard. We also change the stride of the first convolution of the first block of layer4 from $2$ to $1$ and to avoid decreasing the receptive field we increase its dilation from $1$ to $2$. A dilation different from $1$ impacts the position of the reference kernel $\Delta^{\textit{ref}}$ of deformable and shifted convolutions, and does not impact our metric convolution. The methodology described here is a direct imitation of that of \cite{dai2017deformable,zhu2019moredeformable,yu2022entire}. However, unlike \cite{zhu2019moredeformable,yu2022entire}, we do not use modulation for simplicity as explained in the main paper: we wish to preserve the weight sharing assumption and sample uniformly the unit balls.

We propose to initialise our metric UTB convolution modules in the following way. Denoting $c_{\textit{in}}$ the number of input channels, kernel weights are initialised (and fixed in FKW) to $z_{k,c_{\textit{in}}} = \tfrac{1}{c_{\textit{in}} k^2}$. As for the weights of intermediate standard convolution with $5$ output channels computing the metric parameters $L_{1,1}$, $L_{1,2}$, $L_{2,2}$, $\omega_1$, and $\omega_2$ in this order, we initialise them as follows per output channel: the first and third ones have uniform weights set to $z_{k,c_{\textit{in}}}$, and the other ones are set uniformly to $\varepsilon = 10^{-6}$. In particular, this means that $\omega\approx 0$ initially, and the network must learn how much asymmetry is best. For simplicity however, we took $\omega=0$ always, i.e.\ restricting the metric to Riemannian ones, by taking $\varepsilon_\omega=1$. We also took $\varepsilon_L = 0.01$.

Like in the previous experiments, we test both fixing the kernel weights (FKW) of the non-standard convolutions to uniform values and learning only the sample locations, or learn simultaneously sample locations and the weights (LKW). Note that FKW has never been tested in the community of non-standard convolutions for neural networks. Prior works \cite{li2020anisotropic} start only with pretrained weights, up to module conversion, obtained on ImageNet \cite{deng2009imagenet} classification with vanilla modules. We argue that such a methodology does not properly reflect the strengths of convolutions with changeable supports. Indeed, we only switch a convolution with another one, thus the obtained network is still a CNN, albeit non-standard and theoretically more general. It should thus still provide good results when weights are learned from scratch. We thus train either from scratch (SC) or do transfer learning (TL) by starting from pretrained weights obtained on ImageNet.

All networks are trained for $240$ epochs with the Adam optimiser \cite{kingma2014adam} on the cross-entropy loss. We take a batch size of $128$, a base learning rate of $\eta = 0.0001$, and we use cosine annealing \cite{loshchilov2016sgdr} for scheduling the learning rate with maximal temperature $T_\mathrm{max} = 240$ as is commonly done. Following the common practice, images fed to the networks are centred and normalised following the dataset mean and standard deviation. 
When training on CIFAR with a single GTX 2080 Ti GPU, our metric CNN takes 7h and uses 1021MB GPU memory, compared to 5.5 hours and 940MB for a CNN using deformable convolutions (see \cref{tab: memory time cifar}). Although more expensive than deformable convolution, it is still faster than FlexConv  while providing superior performance \cref{tab: metric utb cnn vs flexconv}.
Note that our unoptimised code leaves room for improvements. Optimising our code (e.g.\ fused CUDA kernels or C level code) could greatly improve speed. To reduce memory usage, analytical offsets (and the convolution result) could be computed in a small loop rather than being loaded simultaneously in memory, unlike in deformable convolution.

Note though that although we only used fairly small datasets due to our technical limitations, they are nevertheless large enough to provide valuable insights \cite{dages2023model,dages2025model}.

\begin{table}[]
    \centering
    \resizebox{0.8\columnwidth}{!}{%
    \begin{sc}
    \begin{tabular}{ccc}
        \toprule
        & \bf{Memory (MB)} & \bf{Time (s)} \\
        \midrule
        \bf{FlexConv} & - & 127 \\
        \midrule
        \bf{Deformable} & 940 & 83\\
        \bf{Metric UTB (Ours)} & 1021 & 105 \\
        \bottomrule
    \end{tabular}
    \end{sc}
    }%
    \caption{Peak GPU memory used and single epoch training time on CIFAR. FlexConv training time is taken from Tab. 3 in the original paper \cite{romero2021flexconv} for its best neural network FlexNet-16, but its maximum GPU utilisation is not reported there. Our method uses comparable memory to the more general deformable convolution but is slightly slower. Nevertheless, it is faster than the even more general FlexConv, while providing superior performance (see \cref{tab: metric utb cnn vs flexconv}), even though FlexConv uses large convolution kernels with weights computed from expensive MLPs and then modulated by Gaussian masks.}
    \label{tab: memory time cifar}
\end{table}

\begin{table}[]
    \centering
    \resizebox{0.8\columnwidth}{!}{%
    \begin{sc}
    \begin{tabular}{ccc}
        \toprule
        & \bf{MNIST} & \bf{CIFAR10} \\
        \midrule
        \bf{FlexConv} & 99.7 $\pm$ 0.0  & 92.2 ($\pm 0.1\%$) \\
        \bf{Metric UTB (Ours)} & $99.7^\dagger$ & 93.1 ($\pm 0.1\%$)\\
        \bottomrule
    \end{tabular}
    \end{sc}
    }%
    \caption{Reported performance of FlexConv from the original paper \cite{romero2021flexconv} (Tabs. 3 and 9) with their best model FlexNet-16, compared to our metric UTB convolution CNN. Although our network only uses $3\times 3$ convolution samples and we did not use any modulation, our method outperforms FlexConv even though it is significantly more complex and expensive \cref{tab: memory time cifar}, requiring large convolution kernels, and thus high number of $k$, while also using Gaussian modulation. FlexConv results are averaged over three random seeds, whereas ours uses one on MNIST ($\dagger$) and eight on CIFAR. Our performance corresponds to those in \cref{tab: classif resnet18 mnist fmnist,tab: classif resnet18 cifar10 cifar100} with precision rounded to the first decimal, like the results reported for FlexConv \cite{romero2021flexconv}.}
    \label{tab: metric utb cnn vs flexconv}
    \vspace{-1em}
\end{table}

For all datasets, including both train and test splits, input images are first normalised according to the training dataset's mean and standard deviation.
Since MNIST and Fashion-MNIST are curated datasets with objects centred and roughly aligned, we do not need data augmentation to train the models. However, the natural images in CIFAR-10 and CIFAR-100 are not, and therefore we apply data augmentation on training images by randomly cropping the input image to a patch, resizing the patch to the full image size, and then randomly horizontally flipping the image. 

All CNNs with our metric UTB convolutions use the onion peeling sampling polar kernel sampling strategy (\cref{sec: polar sampling strategies}) and the metric computation from 7 numbers (\cref{alg: computing the metric from 7 numbers}), except for those with fixed kernel weights (FKW) which use the version with 6 numbers (\cref{alg: computing the metric from 6 numbers}).
On CIFAR-10 and CIFAR-100, we got marginally better results when training the networks with learnable kernel weights (LKW) using an L1 regularisation loss on the weights of the intermediate convolutions with a Lagragian coefficient of $5000$. On CIFAR-10 with learnable weights and transfer-learned (LKW-TL), we got even slightly better results when using 50 warmup epochs, where during the warmup the output of the intermediate convolution is multiplied by 0. This warmup imitates the baseline sampling strategy of a fixed kernel while still using our metric framework.

\subsection{Further Ablation Experiments of CNN Classification}
\label{sec: Further Ablation Experiments of CNN Classification}

\begin{table}[ht]
    \Large
    \centering
    \resizebox{\columnwidth}{!}{%
        \begin{tabular}{lllccc}
            \toprule
             & & & \textbf{ResNet18} & \textbf{ResNet50} & \textbf{ResNet152} \\
            \midrule
            \multirow{5}{*}{\textbf{TOP1}}
             & Standard             & 
             & 73,61 ($\pm0,31\%$)
             & 79,06 ($\pm0,29\%$)
             & 79,86 ($\pm0,30\%$) \\
             \cmidrule{2-6}
             & Deformable           & 
             & 73,55 ($\pm0,26\%$)
             & 78,15 ($\pm0,25\%$)
             & 79,73 ($\pm0,41\%$) \\
             & Shifted              & 
             & 73,15 ($\pm0,17\%$)
             & 78,33 ($\pm0,32\%$)
             & 79,70 ($\pm0,35\%$) \\
             \cmidrule{2-6}
             & \multirow{2}{*}{Metric UTB (Ours)} 
             & Randers
             & \textbf{74,20 ($\pm0,58\%$)}
             & \textbf{79,17 ($\pm0,60\%$)}
             & \textbf{80,56 ($\pm0,80\%$)} \\
             &                               
             & Riemann
             & \underline{74,19 ($\pm0,68\%$)}
             & \underline{79,11 ($\pm0,44\%$)}
             & \underline{80,27 ($\pm0,43\%$)} \\
            \midrule
            \multirow{5}{*}{\textbf{TOP5}}
             & Standard              
             & 
             & 91,83 ($\pm0,20\%$)
             & 94,76 ($\pm0,17\%$)
             & 94,53 ($\pm0,22\%$) \\
             \cmidrule{2-6}
             & Deformable           
             & 
             & 91,50 ($\pm0,23\%$)
             & 94,08 ($\pm0,17\%$)
             & 94,51 ($\pm0,23\%$) \\
             & Shifted              
             & 
             & 91,37 ($\pm0,15\%$)
             & 94,16 ($\pm0,15\%$)
             & 94,30 ($\pm0,15\%$) \\
             \cmidrule{2-6}
             & \multirow{2}{*}{Metric UTB (Ours)} 
             & Randers             
             & \textbf{92,31 ($\pm0,36\%$)}
             & \textbf{94,85 ($\pm0,34\%$)}
             & \textbf{94,94 ($\pm0,32\%$)} \\
             &  
             & Riemann             
             & \underline{92,26 ($\pm0,28\%$)}
             & \underline{94,81 ($\pm0,19\%$)}
             & \underline{94,75 ($\pm0,23\%$)} \\
            \bottomrule
        \end{tabular}
    }%
    \caption{Mean test accuracies of different ResNet architectures with replacement of convolutions only in the last layers (\textit{layer4}), using either standard or non‑standard convolutions, with 10 independent runs per configuration. Higher is better. In parenthesis is the standard deviation (lower is better).}
    \label{tab: CIFAR100 last layer replacement}
\end{table}

We here present additional ablation experiments on {$\textrm{CIFAR100}$} evaluating a different replacement strategy for non-standard convolution layers and the use of Riemannian metric convolutions instead of Finsler ones.

\paragraph{Replacement layers.} 
There is no standard practice for selecting which convolution layers of the CNN to replace with non-standard convolutions. However, it is common in the field to focus on deeper layers. Our design presented in the main paper and in \cref{sec: cnn mnist implementation details} follows the strategy of deformable convolution v2 \cite{zhu2019moredeformable}, which expands on the original version \cite{dai2017deformable}, where only the very last layers were replaced. While the original work \cite{dai2017deformable} found worsening or diminishing returns when replacing more layers, \cite{zhu2019moredeformable} argued this was due to the simplicity of the task and showed benefits when scaling up. To avoid misleading conclusions, we adopted the broader replacement strategy of \cite{zhu2019moredeformable} in the main paper.

That said, our tasks are simple, and replacing only the last layers, similar to \cite{dai2017deformable} and also \cite{li2020anisotropic}, would likely improve both speed and memory footprint while slightly boosting performance on these simple small-scale tasks. However, this might not reflect large-scale behaviour. Nevertheless, in 
\cref{tab: CIFAR100 last layer replacement}, we report CIFAR100 (LW-TL) results over 10 runs with this modified setup, where only the convolutions in \textit{layer4} are replaced, and we also test larger models.
Similar to \cite{dai2017deformable}, we obtain better performance when replacing only the last convolution layers. However, we caution against generalising this result to larger scale problems, similar to what happened with deformable convolution \cite{zhu2019moredeformable}.

As for full replacement, it is uncommon in the field. Even the recent advanced v3 version of deformable convolution -- InternImage \cite{wang2023internimage} tailored for large-scale foundation models -- has first standard convolutions for shrinking the resolution, and also uses them for downsampling between deformable blocks. In summary, full replacement is not done due sub-optimal performance, high memory cost from early high-resolution features, as offsets are computed per pixel, and redundancy in adapting low-level filters, e.g.\ edge filters. Later layers capture more complex, semantic features and benefit more from sampling location adaptation. On small tasks, early layers pretrained on ImageNet are more robust, so replacing only final layers, which are more task-dependent, is preferable. Replacing earlier layers is more interesting as the task scales.

\paragraph{Riemann or Finsler metric convolution.} 
Unlike standard and dilated convolutions, shifted and deformable convolutions sample asymmetrically around each pixel, requiring asymmetric metrics, i.e.\ Finsler, in our theory to offset unit balls
(\cref{fig: duck motivation adaptive kernels,fig: Finsler unit tangent balls,fig: overview unit ball conv,fig: adaptive neighbourhoods sampled}). We use Randers as Finsler metrics for their simple parametric form that generalises Riemannian ones.
However, we can also use symmetric metrics for our metric convolutions, such as Riemannian metrics, i.e.\ $\omega \equiv 0$ or in our implementation $\varepsilon_\omega = 1$. We present the mean test classification performance obtained with Riemannian metric convolutions in \cref{tab: CIFAR100 last layer replacement}, when replacing only the  convolutions in the last layers (\textit{layer4}). As expected, performance degrades slightly from using the more general Randers metric. However, our symmetric Riemannian metric convolutions outperform the theoretically more general asymmetric deformable and shifted convolutions. This further proves that adopting an explicit metric perspective to convolution is beneficial and induces a powerful geometric regularisation.

\begin{figure*}[t]
    \centering
        \centerline{\includegraphics[width=\textwidth]{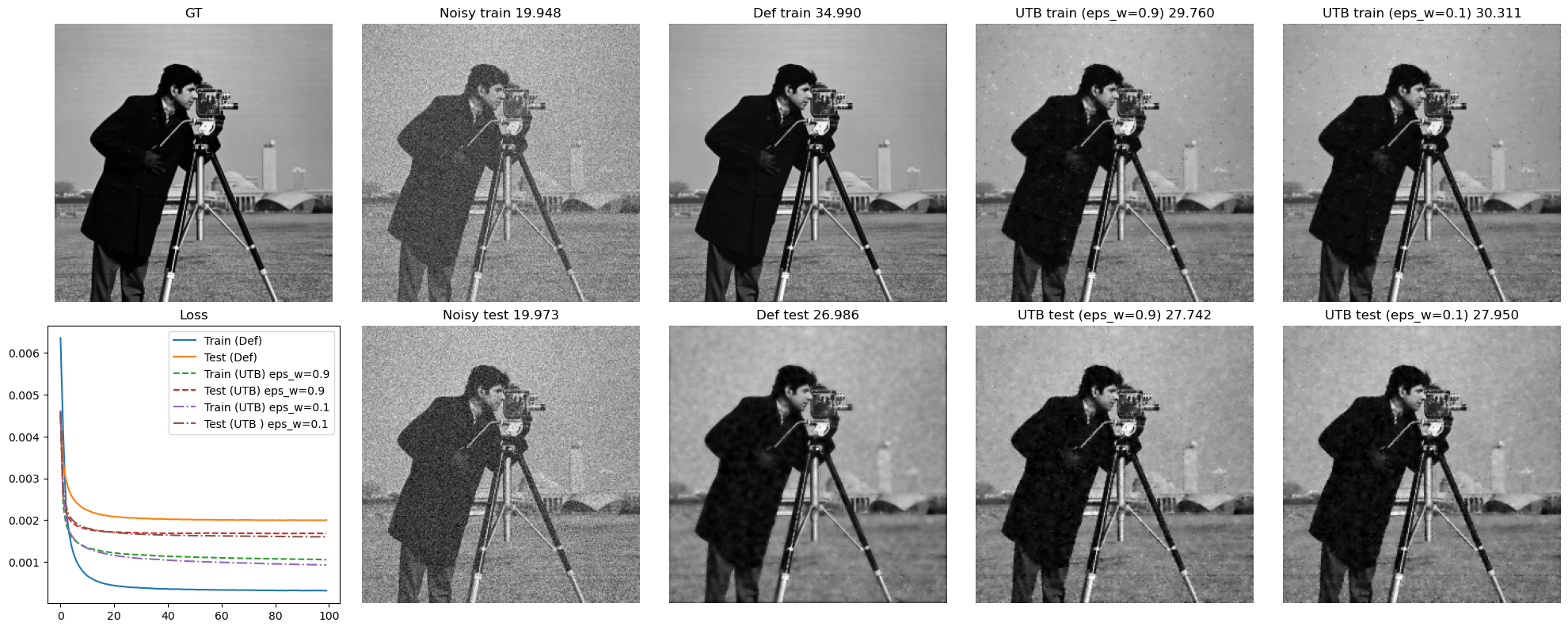}}
        \centerline{}
        \centerline{\includegraphics[width=\textwidth]{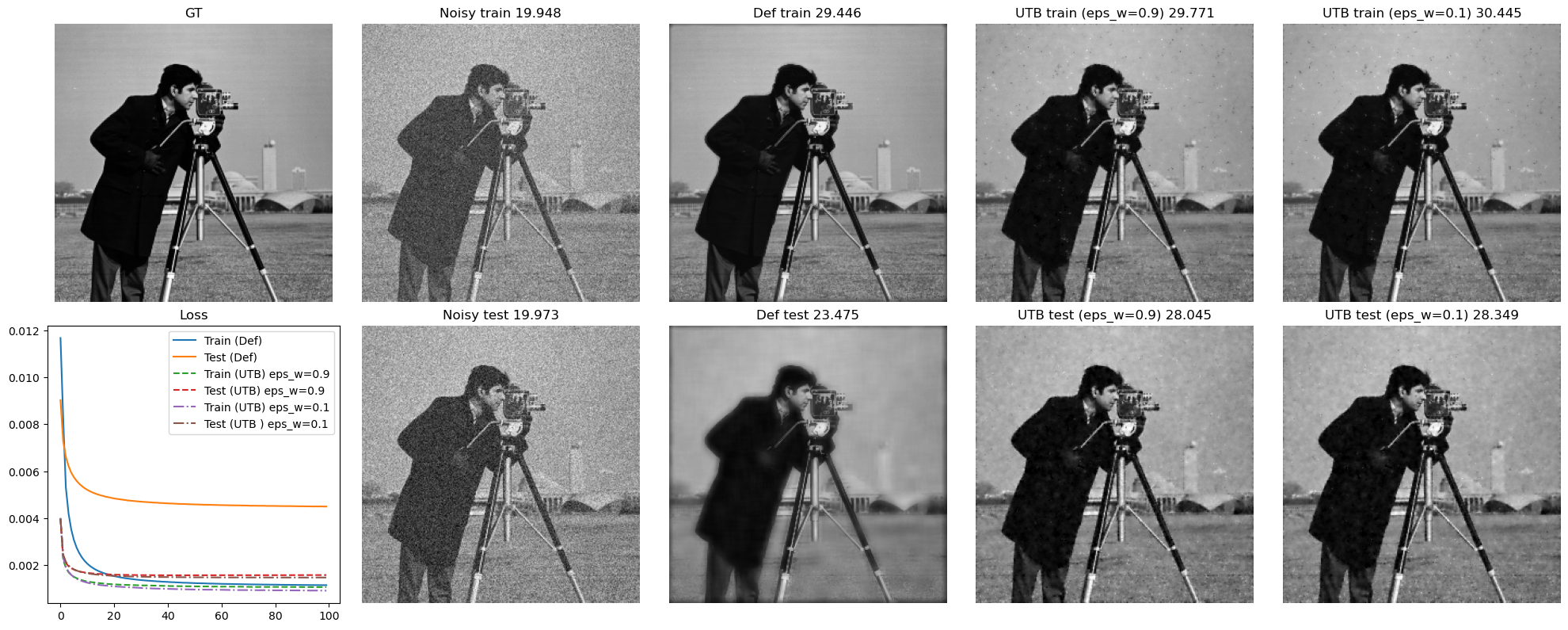}}
        \centerline{}
        \centerline{\includegraphics[width=\textwidth]{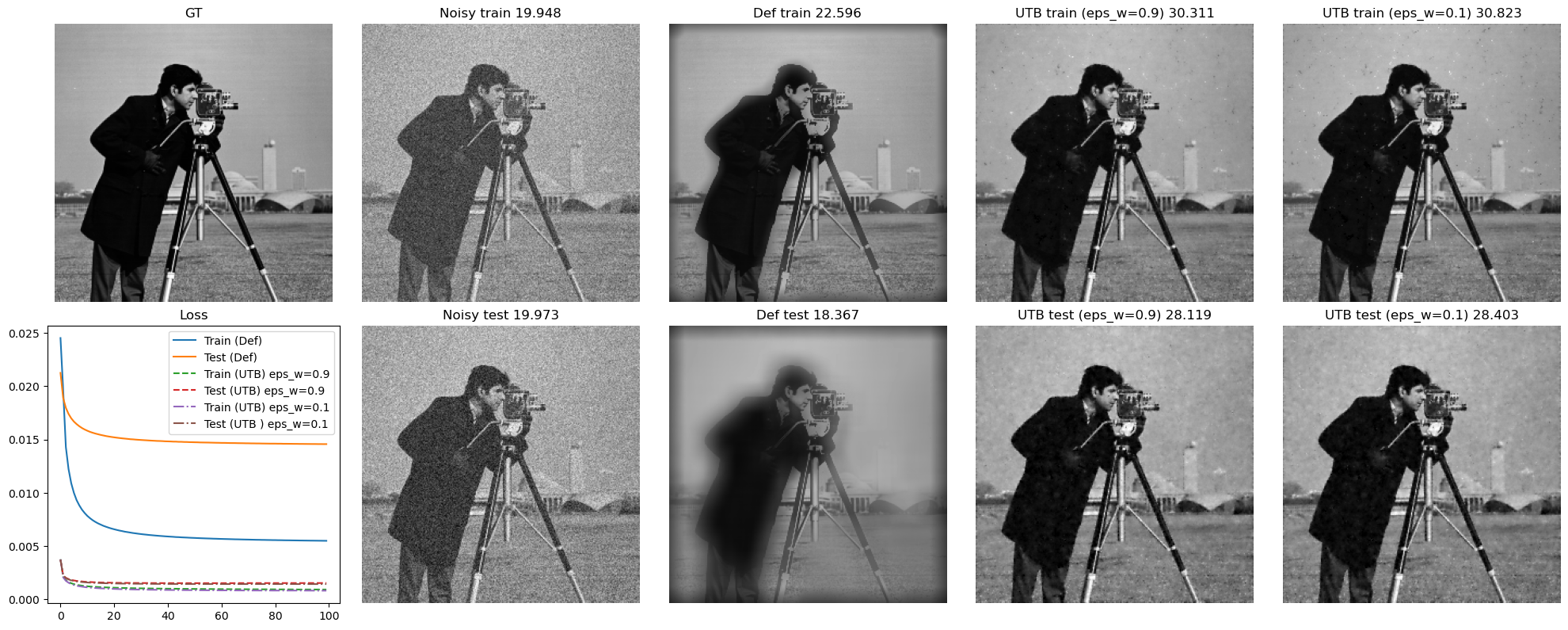}}
        \caption{
        Results of learnt deformable and our metric UTB convolutions with
        $\sigma_n = 0.1$ and $k = 5,11,31,$ from top to bottom.
        }
        \label{fig: cameraman deform vs tangent sigma 01 k 5 11 31}
\end{figure*}

\begin{figure*}[t]
    \centering
        \centerline{\includegraphics[width=\textwidth]{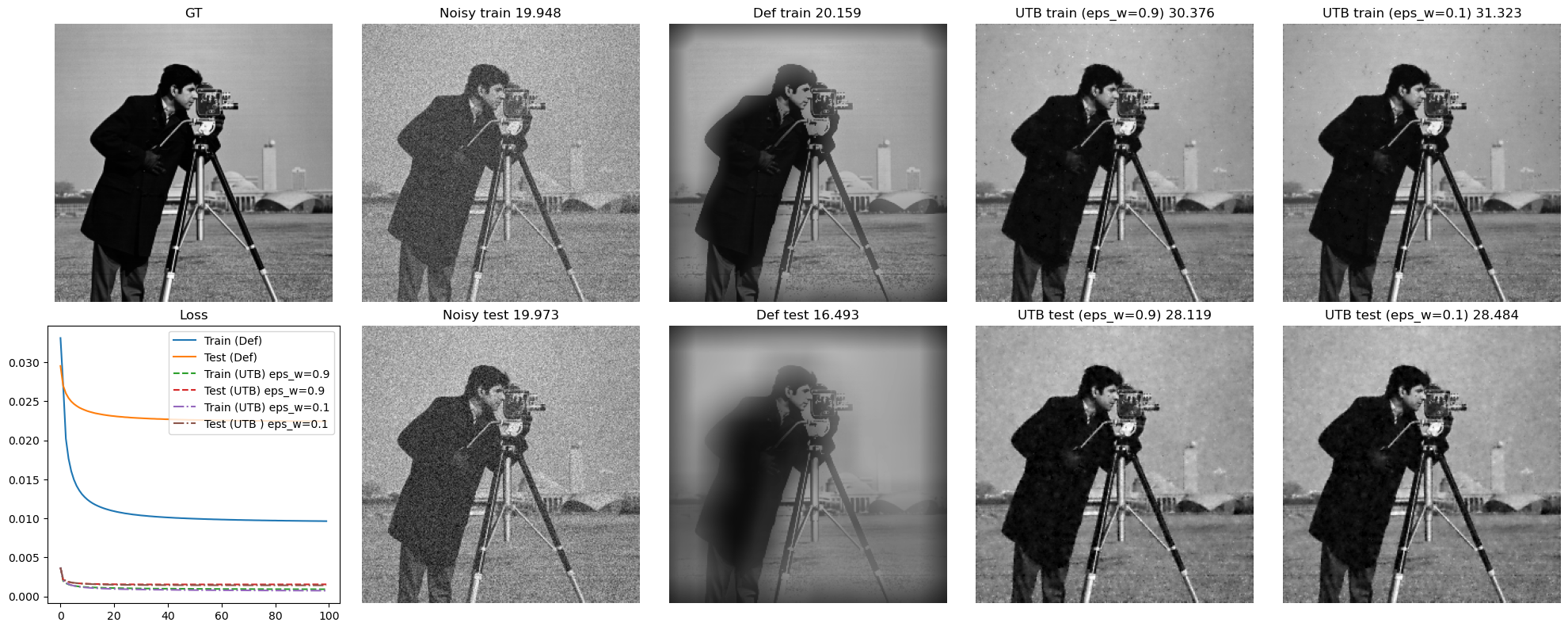}}
        \centerline{}
        \centerline{\includegraphics[width=\textwidth]{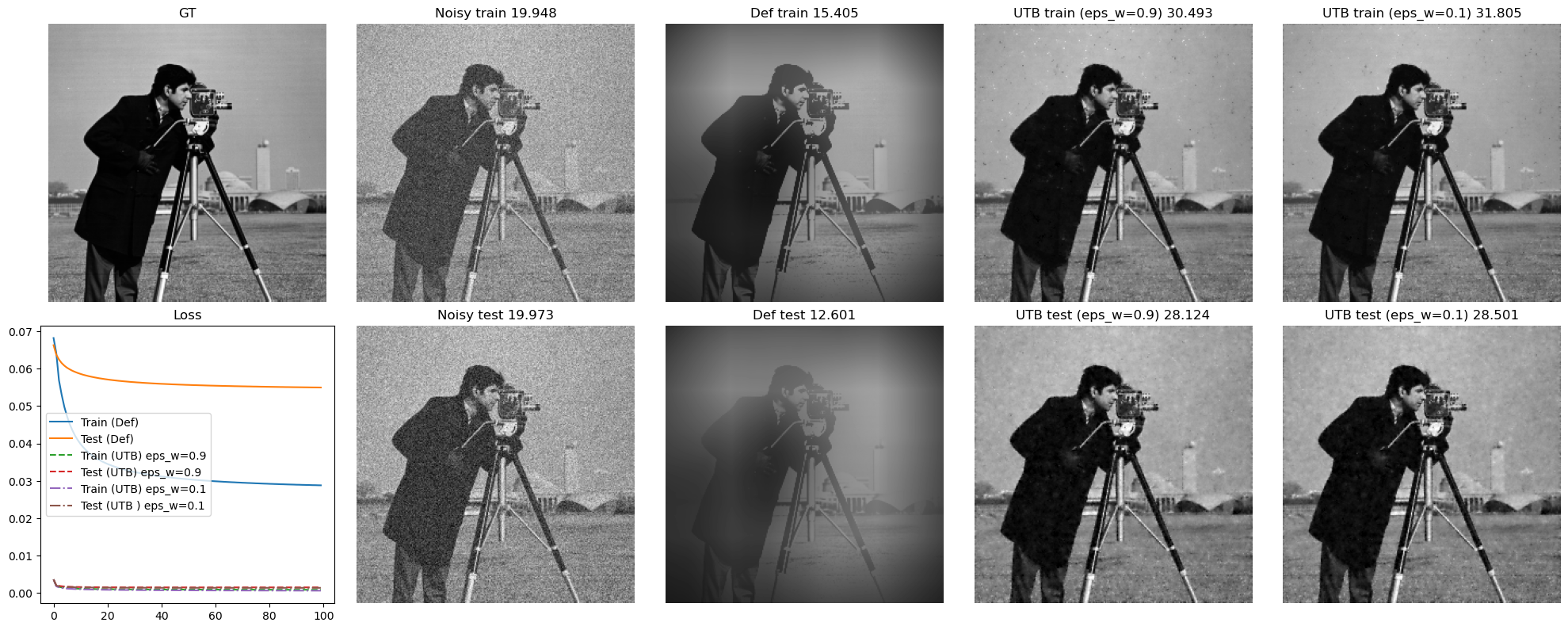}}
        \caption{Results of learnt deformable and our metric UTB convolutions with $\sigma_n = 0.1$ and $k = 51,121$ from top to bottom.
        }
        \label{fig: cameraman deform vs tangent sigma 01 k 51 121}
\end{figure*}

\begin{figure*}[t]
    \centering
        \centerline{\includegraphics[width=\textwidth]{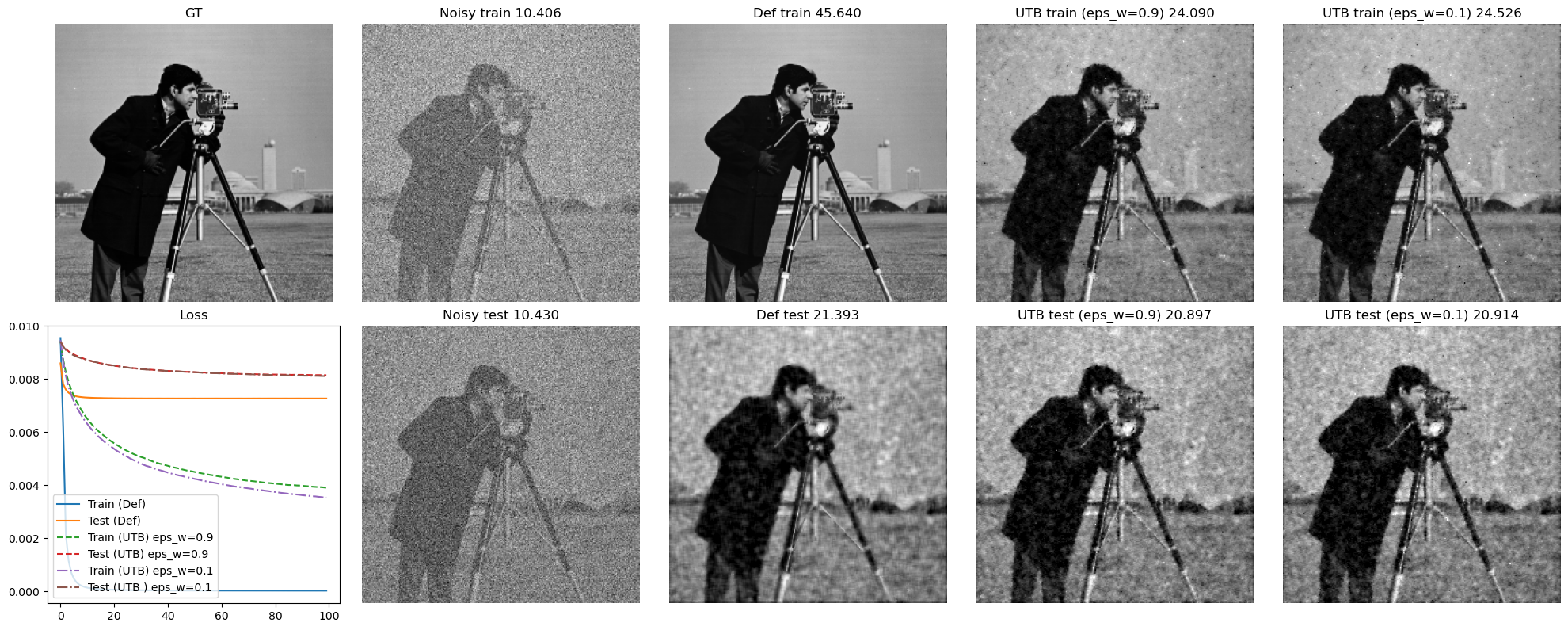}}
        \centerline{}
        \centerline{\includegraphics[width=\textwidth]{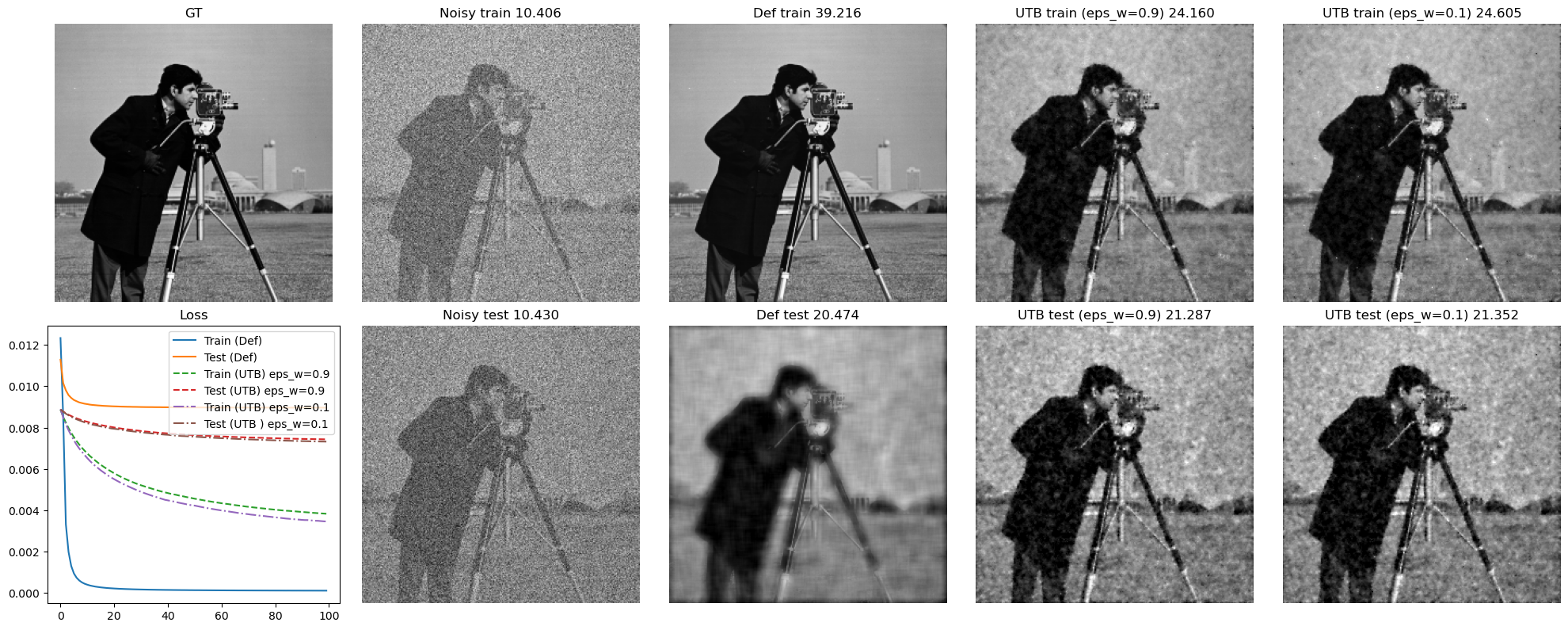}}
        \centerline{}
        \centerline{\includegraphics[width=\textwidth]{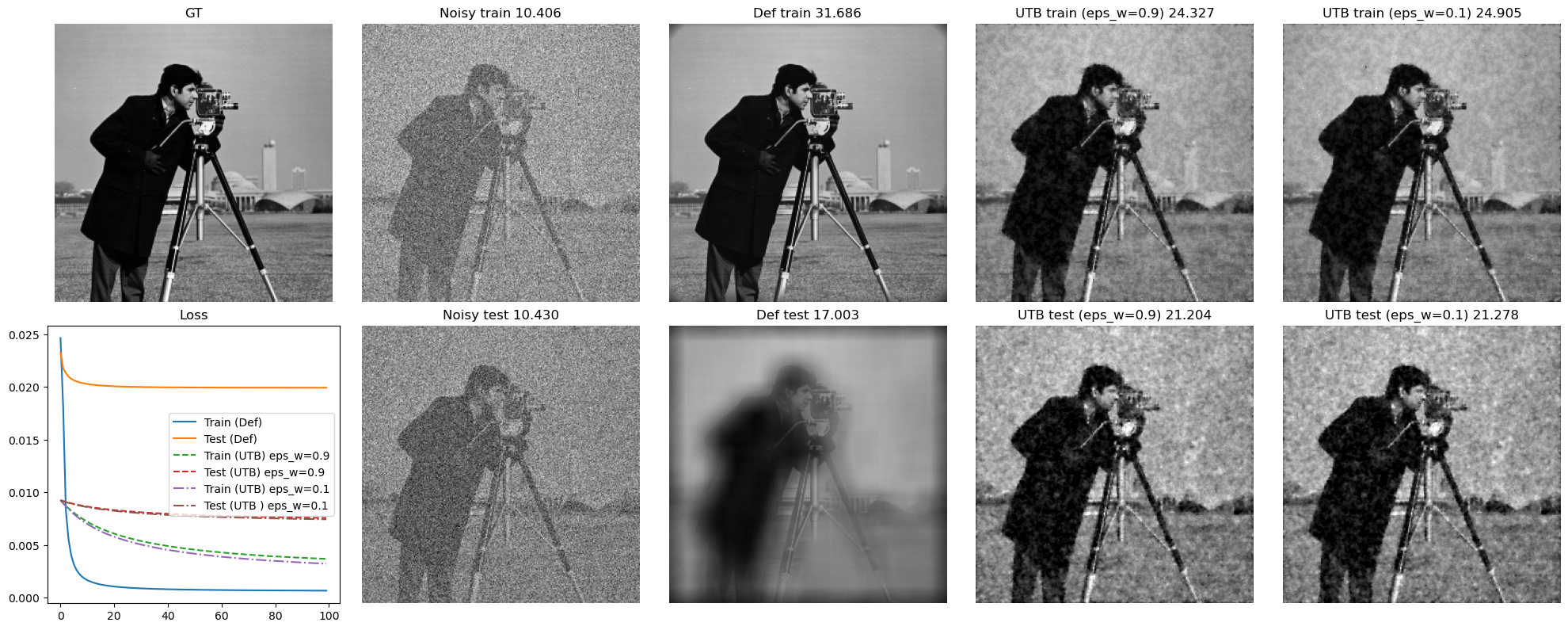}}
        \caption{Results of learnt deformable and our metric UTB convolutions with $\sigma_n = 0.3$ and $k = 5,11,31$ from top to bottom.
        }
        \label{fig: cameraman deform vs tangent sigma 03 k 5 11 31}
\end{figure*}

\begin{figure*}[t]
    \centering
        \centerline{\includegraphics[width=\textwidth]{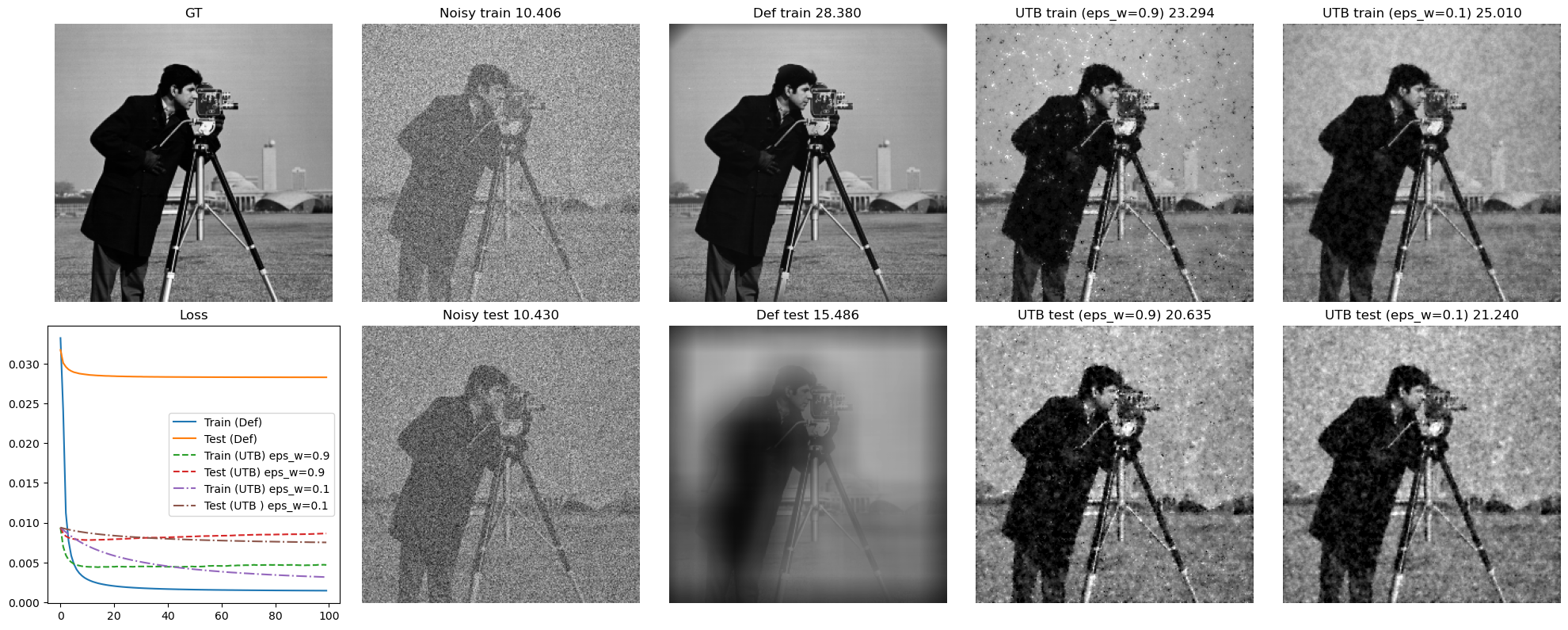}}
        \centerline{}
        \centerline{\includegraphics[width=\textwidth]{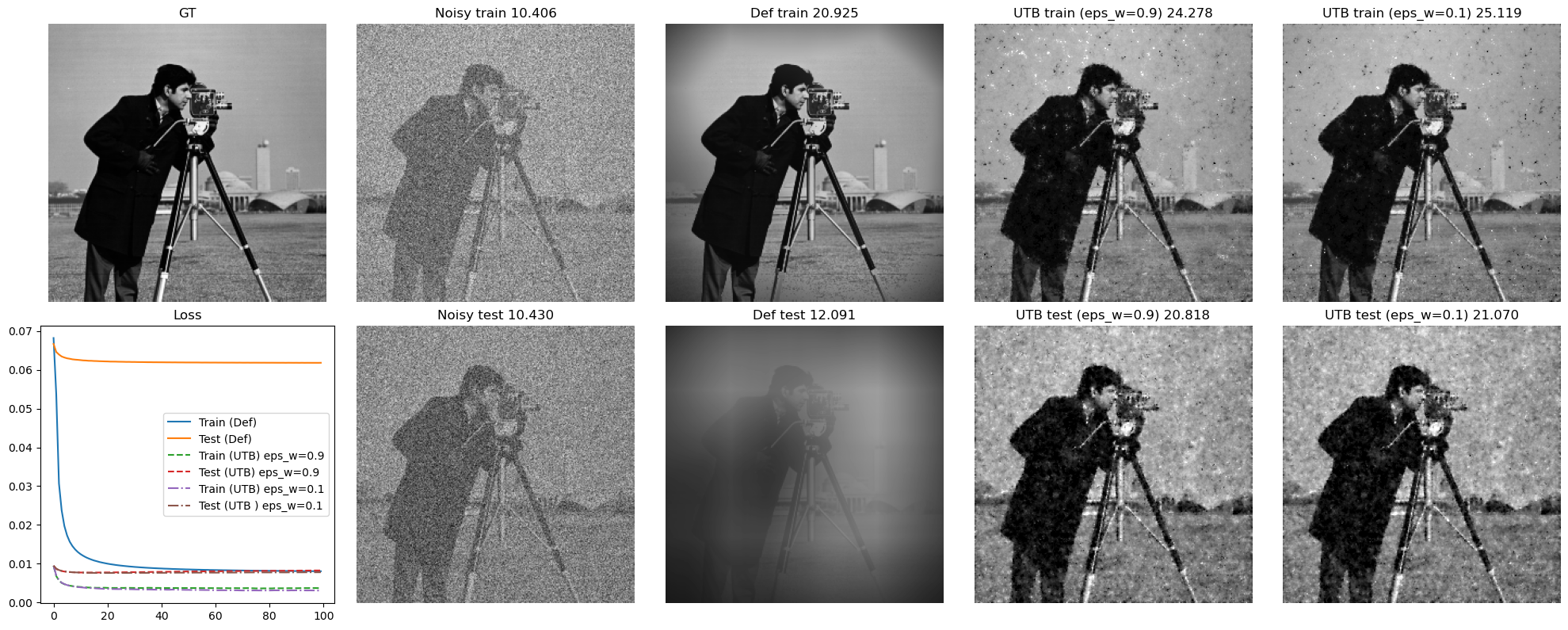}}
        \caption{Results of learnt deformable and our metric UTB convolutions with $\sigma_n = 0.3$ and $k = 51,121$ from top to bottom.
        }
        \label{fig: cameraman deform vs tangent sigma 03 k 51 121}
\end{figure*}

\begin{figure*}[t]
    \centering
        \centerline{\includegraphics[width=\textwidth]{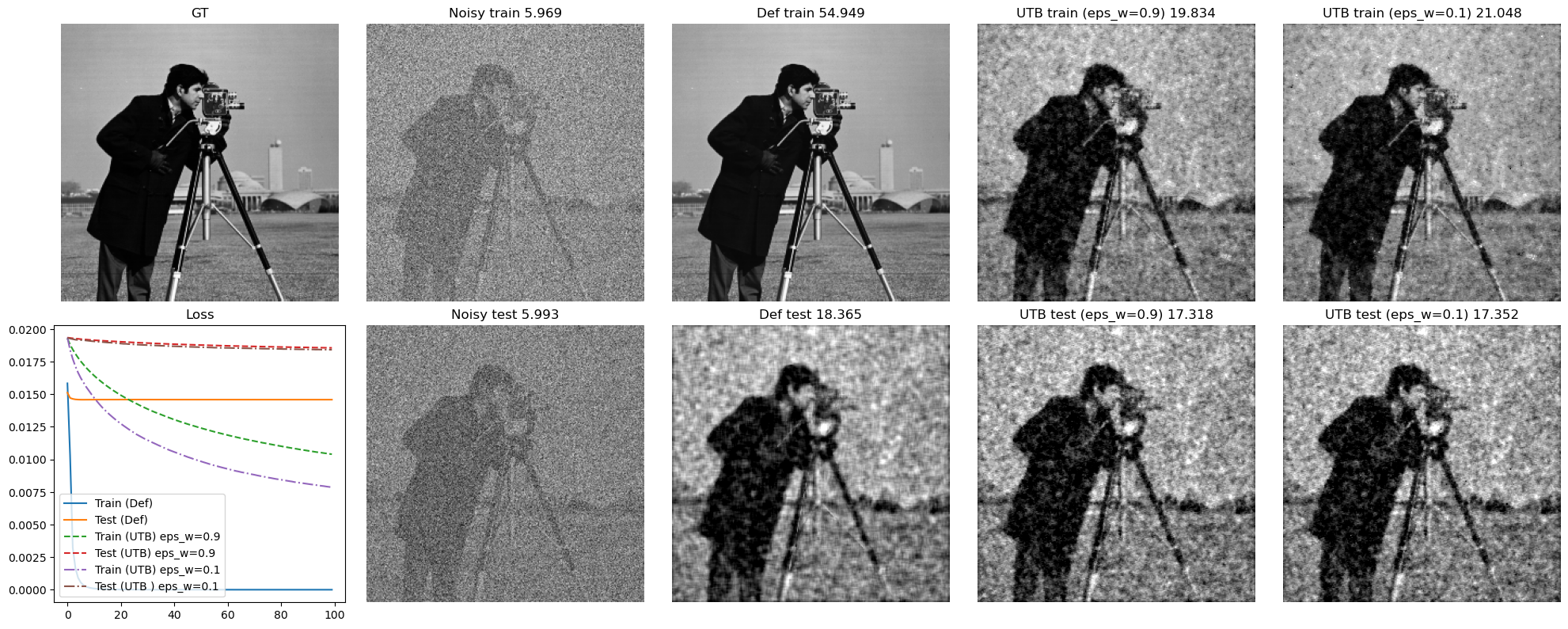}}
        \centerline{}
        \centerline{\includegraphics[width=\textwidth]{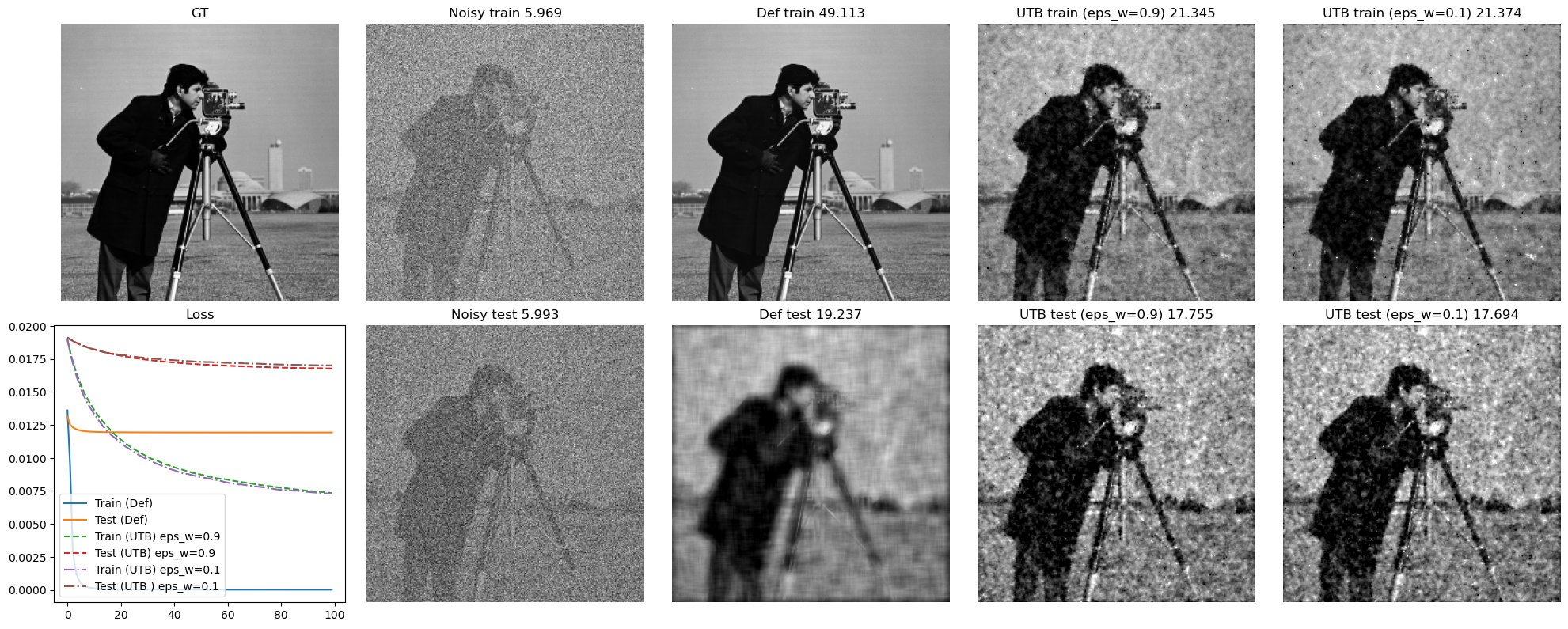}}
        \centerline{}
        \centerline{\includegraphics[width=\textwidth]{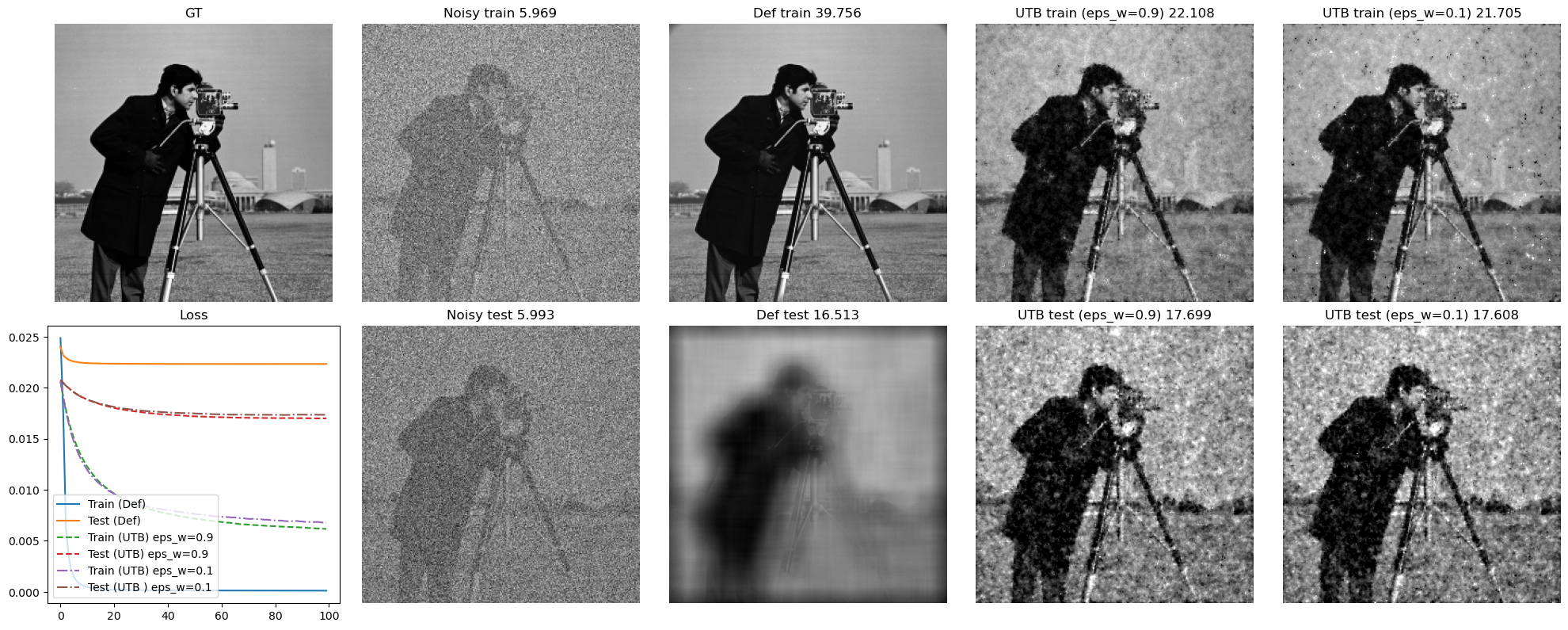}}
        \caption{Results of learnt deformable and our metric UTB convolutions with $\sigma_n = 0.5$ and $k = 5,11,31$ from top to bottom.
        }
        \label{fig: cameraman deform vs tangent sigma 05 k 5 11 31}
\end{figure*}

\begin{figure*}[t]
    \centering
        \centerline{\includegraphics[width=\textwidth]{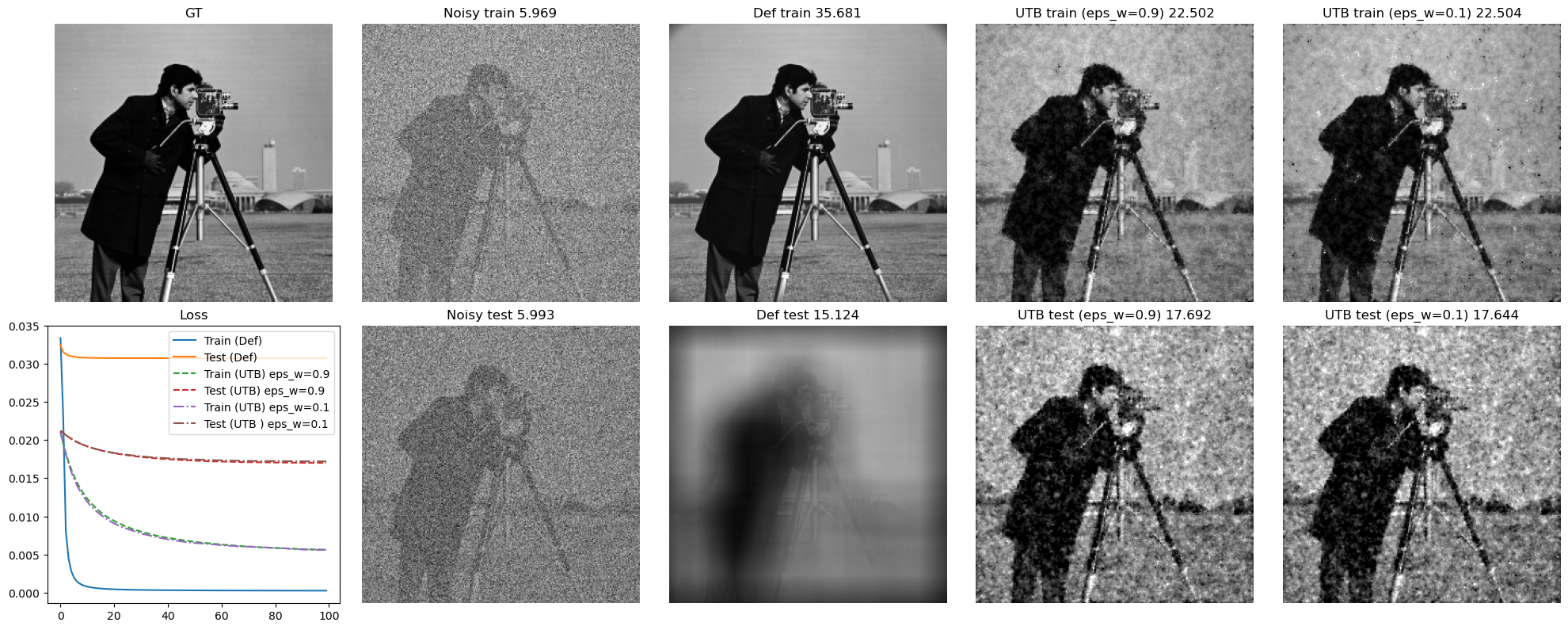}}
        \centerline{}
        \centerline{\includegraphics[width=\textwidth]{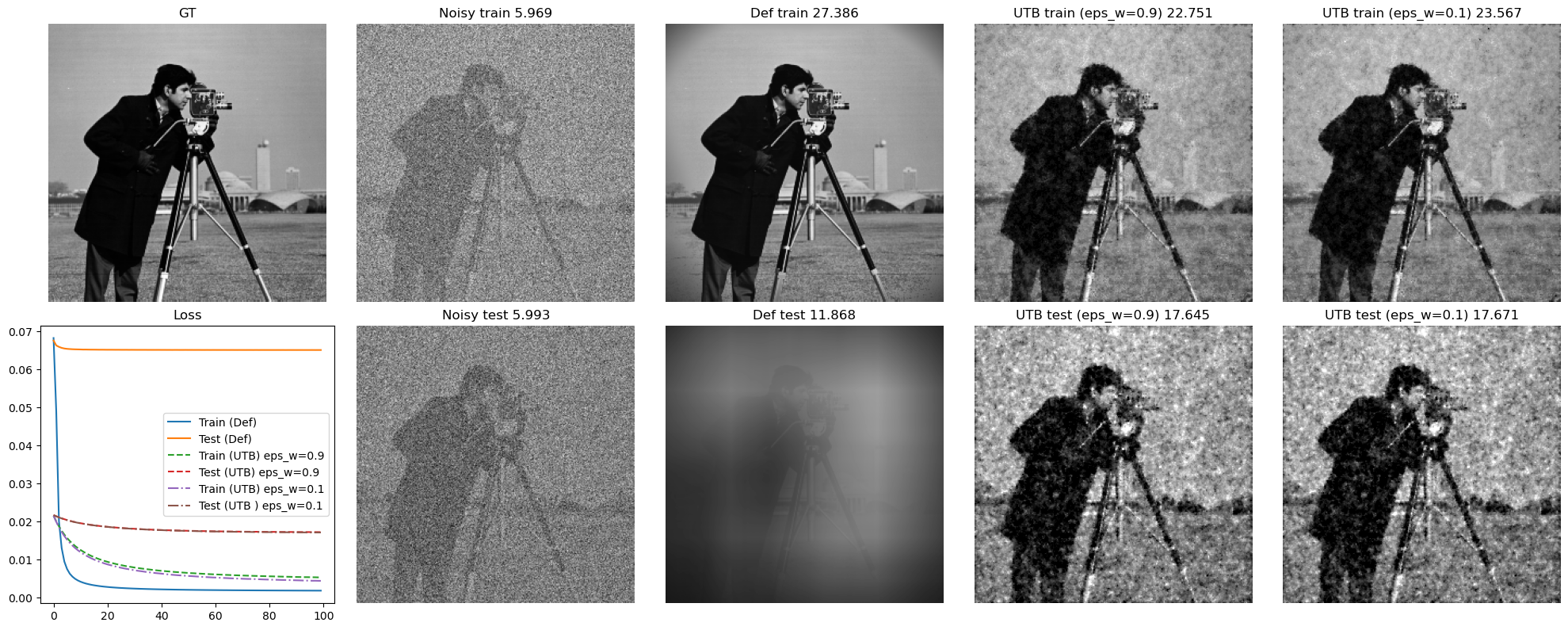}}
        \caption{Results of learnt deformable and our metric UTB convolutions with $\sigma_n = 0.5$ and $k = 5,11,31$ from top to bottom.
        }
        \label{fig: cameraman deform vs tangent sigma 05 k 51 121}
\end{figure*}

\end{document}